\documentclass{article}
\usepackage{arxiv}

\usepackage{amsmath, amsthm, amssymb, epsfig, sectsty, graphicx, float, url}
\usepackage{algorithm, algorithmicx, algpseudocode}
\usepackage{float}
 \usepackage{subcaption}
 \usepackage{caption}

\usepackage[utf8]{inputenc} 
\usepackage[T1]{fontenc}    
\usepackage{hyperref}       
\usepackage{url}            
\usepackage{booktabs}       
\usepackage{amsfonts}       
\usepackage{nicefrac}       
\usepackage{microtype}      
\usepackage{lipsum}
\usepackage{bm}
\usepackage{xcolor}

\numberwithin{equation}{section}

\usepackage{natbib}
 \bibpunct[, ]{(}{)}{,}{a}{}{,}%

 \usepackage{amsmath}
\usepackage{algorithm, algorithmicx, algpseudocode}
\usepackage{float}
\usepackage{subcaption}
\usepackage{bm}

\newcommand{\R}{\mathbb{R}}

\newcommand{\N}{\mathbb{N}}
\newcommand{\X}{\mathbb{S}}

\newcommand{\E}{\mathbb{E}}
\newcommand{\U}{\mathcal{U}}
\newcommand{\I}{\mathcal{I}}
\newcommand{\W}{\mathcal{W}}

\theoremstyle{plain}
\newtheorem{theorem}{Theorem}[section]
\newtheorem{lemma}[theorem]{Lemma}

\newtheorem{corollary}[theorem]{Corollary}
\newtheorem{proposition}[theorem]{Proposition}
\newtheorem{assumption}[theorem]{Assumption}
\theoremstyle{definition}
\newtheorem{definition}[theorem]{Definition}
\newtheorem{remark}[theorem]{Remark}

\title{Robust Policies for Proactive ICU Transfers}

\author{%
  Julien Grand-Cl\'ement\\
  IEOR Department, Columbia University\\
  \texttt{jg3728@columbia.edu} \\
   \And
   Carri W. Chan \\
   Columbia Business School, Columbia University \\
   \texttt{cwchan@gsb.columbia.edu} \\
      \And
   Vineet Goyal \\
   IEOR Department, Columbia University \\
   \texttt{vg2277@columbia.edu} \\
       \And
   Gabriel Escobar \\
  Kaiser Permanente Division of Research, \\
   \texttt{gabriel.escobar@kp.org} \\
}

\begin{document}

\maketitle

\begin{abstract}
Patients whose  transfer to the Intensive Care Unit (ICU) is unplanned are prone to higher mortality rates and longer length-of-stay. Recent advances in machine learning to predict patient deterioration have introduced the possibility of \emph{proactive transfer} from the ward to the ICU.  In this work, we study the problem of finding \emph{robust} patient transfer policies which account for uncertainty in statistical estimates due to data limitations when optimizing to improve overall patient care. We propose a Markov Decision Process model to capture the evolution of patient health, where the states represent a measure of patient severity.  Under fairly general assumptions, we show that an optimal transfer policy has a threshold structure, i.e., that it transfers all patients above a certain severity  level to the ICU (subject to available capacity). As  model parameters are typically determined based on statistical estimations from real-world data, they are inherently subject to misspecification and estimation errors. We account for this parameter uncertainty by deriving a robust policy that optimizes the worst-case reward across all plausible values of the model parameters. {\color{black} We are able to prove structural properties on the optimal robust policy and to compare it to the optimal nominal policy.} In particular, we show that the robust policy also has a threshold structure under fairly general assumptions. Moreover, it is more aggressive in transferring patients  than the optimal nominal policy, which does not take into account parameter uncertainty.  We present computational experiments using a dataset of hospitalizations at 21 Kaiser Permanente Northern California hospitals, and present empirical evidence of the sensitivity of various hospital metrics (mortality, length-of-stay, average ICU occupancy) to small changes in the parameters. While threshold policies are a simplification of the actual complex sequence of decisions leading (or not) to a transfer to the ICU, our work provides useful insights into the impact of parameter uncertainty on deriving simple policies for proactive ICU transfer that have strong empirical performance and theoretical guarantees.
\end{abstract}


\section{Introduction.}
In a hospital, critically ill patients are treated in the Intensive Care Unit (ICU), where they require a significant amount of human and material resources \citep{Milbrand}. Effective management of ICUs has substantial implications, both for the patient outcomes and for the operational costs of the hospital.
The sudden health deterioration of a patient in the general medical/surgical ward can result in an unplanned transfer to the ICU and a severe downturn in the chance of survival of the patient. Such unplanned transfers  typically have worse outcomes than patients who are directly admitted to the ICU (e.g. \cite{Barnett,Escobar-2013}). Developing strategies to effectively manage the limited ICU beds \citep{Green} is becoming even more critical as demand for ICU care is increasing \citep{Mullins}.  The primary focus of this work is to derive and evaluate \emph{robust proactive ICU transfer policies} in order to improve patient flow and patient outcomes.

Recent advances in machine learning have brought real-time risk scores of a patient's likelihood of deterioration available to clinicians' use in hospitals \citep{Escobar-2012}. Consequently, understanding the impact of intervening on patients based on such scores (which is currently occurring based primarily on informed clinical judgment and empirical observation) needs to be better understood from a theoretical perspective. In practice, alerts based on such scores are known to trigger multiple types of response, which can range from simple maneuvers (e.g., increased monitoring, one time fluid boluses) to immediate transfer to the ICU. In some cases, alerts trigger discontinuation of life support. In this work, as a first step towards better theoretical understanding of the  pathways involved in early warning systems and in order to focus on the impact of parameter uncertainty, we focus on proactive ICU admissions, while recognizing this is just one of many potential interventions that could take place.
 Using the data of nearly $300,000$ hospitalizations at KPNC,  the authors in \cite{ICU-wenqi} provide empirical evidence that proactively transferring patients to the ICU can significantly reduce the average mortality risk and the Length-Of-Stay (LOS). However, using simulation to consider the system-wide effect of various proactive transfer policies, the authors  provide a cautionary tale that overly aggressive transfers can have a significant impact on increasing ICU occupancy, which is associated with worse outcomes (e.g. \cite{bump-2}).

While \cite{ICU-wenqi} demonstrates that there is promising potential in the use of proactive transfers, there are some limitations with respect to the insights developed at the system-wide level.
First, in practice, the actual ICU admission decision relies on a complex sequence of events that are activated when a patient's severity score reaches the alert threshold. Therefore, the model of ICU transfer based only on the severity scores is a simplification. Second, there is limited theoretical basis for the class of policies
 (threshold and random) which are considered.  But, perhaps more critically, the core parameters of the simulation model are calibrated from real data and are subject to uncertainty. In particular, the transition rates of a Markov chain are estimated from a finite dataset of patient hospitalizations and are an approximation of the true parameters. This is concerning because the performance of a  policy can significantly deteriorate, even under small variations from the true parameters (e.g. Section \ref{sec:exp} of this paper). Consequently, an optimal transfer policy for the estimated parameters might perform very poorly in practice even if the true parameters are close but different. This limitation is widely acknowledged in the healthcare community and is typically addressed by conducting sensitivity analysis. This is the approach taken in \cite{ICU-wenqi}.   However, when models have many parameters -- as ours does -- the comprehensiveness of these types of sensitivity analysis can be limited due to computational reasons.  This paper proposes to look at the problem by optimizing  the worst-case  performance over an uncertainty set by using tools from robust optimization.

Our goal in this paper is to develop \textit{robust} transfer policies, i.e., transfer policies with guarantees of good performance over a given set of plausible hospital parameters, which are consistent with available data.  This is in contrast to \textit{nominal} transfer policies, which are only guaranteed to have good performance for known fixed values of the parameters and could have very bad performance for close, but different, parameters. In doing so, we will leverage results from the Robust MDP literature (e.g. \cite{Iyengar, Mannor, Kuhn, Goh,GGC}) to develop a theoretical and empirical basis for our proposed transfer policies.

Our main contributions, both methodological and practical, can be summarized as follows:

\begin{itemize}
\item[] {\textbf{Markov model for a single patient.}}
We propose an approximation of the full hospital dynamics, using the health evolution of a \textit{single} patient. In particular, we present a Markov Decision Process (MDP) to model the patient health evolution. This MDP is able to capture the fundamental trade-off between the benefit of proactive transfer for individual patients versus suboptimal use of limited ICU resources for patients who may not ``really need it''.  We also show that our single-patient MDP can be interpreted as a relaxation of a more expressive, but intractable, multi-patient MDP which directly incorporates the ICU capacity constraints in the decision-making.

\item[] {\textbf{Structure of optimal nominal policies.}}
Under fairly general and interpretable assumptions that we expect to hold in practice, we show that an optimal proactive transfer policy in our single-patient MDP is a threshold policy. In particular, there exists an optimal policy that transfers all patients above a certain severity score. This threshold structure is particularly nice because of its interpretability and implementability.

\item[] {\textbf{Robustness of transfer policies.}} Building upon the nominal model, we incorporate the real-world limitations of parameter uncertainty, by considering parameter misspecification for the transition matrix. Prior work \citep{Iyengar,Kuhn} assume \textit{rectangularity}, where transitions out of different health states are unrelated. However, underlying factors, such as genetics, demographics, and/or physiologic characteristics of certain diseases, could dictate the health evolution of patients in specific health states.  As such, we consider a model of uncertainty where the transition probabilities of different health states are \emph{correlated} and depend on a factor model \citep{Goh,GGC}.

We present an efficient algorithm to compute an \textit{optimal robust policy} that maximizes the worst-case possible outcomes over all plausible transition matrices. Moreover, we prove structural results for the optimal robust policy and compare it to the nominal policy. In particular, an optimal robust policy is always deterministic and -- under the same assumptions as in our nominal model -- of threshold type. {\color{black}This is in contrast to the general situation in robust MDPs, where there may not exist a optimal robust policy that is deterministic.} Additionally, the threshold of the optimal robust policy is lower than the threshold of the optimal nominal policy. Therefore, the optimal robust policy transfers more patients to the ICU than the optimal nominal policy.

\item[] {\textbf{Numerical experiments.}}
We present detailed numerical experiments to compare the performance of the optimal nominal and robust transfer policies, making use of the hospitalization data of almost $300,000$ patients at Kaiser Permanente.
We  observe that, for our single-patient MDP, the performance of the optimal nominal policy can deteriorate even for small variations of the model parameters. Moreover, there are significant differences in the recommended thresholds between the nominal and  robust policies, which reflects that these polices could have substantial differences in the proportion of patients who are  proactively transferred. When considering the full hospital model, we observe similar deterioration in performance (as measured by mortality, length-of-stay, average ICU occupancy) even for small parameters deviations. Additionally, we find that correlated uncertainty in the transition matrix results in different and more useful insights than when considering uncorrelated uncertainty which can be overly conservative. We also highlight the contrast between this worst-case analysis and more standard sensitivity analysis approaches. {\color{black} Our work suggests that proactively transferring the patients with the riskiest severity scores has the potential to improve the hospital mortality and LOS, without significantly increasing the ICU occupancy, even in the worst-case.}
\end{itemize}

The rest of the paper is structured as follows. We finish this section with a brief overview of related literature. In Section \ref{sec:hosp}, we present the hospital model and the Markov chain that describes the evolution of a patient's health.
In Section \ref{sec:MDP}, we introduce a Markov Decision Process  to approximate the full hospital model and we theoretically characterize the structure of optimal nominal policies. We address parameter uncertainty in Section \ref{sec:rob}, where we introduce our model of uncertainty and  we prove some theoretical results on the structure of optimal robust policies. In Section \ref{sec:exp}, we present computational experiments based on a dataset from Kaiser Permanente Northern California and we examine the contrast between the optimal nominal and optimal robust policies.

\vspace{2mm}
\noindent \textbf{Notations.} For an integer $n \geq 0$, we denote by $[n]$ the set $\{1,...,n\}.$ Vectors and matrices are in bold font whereas scalars are in regular font, except for policies $\pi$ which are also in regular font. The vector $\boldsymbol{e}$ has every component equal to one and its dimension depends on the context.

\subsection{Related work.}
Our work mainly involves three topics of research: (i) ICU management and proactive care in hospitals, (ii) Markov Decision Process in healthcare and (iii) robust Markov Decision Process,  particularly those applied to problems in healthcare.

\paragraph{ICU management and proactive care in hospitals.} There is a large and growing body of literature in both the operations and medical literature on the management of ICUs. For instance, the impact of congestion and demand-driven discharges has been considered both empirically (e.g. \cite{bump} and \cite{bump-2}) and theoretically (e.g. \cite{chan2012optimizing}). More closely related to our work is admission into the ICU.  A number of papers, including \cite{adm-decision-1}, \cite{simu-1}, and \cite{adm-decision-2}  consider the impact  of ICU admission decisions on patient outcomes when patients arrive to the hospital.  {\color{black} Threshold policies have been investigated in various admission control settings (including the ICU), but most prior works either focus on a simple transition model (e.g. a patient in severity class $i$ can only transition to condition $i+1$ or $i-1$), or only focus on the \textit{empirical} performance of threshold policies (e.g. \cite{barron2016performance,barron2018threshold} for related approaches in machine maintenance). \cite{altman} prove the optimality of threshold policies, but rely on a submodularity assumption, which is less interpretable than ours in a healthcare setting.
In contrast to these works, we consider a setting where patients can be admitted to the ICU at any point while they are in the general medical/surgical ward and can transition to any other severity condition. In such a setting we give interpretable conditions for the optimality of threshold policies.}

Given the limited number of hospital resources and the adverse impact of strained ICUs on the quality of care provided (e.g. \cite{Alberta}), there has been a growing interest in the development of predictive models for patients dynamics and outcomes, including LOS, death and readmission rates to the ICU. For instance,  \cite{chronic} develop a risk-adjustment metric to predict patient hospitalization. \cite{score-ward} study the performance of the National Early Warning Score (NEWS) as a risk score for ED patients, while the NEWS risk score has been specifically developed for patients in the hospital ward.
 \cite{pred-2} utilize expert opinion, naive Bayes and logistic regression to predict the number of patients in the Emergency Department (ED) who will be admitted to a particular inpatient unit. \cite{score-icu-1} develop ICU admission scores aimed at predicting mortality risk after coronary artery bypass grafting, while \cite{score-icu-2} develop a scoring system for predicting ICU admission after major lung resection; see \cite{score-icu-survey} for a review of the medical severity scores for ICU patients.

The operations community has studied how preventative/proactive care can be used to improve patient care. For instance,  \cite{pred-1} design efficient proactive ED admission control policies  based on predictions of potential patient arrivals, while proactive care using Markov models can be dated back to at least \cite{ozekici}, where the authors introduce an MDP to compute an optimal inspection schedule in the case of post-operative periumbilical pruritus and breast cancer.

This paper focuses on the dynamic decision of whether and when to proactively transfer a patient to the ICU based on a patient's risk of deterioration (e.g. \cite{Escobar-2012}). The particular problem we study is  related to that  in \cite{ICU-wenqi} which uses simulation to investigate the impact of proactive transfers to the ICU on the patients' flow in the hospital and on the in-hospital mortality, LOS and ICU occupancy. While a substantial focus of \cite{ICU-wenqi} is to rigorously estimate the causal effect of proactive transfers on individual patients, we focus on utilizing MDP approaches to derive theoretically justified transfer policies.  Interestingly,  \cite{ICU-wenqi} conduct a sensitivity analysis by considering random deviations in the parameters of their model over the confidence intervals of the parameter estimates as an acknowledgement of the potential impact of parameter uncertainty.
{\color{black} Most of the healthcare literature primarily focuses on sensitivity analysis of the selected policy.  Unfortunately, sensitivity analysis is unable to capture possible \textit{adversarial} deviations of these parameters and, as we will show in this work, such deviations can substantially impact system performance. In contrast, robust optimization explicitly accounts for the uncertainty (and adversarial deviations) when deriving good policies. }

\paragraph{Markov Decision Process in healthcare.} In this work, we will leverage the methodology of Markov Decision Processes (MDP).  This modeling framework has been used extensively in many healthcare applications including early detection, prevention, screening and treatment of diseases. MDPs are particularly efficient to analyze chronic diseases and decisions that are made sequentially over time in a stochastic environment. In particular, MDPs have been used, among others, for kidney transplantation \citep{mdp-kidney}, HIV treatment recommendation \citep{mdp-HIV}, breast cancer detections \citep{mdp-breast-cancer},  cardiovascular controls for patients with Type 2 diabetes \citep{mdp-steimle} and determining the optimal stopping time for medical treatment \citep{cheng2019}. We refer the reader to \cite{mdp-med-1} and \cite{mdp-med-2} for reviews of applications of MDP to medical decision making.

\paragraph{Robust Markov Decision Processes.}
In most medical applications, we only have access to observational data. Consequently, we can only obtain a \emph{noisy estimate} of the true parameters of the MDP, and the decision-maker may  recommend a treatment that performs poorly with respect to the true parameters.
{\color{black}
Partially Observable MDPs (POMDPs) assume that the system dynamics are determined by an MDP, but the agent cannot fully observe the underlying states. Instead, the decisions must be based on the \textit{observed} states. POMDPs are generally hard to solve \citep{madani1999undecidability}. Robust POMDPs are known to be even harder to solve and may lead to randomized optimal policies \citep{rasouli2018robust}, making POMDPs unusable in our healthcare application.}
Robust MDPs address the issue of parameter misspecification in the MDP \citep{Iyengar,Nilim,Kuhn, Xu-Mannor}. The goal is to compute an optimal \textit{robust} policy, i.e., a policy that maximizes the worst-case expected outcome over the set of all plausible parameters. More specifically, the authors in \cite{Iyengar, Nilim} and \cite{Kuhn} present algorithms to efficiently compute an optimal robust policy, provided that the parameters related to different state-action pairs are unrelated. Such \textit{rectangular} uncertainty sets are quite conservative and do not allow relations across transition probabilities from different  states. This is potentially very conservative especially if the transition probabilities depend on a common set of underlying factors, as could be the case in healthcare applications.
{\color{black}
In principle one could also use \textit{distributionally} robust MDPs \citep{xu2010distributionally} to ameliorate the issues with parameter uncertainty, but this leads harder optimization problems than the linear program of our own robust MDP formulations.  Additionally,  it is not clear how to build ambiguity sets around the nominal density estimation for distributionally robust MDPs, while it is reasonably easy to incorporate confidence intervals  in the uncertainty sets of robust MDPs.}

\paragraph{Robust Markov Decision Processes in healthcare.}
In light of the limitations of the \textit{rectangularity} assumption for modeling  parameters uncertainty, the authors in \cite{steimle} develop a \textit{multi-model} MDP approach and apply it to a case study of cholesterol management. However, computing the optimal robust policy of a multi-model MDP is intractable in general. We use the model of \textit{factor matrix uncertainty}, introduced in \cite{Goh} and further analyzed in \cite{GGC}. In particular, the authors in \cite{Goh} use a model of uncertainty (later referred to as \textit{factor matrix} uncertainty set) which captures transitions that are jointly varying in the set of all plausible parameters. They show how to compute the worst-case reward for a given policy, and apply these methods to a cost-effectiveness analysis of fecal immunochemical testing for detecting colorectal cancer. \cite{GGC} show that for a factor matrix uncertainty set, one can also compute the optimal robust policy, i.e., the policy that maximizes the worst-case reward. {\color{black} They also prove important structural properties on the optimal value vectors, which we can use to prove structural properties on the optimal nominal and robust policies.} Our work builds upon \cite{Goh} and \cite{GGC} by making use of a factor matrix uncertainty set in the specific setting of a Markov chain to model the patient's trajectory in a hospital and capture different levels of connectivity in parameter deviations.   Specifically,  we are able to (i) compute an optimal robust policy, and (ii) give theoretical guarantees on the \textit{structure} of the optimal policies, namely the optimal nominal and the optimal robust policies  both have a threshold structure. {\color{black}

 While \cite{Goh} consider similar uncertainty formulation and present an algorithm to compute the \textit{worst-case} performance of a given policy, they do not consider the problem of finding \textit{optimal} robust policies.


Building upon the results in \cite{GGC}, we focus on the specific problem of proactive transfer to the ICU.  Our single-patient and multi-patient MDP models and interpretable assumptions appear to be novel in the literature. Our detailed numerical experiments with our hospital model emphasize the advantage of our robustness analysis (compared to a simple sensitivity analysis) in highlighting the potential detrimental impact of adversarial parameter deviations. We also present extensive details into the construction of a factor matrix uncertainty sets directly from the data, in contrast to the simulations in \cite{GGC} which consider that the factor matrix uncertainty set is already given.}



\section{Hospital model and proactive transfer policies.}\label{sec:hosp}
We formally present our discrete time hospital patient flow model. This model is similar to \cite{ICU-wenqi} and is depicted in Figure \ref{fig:hosp_model}. We consider a hospital with two levels of care, the Intensive Care Unit (ICU) and the general medical/surgical ward (ward). In order to focus on the management of the ICU, we assume that the ward has unlimited capacity while  the ICU has a limited capacity $C < + \infty$.

 \begin{figure}[htb]
 \begin{center}
 \includegraphics[scale=0.3]{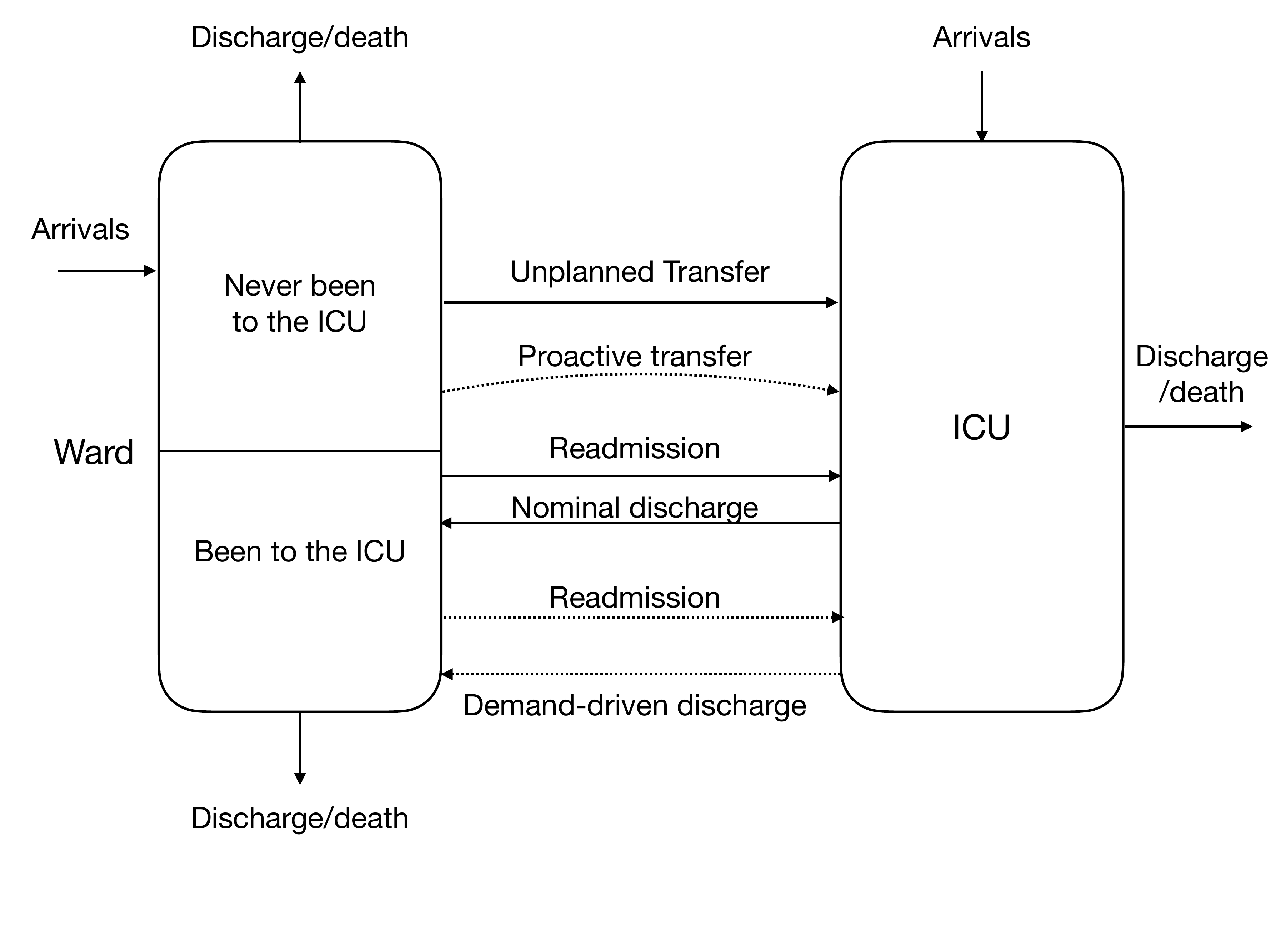}
 \caption{Simulation model for the hospital.}
 \label{fig:hosp_model}
  \end{center}
 \end{figure}

 \subsection{Model Dynamics}
{\bf \emph{ Ward patients:}}
We start by describing the dynamics of the patients on the ward.
These patients are divided into those who have already been to the ICU during their hospital stay and those who have not. The state of a patient who has never been to the ICU is captured by a severity score $i \in \{1,...,n\}$, for a given number of severity scores $n \in \N$. Each patient is assigned  a severity score at arrival in the hospital, and this score is then updated in each time slot. We model the arrivals of patients with  severity score $i$ as a non-homogeneous Poisson process $\lambda_{i}(t)$.
We model the evolution of the severity scores as a Markov Process with transition matrix $\bm{T}^{0} \in \R^{n \times (n+3)}$.

In each time slot, a patient whose current severity score is $i \in [n] $ transitions to another severity score in $j \in [n]$ with probability $T_{ij}^{0}$. In addition, one of the three following events may happen at the end of each time slot:

1) With probability $T_{i,n+1}^{0}$, the patient may \textit{crash} and require a \emph{reactive} transfer to the ICU.

2) With probability $T_{i,n+2}^{0}$, the patient may fully recover and leave the hospital.

3) With probability $T_{i,n+3}^{0}$, the patient may die.

 {\color{black} We verify the consistency of this approach by comparing the empirical probability of crashing, dying in the ward, and surviving to hospital discharge in the data with the results of our hospital simulations using our Markov chain. We find that with our parameters calibrated as in Appendix \ref{app:detail_hosp}, the key metrics are comparable. For instance the mortality rate in the ward is 1.93 \% in our simulations, which is comparable to the 2.2 \% computed from the data directly.}

\emph{Reactive transfer to the ICU.} If a patient crashes on the ward, s/he will be admitted to the ICU. Upon ICU admission, the patient's remaining hospital Length-Of-Stay (LOS) is modeled as being log-normally distributed with mean $1/\mu_{C}$ and standard deviation $\sigma_{C}$. We consider a model where a proportion $p_{W}$ (following a distribution with density $f_{p_{W}}$) of this LOS is spent in the ICU, while the remaining proportion of time, $1-p_{W}$, is spent in the ward. During this time, the patient may again require ICU admission, with a rate of $\rho_{C}$. Any patient still in the ward at the end of this LOS is assumed to die  with probability $d_{C}$, or to fully recover and be discharged with probability $1-d_{C}$.
If there are no available beds in the ICU when a patient crashes, the ICU patient with the shortest remaining service time in the ICU is discharged back to the ward in order to accommodate the incoming patient. We refer to this event as a \textit{demand-driven discharge}. A demand-driven discharged patient  has an ICU readmission rate of $\rho_{D}$. The authors in \cite{bump-2} suggest that $\rho_{D}$ is higher than  $\rho_{E}$, the readmission rate of patients who are naturally discharged from the ICU.
{\color{black} We note that in \cite{chan2012optimizing}, the ratio $\rho_{D}/\rho_{E}$ fluctuates between 1.11 and 1.18.  Therefore, we set $\rho_{D}=1.15 \cdot \rho_{E}$. Note that this choice of parameters is stress-tested in the numerical experiments in \cite{ICU-wenqi} (Section 5 and Appendix B).}
{\color{black}
\begin{remark} 
We do not model with a Markov chain the health dynamics of patients discharged from the ICU. This would require using a transition matrix different than $\bm{T}^{0}$, since these patients are sicker than the population of patients who have not been to the ICU. This would have two issues.  First, there are way fewer observations for patients who have been to the ICU than those who have never been to the ICU, and this would result in unreasonably large confidence intervals for the coefficients of the transition matrix.  Second, and perhaps most importantly,  using a different Markov chain would not allow us to explicitly capture the impact of ICU discharge on LOS/mortality (whereas our current approach does). We find it more relevant to directly use simpler statistics of mortality rate and LOS to summarize the impact of ICU admission onto these patients, rather than using a different Markov chain.
\end{remark}
}


\emph{Proactive transfers to the ICU.} If there are beds available in the ICU, a patient can be  \textit{proactively transferred} from the ward to the ICU. Such patients typically have better outcomes than those who crash and require a reactive ICU transfer \citep{ICU-wenqi}. If a patient with severity score $i \in [n]$ is proactively transferred, the LOS is modeled as being log-normally distributed with mean $1/\mu_{A,i}$ and standard deviation $\sigma_{A,i}$, while a proportion $p_{W} \sim f_{p_{W}}$ of this LOS is spent in the ICU. As in the case of reactive transfer, the patient will then survive to hospital discharge with probability $1-d_{A,i}$. We assume that  $1-d_{A,i}\geq 1-d_{C}$, i.e., the patient is more likely to survive if proactively transferred. If the patient is naturally discharged from the ICU, the readmission rate is $\rho_{A,i}$, otherwise it is $\rho_{D}$. We set $\rho_{A,i}=\rho_{C}$,  as these two types of patients are transferred to the ICU from the ward, in contrast to direct admits patients.

Note that in practice when a patient reaches an alert threshold he may enter an evaluation state where further tests and examinations are required before an admission decision is made. In some instances, the patient may never be admitted to the ICU; e.g. if the alert is an error or the patient has a directive to not provide rescue measures. In order to focus on the impact of parameter uncertainty, our model assumes that ICU admission decisions are made right after the alert threshold is attained.

{\bf \emph {Direct admits to the ICU:}} In addition to reactive and proactive ICU transfers from the ward, patients can also be directly admitted to the ICU.  We model the arrivals of these patients  with a non-homogeneous Poisson process with rate $\lambda_{E}(t)$. Their LOS is log-normally distributed with mean $1/\mu_{E}$ and standard deviation $\sigma_{E},$ and a proportion $p_{E}$ (following a distribution with density $f_{p_{E}}$) of this LOS is spent in the ICU, while the remaining time is spent in the ward. At the end of this LOS, the patient fully recovers and leaves the hospital with probability $1-d_{E}$, or dies with probability $d_{E}$.

Details about the distribution laws of the different stochastic processes (arrivals in the ward and in the ICU, transition matrix across severity scores, distribution of LOS and mortality rate, etc.) involved in the hospital model of Figure \ref{fig:hosp_model} can be found in Appendix \ref{app:detail_hosp}.
\subsection{Transfer policies}
A transfer policy $\pi$ is a decision rule that, for each patient in the ward, decides whether and when to  proactively transfer the patient to the ICU (subject to bed availability). Our goal is to study the impact of the proactive transfer policies on  hospital performance  as measured by the mortality rate, average LOS and average ICU occupancy.  A particular class of simple and interpretable transfer policies is the class of \textit{threshold} policies. A policy is said to be threshold if it transfers to the ICU all patients whose severity score are higher than a certain fixed threshold. Such proactive transfers are subject to bed availability in the ICU.

\subsection{Challenges}\label{ssec:challenges}
The hospital model just described captures many salient features of real patient flows. Moreover, it is able to capture the core trade-off we are interested in studying between the benefits of proactive transfers for individual patients and needlessly utilizing expensive ICU resources.  That said, the model also  suffers from some limitations that we elaborate below.

\emph{Tractability.} While our model could be described as a Markov Decision Process (MDP, e.g. \cite{Puterman}), the state space is prohibitively large. For instance, with $p$ patients in the ward and $n$ severity scores, one needs a state space of cardinality $n^{p}$ to describe the state of the ward. Thus, numerically solving this MDP is highly intractable.

Alternatively, one could take the approach in \cite{ICU-wenqi} and use simulation.  However, there are $2^{n}$ deterministic, state-independent, transfer policies (with $n$ the number of severity scores). The number of state-dependent policies grows by the size of the ICU and/or the large ward state-space. Expanding to allow for randomized policies results in an uncountable number of potential transfer policies. Therefore, it is intractable to simply compare all the deterministic transfer policies using simulations\footnote{For each transfer policy, computing the hospital performance (average mortality, LOS, ICU occupancy) takes a couple of hours on a laptop with 2.2 GHz Intel Core i7 with 8 GB of RAM.}.

\emph{Parameter uncertainty.} The hospital model is specified by the stochastic processes detailed above, including a transition matrix $\bm{T}^{0} \in \R^{n \times (n+3)}_{+}$. The coefficients of this matrix are estimated from historical data and consequently suffer from statistical estimation errors. The hospital performance could be highly sensitive to variations in coefficients of the transition matrix $\bm{T}^{0}$ and therefore, optimizing the policies using the estimated transition matrix could lead to suboptimal policies.  In particular, the optimal transfer policy for the hospital model with transition matrix $\bm{T}^{0}$ might lead to significantly suboptimal performance for the true, underlying transition matrix, even for small deviations in $\bm{T}^{0}$. Additionally, the hospital model is itself an approximation of the true hospital dynamics and is therefore, subject to further misspecification errors.

Given these limitations of the model, we turn our attention to the development of insights using an approximation of the hospital dynamics.

\section{A single-patient Markov model.}\label{sec:MDP}
In light of the discussion in Section \ref{ssec:challenges}, we propose a  tractable approximation of the full hospital model using an MDP  that captures the health dynamics of a \textit{single} patient.

\subsection{Single-patient MDP.}\label{subsec:MDP}

\paragraph{State and Action spaces.}
We consider an MDP with $(n+4)$ states. The set of states is
$$ \X = \{1,...,n\} \bigcup \left\{n+1=CR,n+2=RL,n+3=D,n+4=PT \right\}.$$
The states $i \in \{1,...n\}$ model the severity scores of the patient. There are $4$ terminal states, $CR, RL, D$ and $PT$. The state $CR$ models the \textit{crash} of a patient, its subsequent reactive transfer to the ICU, as well as the outcome when the crashed patient finally exits the hospital (i.e., fully recovering or dying). The state $RL$ corresponds to \textit{Recover and Leave}, the state $D$ corresponds to in-hospital \textit{Death}, and the state $PT$ corresponds to a patient who has been \textit{Proactively Transferred}, as well as the outcome when the patient finally exits the hospital (i.e., fully recovering or dying).


For each state $i=1,...,n$, there are $2$ possible actions, which model the decision of proactively transferring the patient (action $1$) or not (action $0$).

Figure \ref{fig:mdp_model} depicts the single-patient MDP.
 \begin{figure}[htb]
 \centering
 \includegraphics[scale=0.3]{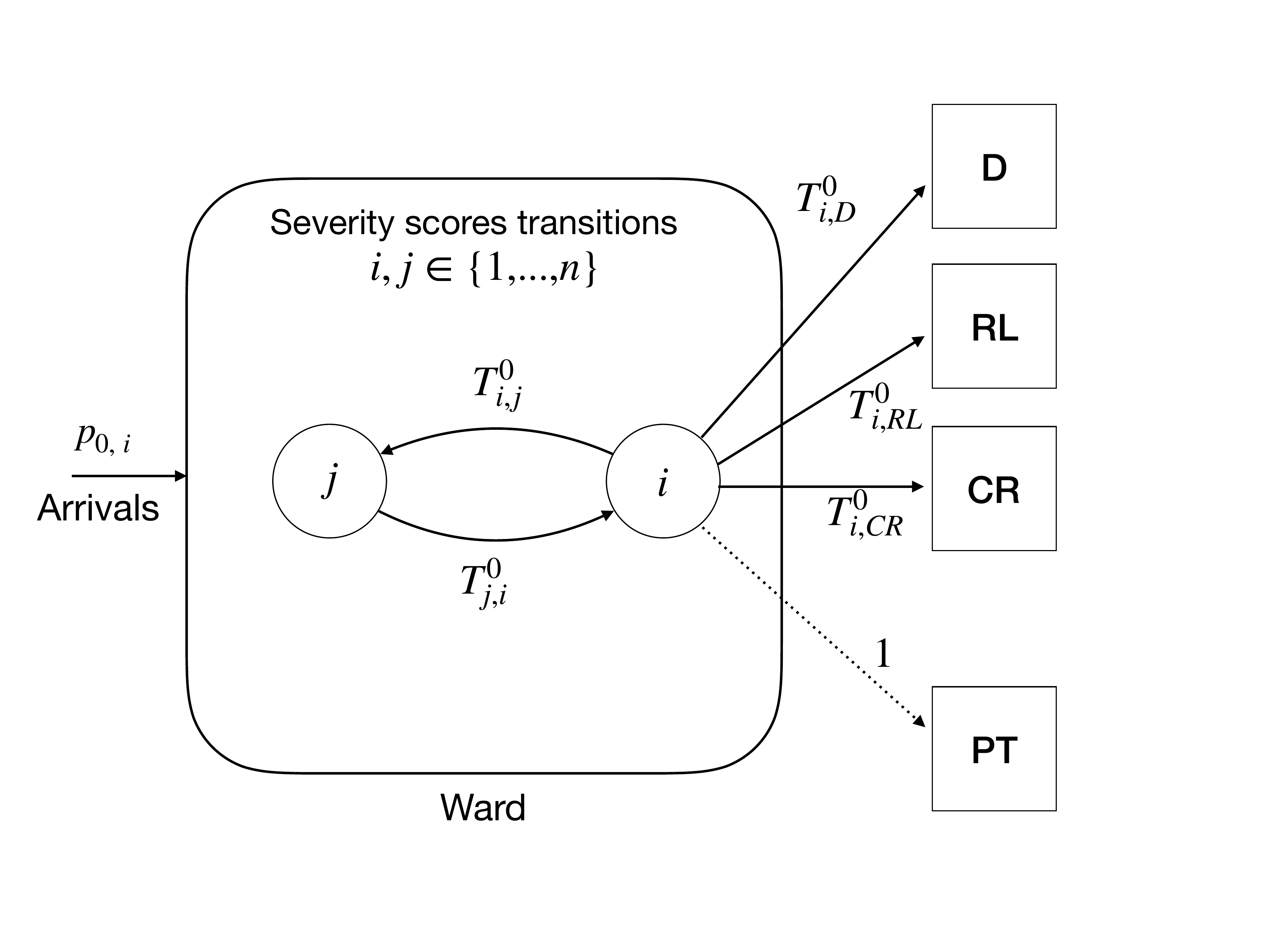}
 \caption{Single-patient MDP model. Terminal states are indicated as square. The patient arrives in the ward with a severity score of $i \in \{1,...,n\}$ with an initial probability $p_{0,i}$. The solid arcs correspond to transitions where no transfer decision is taken (action $0$), and the patient can transition to another severity score $j$ with probability $T_{i,j}^{0}$ or to the terminal states $CR, RL$ or $D$. When the patient is in state $i$ and the decision-maker takes the decision to proactively transfer the patient (action 1), the patient transitions with probability $1$ to the terminal state $PT$ (dashed arc).}
 \label{fig:mdp_model}
 \end{figure}

\paragraph{Policies.}
A policy consists of a map $\pi:\X \rightarrow [0,1]$, where for each severity score $i \in [n]$, $\pi(i) \in [0,1]$ represents the probability of proactive transfer of a patient with current severity score $i$. For terminal states $ i \in \{CR,RL,D,PT\}$, we set $\pi(i)=0$.

A policy $\pi$ is said to be of \textit{threshold} type when
$\pi(i) = 1 \Rightarrow \pi(i+1) = 1, \forall \; i \in [n-1].$ In other words, the policy proactively transfers patients at all severity scores above a given threshold. Therefore, the policy $\pi$ can be described by an integer $\tau \in [1,n+1]$, such that all patients with severity score higher (or equal) than $\tau$ is transferred. Note that a threshold of $\tau=n+1$ means that no patient is proactively transferred, while a threshold of $\tau=1$ means that all patients are proactively transferred. We write $\pi^{[\tau]}$ to denote the threshold policy parametrized by threshold $\tau$. For any threshold policy $\pi$, we write its threshold $\tau(\pi)$.

\paragraph{Transitions.}

The states $i \in [n]$ model the $n$ possible severity scores of a patient. \textit{ Every six hours (a \textit{period}), when a patient is in state $i \in [n]$, the decision-maker can choose to proactively transfer the patient to the ICU (action $1$) or not (action $0$). }
\begin{itemize}
\item
If the patient is proactively transferred from $i \in [n]$, s/he transitions with probability $1$ to the state $PT$. The state $PT$ is a terminal state, the decision-maker receives a terminal reward.
\item
If the patient is not proactively transferred from state $i \in [n]$, the patient transitions to state $j \in [n]$ with probability $T_{ij}^{0}$ in the next $6$ hours (where $\bm{T}^{0}$ is the transition matrix among severity scores). Alternatively, the patient can either transition to state $CR$ with probability $T_{i,n+1}^{0},$ to state $RL$, with probability $T_{i,n+2}^{0}$ or the patient transitions to $D$ with probability $T_{i,n+3}^{0}$. When the patient reaches one of the terminal states -- $CR, RL$ or $D$ -- s/he receives the associated terminal reward.
Note that the patient exits the ward almost surely, assuming that
$ \theta = \min_{i \in [n]} \min \{T_{i,CR}^{0},T_{i,RL}^{0},T_{i,D}^{0} \} >0.$
\end{itemize}
\paragraph{Rewards.} The discount factor $\lambda \in (0,1)$ captures the decreasing importance of future rewards compared to present rewards. The goal of the decision-maker is to pick a policy $\pi$ that maximizes the expected discounted cumulated rewards, defined as $
 R(\pi,\boldsymbol{T}^{0})=E^{\pi, \boldsymbol{T}^{0}} \left[ \sum_{t=0}^{\infty} \lambda^{t}r_{i_{t}a_{t}} \right],
$ where $r_{i_{t}a_{t}}$ is the reward associated with visiting state $i_{t}$ and choosing action $a_{t}$ at time $t \in \N$.

 For each policy $\pi$, we can associate a \textit{value-vector} $\bm{V}^{\pi} \in \R^{n+4}$, defined as
 \begin{equation}\label{eq-valuevec}
V^{\pi}_{i}=E^{\pi, \boldsymbol{T}} \left[ \sum_{t=0}^{\infty} \lambda^{t}r_{_{t}a_{t}}  \; \bigg| \; i_{0} = i \right], \forall \; i \in \{ 1,...,n+4\}.
 \end{equation}

We want our MDP model to capture the trade-off between the benefits of proactive transfers for the patients' health and the costly use of resources and staff in the ICU. We achieve this by choosing the rewards in order to reflect the preference of the decision-maker who is balancing between improving patient outcomes by transferring them to the ICU proactively and the risk of such transfers resulting in a congested ICU.

Without loss of generality, we can consider that all rewards are non-negative.
 We consider a uniform reward across both actions for all states, i.e., $r_{i,0} = r_{i,1}=r_{i}, \forall \; i \in \X.$ Therefore, the reward only depends of the current state, while the action dictates the likelihood of transitioning to states with different rewards.
There is a reward of $r_{W}$ associated with being in the ward: $r_{i} = r_{W}, \forall \; i \in [n].$
If a patient is proactively transferred, s/he transitions to state $PT$ with probability $1$. In state PT, the patient either dies with probability $d_{A}$ or recovers.
Hence, the reward $r_{PT}$ is
$r_{PT} = d_{A} \cdot (r_{PT-D}) + (1-d_{A}) \cdot (r_{PT-RL}),$
where $r_{PT-RL}$ (respectively, $r_{PT-D}$) corresponds to the rewards for a patient recovering (respectively, dying) after having been proactively transferred. The scalar $d_{A}$ is the probability to die when having been proactively transferred and is calibrated to be the same as in the Markov model for the hospital in Section \ref{sec:hosp}.

Similarly, there is a reward of $r_{CR}$ associated with the state $CR$. A patient who crashes (and does not die immediately) will be transferred to the ICU before recovering or dying. We have that,
$r_{CR} = d_{C} \cdot (r_{CR-D}) + (1-d_{C}) \cdot (r_{CR-RL}),$
where $r_{CR-RL}$ (respectively, $r_{CR-D}$) corresponds to the rewards for a patient recovering (respectively, dying) after having been proactively transferred and $d_{C}$ is the probability that a patient dies after  crashing.

We would like to note that the rewards $r_{W}, r_{D},r_{RL},r_{CR}$ and $r_{PT}$ are a priori policy-dependent. For instance, for a policy that proactively transfers many patients, the reward $r_{PT}$ should take into account the (a priori) detrimental increase in ICU occupancy. Moreover, estimating the exact values of the rewards can be challenging (see \cite{mdp-cost-1} for a solution for a two-state hospital model, and see \cite{mdp-cost-2} for a Reinforcement Learning approach). Therefore, we focus on the relative \textit{ordering} of these rewards, in order to capture the trade-off between better health outcomes by proactive transfers and increased congestion in ICU. For the same outcomes (e.g. mortality, LOS), the decision-maker favors the policy which uses the fewest ICU resources. Conditional on the patient recovering and leaving the hospital, it is natural to assume that $r_{RL} \geq r_{PT-RL} \geq r_{CR-RL}$. This is because leaving the hospital after recovering in the ward uses less ICU resources than recovering after being proactively transferred, which in turn uses less ICU resources than if the patient crashes (see \cite{ICU-wenqi} for empirical evidence of this relationship). For similar reasons, $r_{D} \geq r_{PT-D} \geq r_{CR-D}$. We assume that $r_{CR-RL} \geq r_{D}$, since the decision-maker wants to achieve a low in-hospital mortality rate.
{\color{black} Note that we focus here on improving the hospital metrics, as measured by mortality, length-of-stay and ICU occupancy.}

%
%
%
%
%

Note that, as expected, this ordering of the rewards implies that
$r_{RL} \geq r_{PT} \geq r_{CR}.$

In the rest of the paper, for any state $i \in [n]$ we define the \textit{outside option} as
$$out(i)= r_{CR} \cdot T_{i,n+1}^{0} + r_{RL} \cdot T_{i,n+2}^{0} + r_{D} \cdot T_{i,n+3}^{0}.$$
The outside option $out(i)$ represents the expected one-step reward if a patient with severity score $i$ is not proactively transferred and leaves the ward in the next period, i.e., if this patient transitions to one of the states in $CR,RL,$ or $D$ in the next period.

We make a first mild assumption, which has the following interpretation: the total cumulated reward is higher when the patient recovers and leaves after one period in the ward than if s/he stays in the ward forever.
\begin{assumption}\label{assumption-0}
$ \dfrac{r_{W}}{1-\lambda} \leq  r_{W}+\lambda \cdot r_{RL}.$
\end{assumption}
This natural assumption implies the natural relationship that it is most desirable for a patient in the ward to recover and leave the hospital at the next period. This is stated formally in the  following lemma:
\begin{lemma}[Upper bound on the value vector]\label{lem-Vinfty} Let $\boldsymbol{V}^{\pi}$ be the value vector of a policy $\pi$. Under Assumption \ref{assumption-0}, we have $V_{i}^{\pi} \leq r_{W}+\lambda \cdot r_{RL}, \forall \; i \in [n].$
\end{lemma}
We present the proof in Appendix \ref{app-lem-Vinfty}.
We are ready to state the main result in this section,Theorem \ref{th:nom-thr}. Namely, under a mild assumption, the optimal nominal policy in our single-patient MDP is a threshold policy. In particular, we consider the following assumption.
\begin{assumption}\label{assumption-1}
We assume that
\begin{equation}\label{eq:assumption:sub-1}
out(i) \geq out(i+1), \forall \; i \in [n-1].
\end{equation}
and we assume that
 \begin{equation}\label{eq:assumption:sub-2}
     \dfrac{r_{W} + \lambda \cdot r_{PT}}{r_{W} + \lambda \cdot r_{RL}} \geq  \dfrac{\left( \sum_{j=1}^{n} T_{i+1,j}^{0} \right) }{\left( \sum_{j=1}^{n} T_{ij}^{0} \right)}, \forall \; i \in [n-1].
\end{equation}
\end{assumption}
Condition \eqref{eq:assumption:sub-1} implies that $out(i)$ is decreasing in the severity score $i$. This is meaningful since we expect that in practice, the severity score $i$ captures the health condition of a patient, from $i=1$ (as healthy as possible in the hospital) to $i=n$ (a very severe health condition). Therefore, it is reasonable to assume that the outside option of a patient is worse than the outside option of a patient with a better health condition (e.g. the healthier patient is more likely to leave the ward in a better state). {\color{black} Note that  $T^{0}_{i,CR}$ and $T^{0}_{i,DL}$ should be increasing in $i$ and $T^{0}_{i,RL}$ should be decreasing in $i$. Therefore, $out(i+1) \geq out(i)$ is implicitly assuming that the reward $r_{RL}$ (reward for recovering) is significantly larger than $r_{CR}$ (``reward'' for crashing) and $r_{DL}$ (``reward'' for dying), as we expect to hold in practice.} 

Condition \eqref{eq:assumption:sub-2} assumes that the chance of staying in the system  in  risk score $i$ is non-increasing in severity score $i$.  Additionally, the rate at which the chance of staying in the system decreases is higher than the ratio $\dfrac{r_{W} + \lambda \cdot r_{PT}}{r_{W} + \lambda \cdot r_{RL}}$, which captures the preference between the reward for proactively transferring a patient ($r_{W} + \lambda \cdot r_{PT}$) and an optimistic reward in the case that the patient is not transferred ($r_{W} + \lambda \cdot r_{RL}$). Similar to condition \eqref{eq:assumption:sub-1}, we expect condition \eqref{eq:assumption:sub-2} to hold in practice, since patients with more severe health states are more likely to crash or die (and therefore exit the ward) than patients with better health conditions. {\color{black} We are implicitly assuming that the increase in $ T^{0}_{i,CR} + T^{0}_{i,DL}$ outweighs the decrease in $ T_{i,RL}$. when $i$ increases. We can expect this in practice if the changes in these coefficients are of the same order of magnitude.} In particular, it holds for our dataset of nearly $300,000$ patients across $21$ hospitals.

We would like to note that both condition \eqref{eq:assumption:sub-1} and condition \eqref{eq:assumption:sub-2} are homogeneous: they hold if we scale all rewards by the same (positive) scalar. Moreover, condition \eqref{eq:assumption:sub-2} is invariant by translation, i.e., it holds if we add the same scalar to all rewards. However, this is not the case for condition \eqref{eq:assumption:sub-1} (see Lemma \ref{lem:assumption-simple-tr-sc} in Appendix \ref{app:Lemmas-tr-sc}).

\subsection{Optimality of threshold policies.} We are now in a position to characterize structural properties of the optimal transfer policy for our single-patient MDP. Using standard arguments, without loss of generality we can restrict our attention to stationary deterministic policies. We show that there exists an optimal policy that is a threshold policy in the single-patient MDP. Formally, we have the following theorem.
\begin{theorem} \label{th:nom-thr} Under Assumptions \ref{assumption-0} and \ref{assumption-1}, there exists a threshold policy that is optimal in the single-patient MDP.
\end{theorem} The proof relies on proving that all the policies generated by a Value Iteration algorithm (Section 6.3, \cite{Puterman}) are threshold policies, provided that we initialize the algorithm with a threshold policy. Since Value Iteration is known to converge to an optimal policy of the MDP, we can conclude that there exists an optimal policy that is threshold. We present the detailed proof in Appendix \ref{app:pf-main-th}.

We give  some intuition on why Assumptions \ref{assumption-0} and \ref{assumption-1} are sufficient to prove the existence of an optimal policy that is threshold. In order to show that a policy $\pi$ is threshold, it is sufficient to show that for any state $i \in [n-1]$,
$\pi(i)=1 \Rightarrow \pi(i+1)=1.$
We note that the reward associated with a proactive transfer, $\left( r_{W} + \lambda \cdot r_{PT}\right)$, is constant across the severity scores. However, for a patient in severity score $i \in [n-1]$,   the optimal policy is comparing the expected reward associated with a proactive transfer and the expected reward without proactive transfer, which decomposes into $out(i)$ (when the patient transfers to $CR, D$ or $RL$ in the next period) and the expected reward if the patient remains in the ward in the next period. Conditions \eqref{eq:assumption:sub-1} and \eqref{eq:assumption:sub-2} ensure that the reward for not proactively transferring a patient is non-increasing in the severity scores. Indeed, the outside option, i.e., the expected reward when exiting the ward, is non-increasing (condition \eqref{eq:assumption:sub-1}), while the probability to exit the ward is increasing (condition \eqref{eq:assumption:sub-2}).
If the decision-maker chooses to proactively transfer a patient with severity score $i$, this means that the reward for proactive transfer, $r_{W} + \lambda \cdot r_{PT}$, is larger than the reward for not proactively transferring the patient. This last quantity is non-increasing, and therefore, for any severity score $j > i$, the optimal policy should also choose to proactively transfer the patient.

Several remarks are in order.
\begin{remark}
We would like to note that Theorem \ref{th:nom-thr} holds if we replace Assumption \ref{assumption-1} by the following weaker, but less interpretable condition:
\begin{equation}\label{assumption-2}
     \left( \sum_{j=1}^{n} T_{ij}^{0} \right) \cdot ( r_{W}+ \lambda \cdot r_{PT}) + out(i) \geq \left( \sum_{j=1}^{n} T_{i+1,j}^{0} \right) \cdot (r_{W} + \lambda \cdot r_{RL}) +out(i+1), \forall \; i \in [n-1].
\end{equation}
Note that condition \eqref{assumption-2} is implied by Assumption \ref{assumption-1}, simply by summing up \eqref{eq:assumption:sub-1} and \eqref{eq:assumption:sub-2}. Moreover, condition \eqref{assumption-2} is homogeneous and holds under translation of the rewards (see Lemma \ref{lem:assumption-1-tr-sc} in Appendix \ref{app:Lemmas-tr-sc}).
 However, even though condition \eqref{assumption-2} is more general than Assumption \ref{assumption-1}, it is much less interpretable.
 \end{remark}

 \begin{remark} {\color{black}In Appendix \ref{app:counterex}, we show that if Assumption \ref{assumption-1} is not satisfied, there may not exist an optimal policy that is threshold in general.}
  We provide an instance of a single-patient MDP, whose parameters do not satisfy Assumption \ref{assumption-1} where there is no optimal policy that is threshold. The instance has two severity scores $\{1,2\}$, and $out(2) > out(1)$, which violates condition \eqref{eq:assumption:sub-1}. By selecting parameters such that the outside option $out(2)$ is high,  the optimal policy does not proactively transfer the patient in state $2$. However, by setting a sufficiently small  discount factor $\lambda$, the reward in state $1$ becomes unrelated to the reward in state $2$ and the optimal policy will proactively transfer a patient in severity score $1$. We would like to note that condition \eqref{assumption-2} does not hold either for this instance.
 \end{remark}

 \begin{remark}
By assumption, the rewards in the ward and for proactively transferring a patient do not depend of the current severity score. We can relax this assumption and consider a model where the rewards in the ward and the rewards for proactive transfers are heterogeneous across different severity scores and still establish the optimality of a threshold policy under assumptions which are generalizations of Assumptions \ref{assumption-0} and  \ref{assumption-1}. {\color{black} It is straightforward to extend our proof for uniform rewards to the case of non-uniform rewards.}
 \end{remark}

 {\color{black}
\begin{remark}[Incorporating the ICU capacity.]
Our single-patient MDP does not explicitly account for the ICU occupancy or the ICU capacity.
In order to account for capacity constraints, our single-patient MDP penalizes aggressive proactive transfers by assuming a lower reward $r_{PT}$ compared to $r_{RL}$. Our single-patient MDP attempts to find a good balance between a model that is both (i) complex and expressive enough to capture important trade-offs (e.g. transfer or wait, capacity constraint by penalizing the reward $r_{PT}$ for proactive transfer), (ii) sufficiently tractable to allow studying structural properties both in the nominal and in the robust case and (iii) simple enough so that we can evaluate its parameters with satisfying accuracy (see Section \ref{sec:exp}).

In order to take the next step toward incorporating ICU occupancy, we introduce an $N$-patient MDP model, which is a special case of \textit{weakly interacting stochastic processes} \citep{adelman2008relaxations}. There are $N$ single-patient MDPs that run in parallel, with a binding constraint that no more than $m$ patients can be transferred to the ICU at the same time. While the $N$-patient MDP is closer to real hospital operations than the single-patient MDP, this comes at the price of an exponential number of states/actions. In Appendix \ref{app:N-patient-MDP}, we show that our single-patient MDP is a \textit{Lagrangian relaxation} (in the sense of \cite{adelman2008relaxations}) of the $N$-patient MDP, where the reward for proactive transfer $r_{PT}$ is adequately penalized.
Therefore, we study the impact of the variation of $r_{PT}$ onto the threshold of an optimal policy. We relate our work to the \textit{Whittle index} theory \citep{whittle1988restless}  and show that our single-patient MDP is \textit{indexable} (see Appendix \ref{app:whittle}). In more practical terms, this means that the threshold of an optimal policy is \textit{monotonically} increasing from ``\textit{transfer no patient}'' when $r_{PT}=0$ to ``\textit{transfer all patients}'' when $r_{PT}=r_{RL}$. Here, deriving closed-form solutions of the Whittle index appears challenging due to the complexity of our transition model (a patient in severity state $i$ can transition to any other severity state, unlike other models in literature).
Finally, coming back to our single-patient MDP as a relaxation of the $N$-patient MDP, we show that as the number, $m$,  of patients that can be proactively transferred at the same time increases (in the $N$-patient MDP), the threshold of the optimal policy in the single-patient MDP decreases, i.e. an optimal policy transfers more patients.  Thus, one can expect more aggressive proactive transfers when there is ample ICU bed availability versus when capacity is scarce.
\end{remark}
}

\section{Robustness analysis of the single-patient MDP.}\label{sec:rob}
Thus far, we have assumed the health state of a patient evolves according to a Markov chain with known transition matrix $\bm{T}^{0}$. However, the parameters of the transition matrix are estimated from historical data and are subject to statistical estimation errors. At best, we have a noisy estimate of $\bm{T}^{0}$. Therefore, an important practical consideration is to develop an understanding of the impact of small deviations in the hospital parameters. We start by focusing on the robustness of the optimal policy for the single-patient MDP.

\subsection{Robust MDP and model of uncertainty set.}

In a classical MDP, it is assumed that the transition kernel $\bm{T}^{0}$ is known so that one finds a policy that maximizes the expected nominal reward. That is, one solves the optimization problem:
$ \max_{\pi \in \Pi} R(\pi,\bm{T}^{0}).$
In Theorem \ref{th:nom-thr}, we show that the optimal nominal policy for the matrix $\bm{T}^{0}$ is a threshold policy, under certain fairly reasonable assumptions.
In order to tackle model misspecification, we consider a robust MDP framework, where the true transition matrix is unknown. We model the uncertainty as adversarial deviations from the nominal matrix in some \textit{uncertainty set} $\U$, that can be interpreted as a safety region. The goal is to compute a transfer policy that maximizes the worst-case reward over the set of all possible transition matrices $\U$, i.e., our goal is to solve:
$\max_{\pi \in \Pi} \min_{\boldsymbol{T} \in \U} R(\pi,\bm{T}).$

A solution $\pi^{\sf rob}$ to this optimization problem will be called an \textit{optimal robust policy}.

The choice of uncertainty set is important and dictates the conservatism and usefulness of the model.
In this paper we consider a \textit{factor matrix uncertainty set} \citep{Goh,GGC}. In this model we consider that the transition probabilities are convex combination of some factors, which can themselves be uncertain. Such a model allows us to capture correlations across transitions probabilities, unlike rectangular uncertainty sets \citep{Iyengar,Kuhn} that allow unrelated adversarial deviations. In particular, we assume that $\U$ is of the following form:
\begin{equation}\label{eq:fmus}
\U = \left\lbrace  \bm{T} \in \R_{+}^{n \times (n+3)} \; \bigg| \; T_{ij} = \sum_{\ell =1}^{r} u^{\ell}_{i}w_{\ell,j} \forall \; (i,j) \in [n] \times [n+3] ,  \; \boldsymbol{w}_{\ell} \in \W{^{\ell}}, \forall \; \ell \in [r] \right\rbrace
\end{equation}
where $ \; \boldsymbol{u}_{1},..., \boldsymbol{u}_{n}$ are fixed vectors in $\R_{+}^{r}$ and $\W^{\ell}, \ell = 1,...,r$ are  convex, compact subsets of $\R_{+}^{(n+3) \times r}$ such that
$\sum_{\ell=1}^{r} u^{\ell}_{i}  = 1, \; \forall \; i \in [n],
\sum_{j=1}^{n+3} w_{\ell,j}  = 1, \; \forall \; \ell \in [r],\forall \; \bm{w}_{\ell} \in \W^{\ell}.$

To understand better the implications of this uncertainty set, consider $\bm{T}$, a transition matrix in our factor matrix uncertainty set $\U$. Each of the rows of the matrix $\bm{T}$ is a convex combination of the \textit{factors} $\bm{w}_{1}^{\top},...,\bm{w}_{r}^{\top} \in \R^{1 \times (n+3)}$. Therefore, for $\boldsymbol{U} = (u_{i}^{\ell})_{(i,\ell)} \in \R^{n \times r}$ and $\boldsymbol{W} = (\boldsymbol{w}_{1},...,\boldsymbol{w}_{r}) \in \R^{(n+3) \times r},$ we have $\boldsymbol{T} = \boldsymbol{UW}^{ \top}$.
This model of uncertainty set is very general, and covers the case of $(s,a)$-rectangular uncertainty sets \citep{Iyengar}, i.e., the case where all the rows of the transition matrix $\bm{T}$ are chosen independently, when $r=n$ (where $n$ is the number of rows of $\bm{T}$). This model is able to incorporate relationships in the transitions from various severity scores which may arise since the health dynamics of a patient is likely to be influenced by some underlying common factors, such as demographics, past medical histories, current blood sugar level, etc. As different rows of a matrix $\bm{T} \in \U$ are convex combinations of the same $r$ factor vectors, we can use our factor matrix uncertainty set to model correlations between the probability distributions related to different states when $r$ is smaller than $n$.

We assume that the nominal matrix $\bm{T}^{0}$ satisfies Assumption \ref{assumption-1}. Our uncertainty set $\U$ models small parameters variations from $\bm{T}^{0}$. Therefore, it is reasonable to assume that all the matrices in $\U$ will satisfy Assumption \ref{assumption-1}. In particular, we assume:
\begin{assumption}\label{assumption-T-all-U}
Every matrix $\bm{T}$ in $\U$ satisfies Assumption \ref{assumption-1}.
\end{assumption}
{\color{black}
\begin{remark}[Verifying Assumptions \ref{assumption-0}-\ref{assumption-T-all-U}]
We now describe  how to numerically verify that Assumptions \ref{assumption-0}-\ref{assumption-T-all-U} hold.
Assumption \ref{assumption-0} is simply an inequality on the set of rewards. Similarly, verifying Assumption \ref{assumption-1} requires the computation of the cumulative sums $\left( \sum_{j=1}^{n} T_{ij} \right)_{i}$. In order to verify Assumption \ref{assumption-T-all-U} for an uncertainty set $\U$, we can check the nonnegativity of
\[ \min \; \left\lbrace  \dfrac{r_{w}  + \lambda r_{PT}}{r_{W} + \lambda r_{RL}} \sum_{j=1}^{n} T_{ij} - \sum_{j=1}^{n} T_{i+1,j}  \; \bigg| \; \bm{T} \in \U \right\rbrace,\]
for each severity condition $i \in [n]$.
As the objective function is linear in the matrix $\bm{T}$, this optimization program can be solved efficiently when the uncertainty set $\U$ is defined by linear, or convex quadratic or conic inequalities. Note that the wider and the more unconstrained the uncertainty set, the more difficult it is to satisfy Assumption \ref{assumption-T-all-U}.
\end{remark}}
For factor model uncertainty sets, the authors in \cite{Goh} show that one can compute the worst-case reward of a given policy. \cite{GGC} show that an optimal robust policy can be chosen to be deterministic and provide an efficient algorithm to compute an optimal robust policy. Moreover, they show that for $\U$ as in \eqref{eq:fmus},  the  \textit{robust maximum principle} holds. Since our analysis relies on the maximum principle, we state it formally for completeness.

Let  $\boldsymbol{v}^{\pi,\boldsymbol{T}}$ be the value vector of the decision-maker when s/he chooses policy $\pi$ and the adversary chooses factor matrix $\bm{T}=\bm{UW}^{\top}$, defined by the Bellman Equation:
$$v^{\pi,\boldsymbol{T}}_{i} = r_{\pi,i} + \lambda \cdot (1-\pi(i)) \cdot \sum_{\ell=1}^{r} u^{\ell}_{i} \boldsymbol{w}_{\ell}^{\top} \boldsymbol{v^{\pi,\boldsymbol{W}}} + \lambda \cdot \pi(i) \cdot r_{PT}, \forall \; i \in [n].$$
For any state $i \in [n]$, the scalar $v_{i}^{\pi,\bm{T}}$ represents the infinite horizon discounted expected reward, starting from state $i$.
\begin{theorem}[Robust Maximum Principle \citep{GGC}]\label{th:robust-max-principle}
Let $\U$ be a factor matrix uncertainty set as in \eqref{eq:fmus}.
\begin{enumerate}
\item Let $\boldsymbol{\hat{T}} \in \U$ and $\hat{\pi} \in \arg \max_{\pi \in \Pi} \; R(\pi,\boldsymbol{\hat{T}}).$
Then
\begin{equation}\label{eq:max-prin-1}
v^{\pi,\boldsymbol{\hat{T}}}_{i} \leq v^{\hat{\pi},\boldsymbol{\hat{T}}}_{i}, \forall \; \pi \in \Pi, \forall i \in [n].
\end{equation}
\item Let $\hat{\pi}$ be a policy and $\boldsymbol{\hat{T}} \in \arg \min_{\boldsymbol{T} \in \U} \; R(\hat{\pi},\boldsymbol{T}).$
Then
\begin{equation}\label{eq:max-prin-2}v^{\hat{\pi},\boldsymbol{\hat{T}}}_{i} \leq v^{\hat{\pi},\boldsymbol{T}}_{i}, \forall \; \boldsymbol{T} \in \U, \forall \; i \in [n].
\end{equation}
\item Let $(\pi^{*},\boldsymbol{T}^{*}) \in \arg \; \max_{\pi \in \Pi} \; \min_{\boldsymbol{T} \in \U} \; R(\pi,\boldsymbol{T}).$
For all policy $\hat{\pi}$, for all transition matrix $\boldsymbol{\hat{T}} \in \arg \min_{\boldsymbol{T} \in \U} \; R(\hat{\pi},\boldsymbol{T})$, we have  \begin{equation}\label{eq:max-prin-3}v^{\hat{\pi},\boldsymbol{\hat{T}}}_{i} \leq v^{\pi^{*},\boldsymbol{T}^{*}}_{i}, \forall i \in [n].
\end{equation}
\end{enumerate}
\end{theorem}
Inequality \eqref{eq:max-prin-1} implies that in a classical MDP setting, the value vector of the optimal nominal policy is component-wise higher than the value vector of any other policy. Therefore, for any state, the  nominal expected reward obtained when the decision-maker follows the optimal nominal policy is higher than the nominal expected reward obtained when the decision-maker follows any other policy. Following Inequality \eqref{eq:max-prin-2}, when a policy is fixed but the transition matrix varies in the uncertainty set $\U$, the worst-case value vector of the policy is component-wise lower than the value vector of the policy for any other transition matrix. Finally, when we consider an optimal robust policy, Inequality \eqref{eq:max-prin-3} implies that the worst-case value vector of the optimal robust policy is component-wise higher than the worst-case value vector of any other policy. Therefore, the optimal robust policy is maximizing the worst-case expected reward \textit{starting from any state}.

\subsection{Theoretical guarantees.}
%

We show that under Assumption \ref{assumption-0} and Assumption \ref{assumption-T-all-U}, the optimal robust policy is a threshold policy.
Moreover, we show that the threshold of the optimal robust policy is smaller than the threshold of the optimal nominal policy. Therefore, the optimal robust policy is more aggressive in proactively transferring the patients.
In particular, we start with the following theorem.
\begin{theorem}\label{th:thr-all-T} Under Assumptions \ref{assumption-0} and \ref{assumption-T-all-U}, there exists an optimal robust policy that is threshold.
\end{theorem}
\begin{proof}
Following Assumption \ref{assumption-T-all-U} and Theorem \ref{th:nom-thr}, for any transfer policy $\tilde{\pi} \in \Pi,$
$$ \exists \; \bm{\tilde{T}} \in \U, \; \tilde{\pi} \in \arg \max_{\pi \in \Pi} R(\pi,\bm{\tilde{T}}) \Rightarrow \tilde{\pi} \text{ is a threshold policy.}$$
Theorem 4.2 in \cite{GGC} shows that
$$(\pi^{\sf rob},\boldsymbol{T}^{\sf rob}) \in \arg \max_{\pi \in \Pi} \min_{\boldsymbol{T} \in \U} R(\pi,\boldsymbol{T}) \iff \pi^{\sf rob} \in \arg \max_{\pi \in \Pi} R(\pi,\boldsymbol{T}^{\sf rob}).$$
Since the matrix $\boldsymbol{T}^{\sf rob}$ belongs to the uncertainty set $\U$, it satisfies Assumption \ref{assumption-1} by Assumption \ref{assumption-T-all-U} and therefore the optimal robust policy $\pi^{\sf rob}$ is threshold.
\end{proof}

This result highlights the critical role of threshold policies. Not only is the optimal nominal policy a threshold policy (Theorem \ref{th:nom-thr}), but the optimal robust policy, i.e., the policy with the highest worst-case reward, is also a threshold policy. It is natural to compare the thresholds of the optimal nominal and the optimal robust policies.
Our next result states that the threshold of the optimal robust policy $\pi^{\sf rob}$ is always \textit{lower} than the threshold of the optimal nominal policy $\pi^{\sf nom}.$
\begin{theorem}\label{th:rob-th-nom} Under Assumptions \ref{assumption-0} and \ref{assumption-T-all-U}, we have $\tau(\pi^{\sf rob}) \leq \tau(\pi^{\sf nom}),$ where  $\pi^{\sf rob}$ is the optimal robust policy and $\pi^{\sf nom}$ is the optimal nominal policy.
\end{theorem}
\begin{proof}
Let $\hat{\Pi}$ be the set of policies that are optimal for some transition kernel in $\U$:
$ \hat{\Pi} = \{ \pi  \; | \;  \exists \; \boldsymbol{T} \in \U, \pi \in \arg \max_{\pi \in \Pi} R(\pi,\boldsymbol{T}) \}. $ Note that $\pi^{\sf nom} \in \hat{\Pi}$. We will prove that $\tau(\pi^{\sf rob}) \leq \tau(\pi), \forall \; \pi \in \hat{\Pi}. $

Following Theorem \ref{th:thr-all-T}, we can pick $\pi^{\sf rob}$ to be an optimal robust policy that is a threshold policy. We denote $\boldsymbol{T}^{\sf rob}$ a matrix in $\U$ such that
$(\pi^{\sf rob},\boldsymbol{T}^{\sf rob}) \in \max_{\pi \in \Pi} \min_{\boldsymbol{T} \in \U} R(\pi,\boldsymbol{T}).$
 Let $\hat{\pi} \in \hat{\Pi}$. There exists a transition matrix $\boldsymbol{\hat{T}} \in \U$ such that
$\hat{\pi} \in \arg \max_{\pi \in \Pi} R(\pi,\boldsymbol{\hat{T}}).$
Let us assume that $\hat{\pi}(i)=1$ for some $i \in [n].$ We will prove that $\pi^{\sf rob}(i)=1.$
We have
\begin{align}
r_{W} + \lambda \cdot r_{PT} & > r_{W} + \lambda \cdot \boldsymbol{\hat{T}}_{i,\cdot}^{\top}\boldsymbol{V}^{\hat{\pi},\boldsymbol{\hat{T}}} \label{ineq:prop-1} \\
& \geq r_{W} + \lambda \cdot \boldsymbol{\hat{T}}_{i,\cdot}^{\top}\boldsymbol{V}^{\pi^{\sf rob},\boldsymbol{\hat{T}}}, \label{ineq:prop-2}
\end{align}
where Inequality \eqref{ineq:prop-1} follows from the Bellman Equation for the MDP with transition matrix $\boldsymbol{\hat{T}}$ and Inequality \eqref{ineq:prop-2} follows from Inequality \eqref{eq:max-prin-1} of Theorem \ref{th:robust-max-principle}:
$$ \hat{\pi} \in \arg \max_{\pi \in \Pi} R(\pi,\boldsymbol{\hat{T}}) \Rightarrow V^{\hat{\pi},\boldsymbol{\hat{T}}}_{j} \geq V^{\pi,\boldsymbol{\hat{T}}}_{j},  \forall \; j \in [n], \forall \; \pi \in \Pi. $$

Now for the sake of contradiction let us assume that $\pi^{rob}(i)=0.$ Therefore,
\begin{equation}\label{eq:contr-1}
r_{W} + \lambda \cdot \boldsymbol{\hat{T}}_{i,\cdot}^{\top}\boldsymbol{V}^{\pi^{\sf rob},\boldsymbol{\hat{T}}} = V^{\pi^{\sf rob},\boldsymbol{\hat{T}}}_{i}.
\end{equation}
Therefore, if $\pi^{\sf rob}(i)=0$, we can conclude that
\begin{align}
r_{W} + \lambda \cdot r_{PT} & > V^{\pi^{\sf rob},\boldsymbol{\hat{T}}}_{i} \label{ineq:prop-3} \\
& \geq V^{\pi^{\sf rob},\boldsymbol{T}^{\sf rob}}_{i}, \label{ineq:prop-4}
\end{align}
where the strict Inequality \eqref{ineq:prop-3} follows from \eqref{eq:contr-1} and \eqref{ineq:prop-1}, and Inequality \eqref{ineq:prop-4} follows from \eqref{eq:max-prin-2} in the robust maximum principle:
$$
\boldsymbol{T}^{\sf rob} \in \arg \; \min_{\boldsymbol{T} \in \U} R(\pi^{\sf rob},\boldsymbol{T})
 \Rightarrow V^{\pi^{\sf rob},\boldsymbol{\hat{T}}}_{j} \geq V^{\pi^{\sf rob},\boldsymbol{T}^{\sf rob}}_{j}, \forall \; j \in [n].
$$
We can therefore conclude that
\begin{equation}\label{eq:contr-3}
r_{W}+ \lambda \cdot r_{PT} > V^{\pi^{\sf rob},\boldsymbol{T}^{\sf rob}}_{i}.
\end{equation}
But since $\pi^{\sf rob}$ is an optimal robust policy, we know following Theorem \ref{th:robust-max-principle} that
$ \pi^{\sf rob} \in \arg \; \max_{\pi \in \Pi} R(\pi,\boldsymbol{T}^{\sf rob}).$
Therefore, from the Bellman Equation we know that $\pi^{\sf rob}(i)=1$ if $r_{W} + \lambda \cdot r_{PT} > r_{W}+\lambda \cdot \boldsymbol{T}_{i,.}^{\sf rob \; \top}\boldsymbol{V}^{\pi^{\sf rob},\boldsymbol{T}^{\sf rob}}$ and $\pi^{\sf rob}(i)=0$ if $r_{W} + \lambda \cdot r_{PT} \leq r_{W}+\lambda \cdot \boldsymbol{T}_{i,.}^{\sf rob \; \top}\boldsymbol{V}^{\pi^{\sf rob},\boldsymbol{T}^{\sf rob}}$.
This implies that $V^{\pi^{\sf rob},\boldsymbol{T}^{ \sf rob}}_{i} \geq r_{W} + \lambda \cdot r_{PT},$
which contradicts Inequality \eqref{eq:contr-3}, and therefore it is impossible that $\pi^{\sf rob}(i)=0.$ Since $\pi^{\sf rob}$ is a deterministic policy, $\pi^{\sf rob}(i) \neq 0 \Rightarrow \pi^{\sf rob}(i)=1.$
We have proved that if $\hat{\pi}(i)=1$ for some $\hat{\pi}$ in $\hat{\Pi}$ and some $i \in [n]$, then $\pi^{\sf rob}(i)=1$. Therefore, we can conclude that
$\tau(\pi^{\sf rob}) \leq \tau(\pi), \forall \; \pi \in \hat{\Pi}.$
Since $\pi^{\sf nom}$ is the optimal policy for the nominal transition kernel $\bm{T}^{0}$,  we can conclude that $\pi^{\sf nom} \in \hat{\Pi}$ and therefore in particular
$\tau(\pi^{rob}) \leq \tau(\pi^{\sf nom}).$ 
\end{proof}

Theorem \ref{th:rob-th-nom} highlights the crucial role of threshold policies in ICU admission decision-making. In the framework of our single-MDP for modeling the patient dynamics, both an optimal nominal policy and an optimal robust policy can be found in this class of simple and implementable policies. Moreover, there exists a natural ordering on the threshold of a nominal policy and a policy that accounts for parameter misspecification. In particular, the robust optimal policy is more aggressive in proactively transferring patients.
%

\section{Numerical experiments.}\label{sec:exp} In this section, we utilize real data from 21 Northern California Kaiser Permanente hospitals to examine the potential implications of our theoretical results in practice. We utilize this data to estimate the nominal parameters and uncertainty set of our hospital model (Figure \ref{fig:hosp_model}) and our single-patient MDP (Figure \ref{fig:mdp_model}). We then compare the performance of the  optimal nominal and optimal robust policies on several metrics of interest: mortality, Length-Of-Stay (LOS) and average ICU occupancy.

\subsection{Dataset.}
{\color{black} Our retrospective dataset consists of 296,381 unique patient hospitalizations across 21 Northern California Kaiser Permanente hospitals}. For each hospitalization, we  have patient-level data which is assigned at the time of hospital admission: age, gender, admitting hospital, admitting diagnosis, classification of diseases codes, and three scores that quantify the severity of the illness of the patient (CHMR, COPS2, LAPS2, see \cite{ICU-wenqi} for more details). During the patient's hospitalization, we can track each unit (i.e., ICU, Transitional Care Unit, general medical-surgical ward, operating room, or post-anesthesia care unit) the patient stayed in and when. Additionally, we have a sequence of early warning scores, known as Advance Alert Monitor (AAM) scores, that are updated every six hours. This early warning score uses the LAPS2, COPS2, individual vital signs and laboratory tests, interaction terms, temporal markers, and location indicators to estimate the probability of in-hospital deterioration (requiring ICU transfer or leading to death on the ward) within the next 12 hours, with an alert issued at a probability of $\geq 8 \%$. {\color{black} These scores have demonstrated their ability to accurately predict deterioration \citep{Escobar-2012}, and we use them use a proxy for the severity condition of the patients}.  Similar to \cite{ICU-wenqi}, we focus on medical patients who were admitted to the hospital through the emergency department (this comprises more than $60\%$ of all patients).
{\color{black} We remove 11,463 hospitalizations with missing gender or inpatient unit code, time inconsistencies (e.g. arrival after discharge, missing discharge time). We also drop patients involved in hospital transfers (5,781 patients). Our final dataset consists of 174,632 hospitalizations, each corresponding to a patient trajectory that evolves across  $n=10$ severity scores, possible ICU visit(s), and terminates with the patient either recovering and leaving the hospital or dying and leaving the hospital. }
{\color{black} Summary statistics (partition, mortality rate, average length-of-stay, etc.) for the patients across the different severity scores are given in Appendix \ref{app:detail_hosp}. }

\subsection{Transition matrix and model of uncertainty.}
We first calibrate the transition matrix which determines the evolution of patient severity score while in the general ward.
\subsubsection{Nominal transition matrix.}
We use the AAM scores as our severity scores.
The matrix $\bm{T}^{0}$ has dimension $n \times (n+3)$ where $n=10$ is the number of severity scores.  $\bm{T}^{0}$ is constructed as follows. Let $i \in [n]$ and $j \in [n] \bigcup \{ CR,RL,D\}$. The coefficient $T_{ij}^{0}$ represents the probability that a patient in severity score $i$ will transfer to  state $j$ in the next period. We use the empirical mean as the nominal value for $T^{0}_{ij}$.
We use the method in \cite{conf_interval} to obtain the  $95 \%$-confidence intervals for the matrix $\bm{T}^{0}$:
\begin{align}\label{eq:conf_unrelated}
[T_{ij}^{0} - \alpha_{i}, T_{ij}^{0}+2 \cdot \alpha_{i}], \forall \; (i,j) \in [n] \times [n+3].
\end{align}
This expression highlights the skewness of the confidence intervals, which follows from the skewness of the multinominal distribution. Also, note that for a given severity score $i \in [n]$, the parameter uncertainty in $T_{ij}^{0}$ is uniform across all $j \in [n+3]$. See Appendix \ref{app:values-nom-parameters} for the values for $\alpha_{1},...,\alpha_{n},$ which are in the order of $10^{-4}$ to $10^{-3}$.

\subsubsection{Nominal factor matrix.}
In order to construct a factor matrix uncertainty set \eqref{eq:fmus},  we need to compute the coefficients $(u_{i}^{\ell})_{(i,\ell)} \in \R^{n \times r}$, the nominal factors $\bm{W} = (\bm{w}_{1},...,\bm{w}_{r}) \in \R_{+}^{(n+3) \times r}$ such that $\bm{T}^{0} \approx \bm{UW}^{\top}$, and the confidence regions $\W^{i}$ for each factor $\bm{w}_{i}, i=1,...,r$.
To do this, we solve the following Nonnegative Matrix Factorization (NMF) problem:
\[ \min \; \{ \| \boldsymbol{T}^{0} - \boldsymbol{UW}^{\top} \|_{2}^{2} \; | \boldsymbol{Ue}_{r} =\boldsymbol{e}_{n}, \boldsymbol{e}_{n+3}^{\top}\boldsymbol{W} = \boldsymbol{e}_{r}, \boldsymbol{U} \in \R^{n \times r}_{+}, \boldsymbol{W} \in \R^{(n+3) \times r}_{+}\}.\]

This is a non-convex optimization problem. However, there exist  fast algorithms for efficiently computing local minima. We adapt the block-coordinate descent method of \cite{Wotao}, starting from $10^{6}$ different random matrices and keep the best solution\footnote{This takes less than 5 minutes on a laptop with 2.2 GHz Intel Core i7 and 8 GB of RAM.}. For $r=8$, our solution $\boldsymbol{\hat{T}}=\boldsymbol{U\hat{W}}^{\top}$ achieves the following errors: $ \| \boldsymbol{T}^{0} - \boldsymbol{\hat{T}} \|_{1} = 0.0811, \| \boldsymbol{T}^{0} - \boldsymbol{\hat{T}} \|_{\infty} = 0.0074, \| \boldsymbol{T}^{0} - \boldsymbol{\hat{T}} \|_{\sf relat,\boldsymbol{T}^{0}} = 0.3385,$
    where $\| \cdot \|_{\sf relat, \boldsymbol{T}^{0}}$ stands for the maximum relative deviation from a parameter of $\boldsymbol{T}^{0}$:
    $$\| \boldsymbol{T}^{0} - \boldsymbol{\hat{T}} \|_{\sf relat,\boldsymbol{T}^{0}} = \max_{ (i,j) \in [n] \times [n+3]} \dfrac{|T_{ij}^{0} - \hat{T}_{ij}|}{T_{ij}^{0}}.$$
Table \ref{tab:abs-relat} summarizes the errors across the $n\times(n+3) = 130$ elements of $\boldsymbol{T}^{0}$.
\begin{table}[H]
\center
\begin{tabular}{|c|cccc|}
\hline
                   & max. & mean    & median  & 95\% percentile \\
                   \hline
absolute deviation & 0.0074  & 0.0006 & 0.0003 & 0.0022          \\
relative deviation & 0.3385  & 0.0565  & 0.0204  & 0.2656   \\
\hline
\end{tabular}
\caption{Statistics of the absolute and relative deviations of $\hat{T}_{ij}$ from $T^{0}_{ij}$, for all $(i,j) \in [n] \times [n+3]$.}
\label{tab:abs-relat}
\end{table}

{\color{black} Recall that we utilize a factor matrix model of uncertainty in order to capture correlations in transitions due to underlying characteristics such as genetics or demographics. Therefore, we expect the rank to be smaller than  the number of states, $r < n$. We choose $r=8$. This is the smallest integer for which we were able to find a nonnegative matrix factorization (of this rank) belonging to the confidence intervals. We give details about our simulations for $r=7$ in Appendix \ref{app:r=7}, for which we obtain similar insights as for $r=8$.}

The maximum \textit{absolute} deviation between $T^{0}_{ij}$ and $\hat{T}_{ij}$ is less than 0.01 (0.0891 instead of 0.0817). Moreover, the maximum relative deviation is about $34 \%$, with a coefficient of $4.8527 \cdot 10^{-4}$ instead of $7.34\cdot 10^{-4}$ for $\bm{T}^{0}$. This occurs for $T^{0}_{3,9}$, which represents a sudden,  dramatic, and relatively rare health deterioration from state $3$ to state $9$.

By construction, any two rows $\bm{\hat{T}}_{i_{1},\cdot}$ and $\bm{\hat{T}}_{i_{2},\cdot}$ are convex combinations of the same factors $\bm{\hat{w}}_{1},...,\bm{\hat{w}}_{r}$, with coefficients $\bm{U}_{i_{1},\cdot}$ and $\bm{U}_{i_{2},\cdot}$.
 We compute $\Delta(i_{1},i_{2})=\| \bm{U}_{i_{1},\cdot} - \bm{U}_{i_{2},\cdot} \|_{1}$ as a measure of relatedness between the uncertainty on $\bm{T}_{i_{1},\cdot}$ and $\bm{T}_{i_{2},\cdot}$.
We note that our NMF decomposition captures the intuition that close severity scores are more related than very different severity scores. To see this, consider any severity score $i \in [n]$ and any two alternative severity scores $(j,k) \in [n] \times [n]$ which are different from $i$. Then, we observe that
$ | i-j | < | i-k | \Rightarrow \Delta(i,j) < \Delta(i,k).$

\paragraph{Errors related to the confidence intervals.} The intent of the structured nominal transition matrix (and the subsequent uncertainty sets) is to capture the parameter uncertainty inherent in the estimation process from real data.  As such, it is of interest to understand whether our nominal matrix is consistent with the confidence intervals of our parameter estimates.  We consider a relative error of $\boldsymbol{\hat{T}}$ compared to $\boldsymbol{T}^{0}$, measured in terms of the confidence bounds $\alpha_{1},...,\alpha_{n}$. In particular, we compute the ratios \begin{equation}\label{eq:ratios}
    ratio_{(i,j)}= \dfrac{T_{ij}^{0} - \hat{T}_{ij}}{\alpha_{i}},\forall \; (i,j) \in [n] \times [n+3].
    \end{equation}
 For $r=8$, we find that all coefficients are in the confidence intervals as defined by \eqref{eq:conf_unrelated}, i.e., $ratio_{(i,j)} \in [-1,2], \forall \; (i,j) \in [n] \times [n+3]$.  The mean over $(i,j)$ of the absolute values of the ratios \eqref{eq:ratios} is $0.2345$, with a median of $0.1579$. Moreover, $95 \%$ of these absolute values are below $0.6729$. Therefore,  $\bm{\hat{T}}$ (our NMF solution of rank $8$) is a plausible approximation for $\bm{T}^{0}$.


For completeness, we also compute the solutions to the NMF optimization problem for lower ranked matrices: $r =5,6,7$. While the errors in the $L_{1}-$ and $L_{\infty}-$norms remain small, the relative errors increased substantially, up to $0.43$ for $r=7$, $0.98$ for $r=6$ and $5.8$ for $r=5$.  For rank $7$, we have $10$ coefficients outside of the confidence intervals, with a maximum deviation of $4.840 \cdot \alpha_{9}$ for $T_{9,9}^{0}$. Despite that one coefficient being well out of its confidence interval, we find that  the mean of the absolute value of the ratios  is $0.4818$, with a median at $0.2380$. Therefore, the NMF solution for $r=7$ also appears to be a reasonable approximation for $\bm{T}^{0}.$ However, it does not seem reasonable to decrease the rank even further. For instance, for a rank $r=5$, our NMF solution has $70$ coefficients that are outside the $95 \%$ confidence intervals, with a maximum ratio of $48.4931$. Similarly, for $r=6$, there are still $54$ coefficients outside of the confidence intervals, with a maximum ratio of $ 44.0267$. Therefore, in the rest of the paper we will focus on NMF solutions corresponding to rank $r=8$. We also present detailed experiments for $r = 7$ in Appendix \ref{app:r=7}.

\subsubsection{Choice of uncertainty sets.} We consider several uncertainty sets for our analysis:
\begin{itemize}
\item {$\U_{\min}$:}  We consider a factor matrix uncertainty set based on the 95\% confidence interval in the most optimistic manner. Specifically, for $\alpha_{\min} = \min_{i \in [n]} \alpha_{i}$, we consider
\[\U_{\min} = \left\lbrace \boldsymbol{T}  \; \bigg| \; \boldsymbol{T} = \boldsymbol{UW}^{\top}, \boldsymbol{W} \in \W_{\min} \right\rbrace, \]
where $ \W_{\min} = \W^{1}_{\min} \times ... \times \W^{r}_{\min},$
$$ \W^{\ell}_{\min} = \left\lbrace \boldsymbol{w}_{\ell} \; \bigg| \; \forall \; j \in [n+3], w_{\ell,j} - \hat{w}_{\ell,j} \in [ - \alpha_{\min}, + 2 \cdot \alpha_{\min}], \boldsymbol{w}_{\ell} \geq \boldsymbol{0}, \boldsymbol{w}^{\top}_{\ell}\boldsymbol{e}_{n+3}=1 \right\rbrace, \forall \; \ell \in [r].$$
Specifically, the  deviation on each component of the factor vectors must be within $ [ - \alpha_{\min},  2 \cdot \alpha_{\min}]$. This implies that $\bm{T} = \bm{U}\bm{W}^{T}$ is within $ [ - \alpha_{\min},  2 \cdot \alpha_{\min}]$ from the matrix $\bm{\hat{T}}$ (in $\| \cdot \|_{\infty}$).

{\color{black}
Note that while it is in principle possible to construct $\U_{\max}$, where the maximum deviation on each component component of the factor vectors must be within $ [ - \alpha_{\max},  2 \cdot \alpha_{\max}]$, this would result in worst-case matrices where most coefficients are outside of the confidence intervals, contrary to $\U_{\min}$; see details in the next section.}
\item {$\U_{\sf emp}$:} We also consider another possibly less restrictive uncertainty set that is constructed empirically from the 95\% confidence intervals. To do this, we generate 95\% confidence intervals of the \emph{factor vectors}. First, we sample $q$ transition matrices $\bm{T}^{1},...\bm{T}^{q}$ uniformly in the $95 \%$ confidence intervals around $\bm{T}^{0},$ for $q=10^{4}$. For each sampled matrix, we use Nonnegative Matrix Factorization to compute factor vectors $\bm{W}^{1},...,\bm{W}^{q}$ such that $\bm{T}^{m} \approx \bm{UW}^{m \; \top},m=1,...,q.$ Let $\sigma^{\ell}_{j}$ be the empirical standard deviations of each coefficients $w^{\ell}_{j}$, for $(\ell,j) \in [r] \times [n+3]$, from the resulting $\bm{W}^{1},...,\bm{W}^{q}$. We then define the uncertainty set
 $\U_{\sf emp} = \{ \boldsymbol{T}  \; | \; \boldsymbol{T} = \boldsymbol{UW}^{\top}, \boldsymbol{W} \in \W_{\sf emp} \},$ where
$ \W_{\sf  emp} = \W^{1}_{\sf  emp} \times ... \times \W^{r}_{\sf  emp}$ represents the bootstrapped 95\% confidence intervals for the factor vectors. That is: $$ \W^{\ell}_{\sf  emp} = \left\lbrace \boldsymbol{w}_{\ell} \; \bigg| \; \forall \; j \in [n+3], |w_{\ell,j} - \hat{w}_{\ell,j}| \leq \sigma^{\ell}_{j} \cdot \dfrac{1.96}{\sqrt{q}}, \boldsymbol{w}_{\ell} \geq \boldsymbol{0}, \boldsymbol{w}^{\top}_{\ell}\boldsymbol{e}_{n+3}=1 \right\rbrace, \forall \; \ell \in [r].$$

\item {$\U_{\sf sa}$:} Finally,  we also consider the following $(s,a)$-rectangular uncertainty set, where transitions from each state can be chosen unrelated to the transitions out of any other states:
\[\U_{\sf sa} = \left\lbrace \boldsymbol{T} \; \bigg| \; T_{ij} - T^{0}_{ij} \in [- \alpha_{i},+ 2 \cdot \alpha_{i} ], \sum_{j=1}^{13} T_{ij} = 1, \forall \; i \in [n] \; \right\rbrace.\] This uncertainty set is unable to capture the fact that there are likely characteristics which introduce correlations  across different states.
\end{itemize}
\subsection{Robustness analysis for the single-patient MDP.}\label{sec:exp-MDP}
We first give the details about the parameters of our single-patient MDP.
\subsubsection{Choice of MDP Parameters.}\label{sec:MDP-parameters}

\paragraph{Nominal transition matrix.} The probability that the patient transitions from severity score $i \in [n]$ to next state $j \in [n+3]$ is $T^{0}_{ij}$.
The probability that a patient dies after having crashed in the ward is given by $d_{C}=0.4761$, and is estimated by sample mean in our dataset. The probability that a patient dies after having been proactively transferred to the ICU is estimated similarly and is $d_{A}=0.0009$, which is significantly lower than $d_{C}$.
\paragraph{Initial distribution and rewards.}
We set the initial distribution $\bm{p}_{0} \in \R^{n+4}_{+}$ as the long-run average occupation of patients in each severity score group according to the data (see Appendix \ref{app:detail_hosp},  Table \ref{tab:repartition-severity-score}).
We choose a discount factor of $0.95$ to capture the importance of future outcomes for the decision-maker.  While our theoretical results are agnostic to the choice of discount factor, we also verify that with alternative discount factors (e.g. $\lambda=0.99$), we obtain similar insights.
We choose the following rewards, satisfying Assumption \ref{assumption-0} and Assumption \ref{assumption-T-all-U},
\begin{equation}\label{eq:rec-def}
\begin{aligned}
r_{W}=100,  \; r_{RL}& =\dfrac{1}{1-\lambda} \cdot 250,\; r_{PT-RL} = \dfrac{1}{1-\lambda}\cdot 190,\; r_{CR-RL}=\dfrac{1}{1-\lambda}\cdot 160, \\
r_{D}& =\dfrac{1}{1-\lambda} \cdot 30,\; r_{CR-D} = \dfrac{1}{1-\lambda} \cdot 20,\; r_{PT-D}=\dfrac{1}{1-\lambda} \cdot 10,
\end{aligned}
\end{equation}
We would like to note that the following natural ordering conditions are satisfied:
\begin{align*}
r_{RL}  \geq r_{PT-RL} \geq r_{CR-RL},, r_{D}  \geq r_{CR-D} \geq r_{PT-D}.
\end{align*}
Certainly, different choices of rewards may lead to different thresholds for the optimal nominal and the optimal robust policies. {\color{black} It is a notoriously complex problem to estimate the exact values for such quantities as $r_{RL}$ and $r_{W}$, let alone $r_{PT-D}$ and $r_{PT-RL}$ (or more generally rewards in applications of reinforcement learning to healthcare, e.g. \cite{yauney2018reinforcement}, or Section II, ``Representation for Reward Function'' in \cite{yu2019reinforcement}).} Appendix \ref{app:sens} summarizes a detailed sensitivity analysis of our numerical results for different rewards. The single MDP is most valuable in identifying candidate worst-case transition matrices. {\color{black} While the thresholds of the optimal and nominal robust policies for the single MDP are highly dependent on the rewards,  our assumptions are not (we prove this in Appendix \ref{app:Lemmas-tr-sc}) and the performance of the hospital (in terms of mortality, LOS and average ICU occupancy) based on the resulting worse-case matrix is fairly consistent across different rewards, including those in \eqref{eq:rec-def}.}

\subsubsection{Empirical results for the single-patient MDP.}\label{sec:MDP-emp-results}
{\color{black}
We verify that for our choice of rewards, Assumptions \ref{assumption-0}-\ref{assumption-1} are satisfied. Verifying Assumption \ref{assumption-T-all-U} requires solving some linear programs, as the uncertainty sets $\U_{\min}$, $ \U_{\sf emp}$ and $\U_{\sf sa}$ are defined by linear inequalities.} From Theorems \ref{th:nom-thr} and \ref{th:thr-all-T}, we know the optimal nominal and robust policies are of threshold type. Therefore, we consider all threshold policies, denoted by $\pi^{[1]},...,\pi^{[11]}$, and compare their nominal and worst-case rewards in the single-patient MDP for the different uncertainty sets. Figure \ref{fig:perf_MDP} summarizes these results.
\begin{figure}[htb]
\begin{center}
\includegraphics[width=0.6\linewidth,height=7cm]{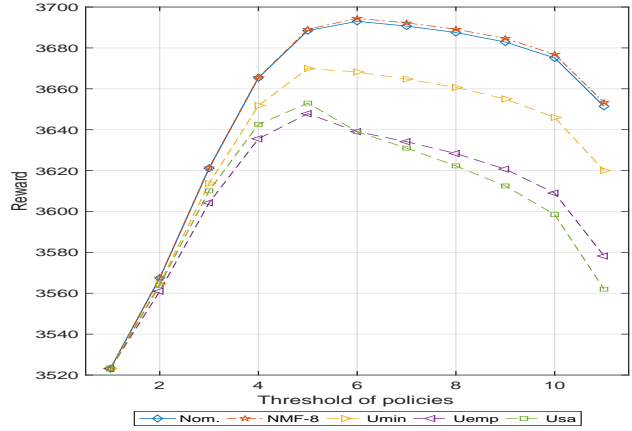}
\caption{Nominal and worst-case performance of threshold policies for an NMF approximation of rank $8$. For any threshold $\tau=1,...,11$, ``Nom.'' stands for $R(\pi^{[\tau]},\bm{T}^{0})$, while ``NMF-8'' stands for $R(\pi^{[\tau]},\bm{\hat{T}})$. The other three curves represent the worst-case reward of $\pi^{[\tau]}$ for the specified uncertainty set ($\U_{\min}, \U_{\sf emp}, \U_{\sf sa}$).}
\label{fig:perf_MDP}
\end{center}
\end{figure}
Note that for all threshold policies $\pi^{[\tau]}$, the corresponding reward using the estimated transition matrix, $R(\pi^{[\tau]},\bm{T}^{0})$,  and that using the NMF approximation of the transition matrix, $R(\pi^{[\tau]},\bm{\hat{T}})$, are practically indistinguishable. This provides additional support for using our NMF solution as an approximation for $\bm{T}^{0}$.

We observe that the optimal nominal policy ($\pi^{[6]}$) is different than the optimal robust policy ($\pi^{[5]}$) for the three different uncertainty sets. Our primary goal with the single MDP is not to provide direct recommendations for the hospital system, but rather to determine candidate transition matrices under which the hospital system can be evaluated. As we will see later, the performance of these policies under their corresponding transition matrices is quite different in the hospital simulation.

We would like to note that the worst-case matrices in $\U_{\min}$ and $\U_{\sf emp}$ do not belong to the $95 \%$ confidence intervals \eqref{eq:conf_unrelated}, even though $\bm{\hat{T}}$ belongs to \eqref{eq:conf_unrelated}. That said, only a few of the coefficients are outside of the $95 \%$ confidence regions, and the violations are small. For instance, for the worst-case matrix in $\U_{\min}$ associated with $\pi^{[5]}$, only five coefficients (out of $130$) are outside of $\eqref{eq:conf_unrelated}$, and the worst-case deviation is $-1.2179 \cdot \alpha_{3}$ (instead of $-\alpha_3$), while the mean of the absolute values of the deviations is $0.3381$ and $95 \%$ of these absolute values are below $1.0690$. The results are similar for $\U_{\sf emp}$. For instance, for the worst-case matrix in $\U_{\sf emp}$ associated with $\pi^{[5]}$, $20$ coefficients out of $130$ are outside the confidence intervals. While the coefficient $(1,1)$ is about $10 \cdot \alpha_{1}$ away from $T_{1,1}^{0}$ (instead of $2\alpha_1$), the mean of the absolute values of the deviations is $0.4430$, and $95 \%$ of these absolute values are below $2.3057$.  Therefore, we can still consider the worst-case matrices for $\U_{\min}$ and $\U_{\sf emp}$ as plausible transition matrices for our hospital model.
%
\subsection{Robustness analysis for the hospital.}\label{sec:exp-hosp}
The primary purpose of our single-patient MDP model is to develop insights into the management of the full hospital system.   To that end, we use our single-patient MDP  to generate transition matrices that are candidates for a worst-case deterioration of the hospital performance. Given the complexity and multi-objective nature of the hospital system (i.e., minimize mortality rate, LOS, and average ICU occupancy), defining, let alone deriving, an optimal policy is highly complex. As such, we focus on the class of threshold policies given their desirable theoretical properties  (see Theorems \ref{th:nom-thr} and \ref{th:thr-all-T}) and their simplicity which can help facilitate implementation in practice. For each threshold policy $\pi^{[\tau]}$ and each uncertainty set $\U$ (among $\U_{\min},\U_{\sf emp}, \U_{\sf sa}$), we  compute $\boldsymbol{T}^{[\tau,\U]}$  a worst-case transition matrix for the single-patient MDP in $\U$:
$$\bm{T}^{[\tau,\U]} \in \arg \min_{\bm{T} \in \U} R(\pi^{[\tau]},\bm{T}).$$
Then, we use the pair $(\pi^{[\tau]},\boldsymbol{T}^{[\tau,\U]})$ to simulate the hospital performance as measured by the mortality rate, length-of-stay, and average occupancy of the ICU.

\subsubsection{In-hospital mortality and Length-Of-Stay.}

\paragraph{In-hospital mortality.}

In Figure \ref{fig:mort_rand_bump_1} we study the variation of the hospital performance over the 95\% confidence intervals for the  nominal transition matrix $\bm{T}^{0}$. In particular, we sample $N=20$ transition matrices in the confidence intervals \eqref{eq:conf_unrelated} and plot the nominal performance (mortality rate versus average ICU occupancy) of all threshold policies as well as the performance for $4$ of the $N$ sampled matrices. The mortality rate for all sampled transition matrices are very close to the nominal mortality rate. The maximum \textit{relative} observed deviation from the nominal mortality rate is $8.82 \%$, with an average relative deviation of $3.84 \%$. We present more details about the statistics of the random deviations from the nominal performance in Appendix \ref{app:rand-sample}.

In Figure \ref{fig:mort_worst_bump_1}, we  compare the worst-case performance of all threshold policies with the nominal performance. For each threshold policy, we construct a worst-case transition matrix that minimizes the single-patient MDP reward and compute the hospital performance for this particular matrix and threshold policy. As we saw in the single-patient MDP experiments,  the hospital performance for $\bm{T}^{0}$ and $\bm{\hat{T}}$ are very close, again suggesting that the NMF approximation is reasonable.  As before, we consider  the three uncertainty sets: $\U_{\sf min}$,  $\U_{\sf emp}$ and $\U_{\sf sa}$. Note that $\U_{\sf min}$ and $\U_{\sf emp}$ are centered around our NMF approximation $\bm{\hat{T}}$. Under uncertainty set $\U_{\sf min}$, the mortality rate can significantly increase, with relative increases from $18 \%$ to $23 \%$. This substantial degradation occurs even though $\U_{\sf min}$ is our most-optimistic uncertainty set, with variations in the order of $10^{-4}$ from $\bm{\hat{T}}$. For worst-case matrices in $\U_{\sf emp}$ or $\U_{\sf sa}$, the mortality rate of any threshold policy increases by $40 \%$ to $50 \%$. Therefore, our worst-case  analysis (Figure \ref{fig:mort_worst_bump_1}) shows that the mortality may severely deteriorate, even for very small parameters deviations from the nominal transition matrix $\bm{T}^{0}$. Note that this is not the case in our random sample analysis (Figure \ref{fig:mort_rand_bump_1}). This suggests that not considering worst-case deviations may lead to overly optimistic estimations of the hospital performance.

As a thought experiment, suppose the decision-maker determined that the average ICU occupancy should not exceed $72 \%$. The decision-maker then chooses the threshold policy with the lowest mortality and average ICU occupancy lower than $72 \%$. Based on the nominal performance, the decision-maker will choose $\pi^{[5]}$ which proactively transfers $27.1 \%$ of the patients. However, our analysis demonstrates there exists a ``worst-case'' transition matrix that is consistent with the available data which, under the selected policy $\pi^{[5]}$, would result in a higher average ICU occupancy of $74.0 \%$. In contrast, if the decision-maker were to account for the parameter uncertainty and consider the worst-case performance in $\U_{\min}$, the decision-maker would choose $\pi^{[6]}$, which proactively transfers $10.2 \%$ of the patients and results in a worst-case average ICU occupancy of $71.9 \%$.
\begin{figure}[htb]
\begin{subfigure}{0.5\textwidth}
 \includegraphics[width=1.1\linewidth,height=11cm]{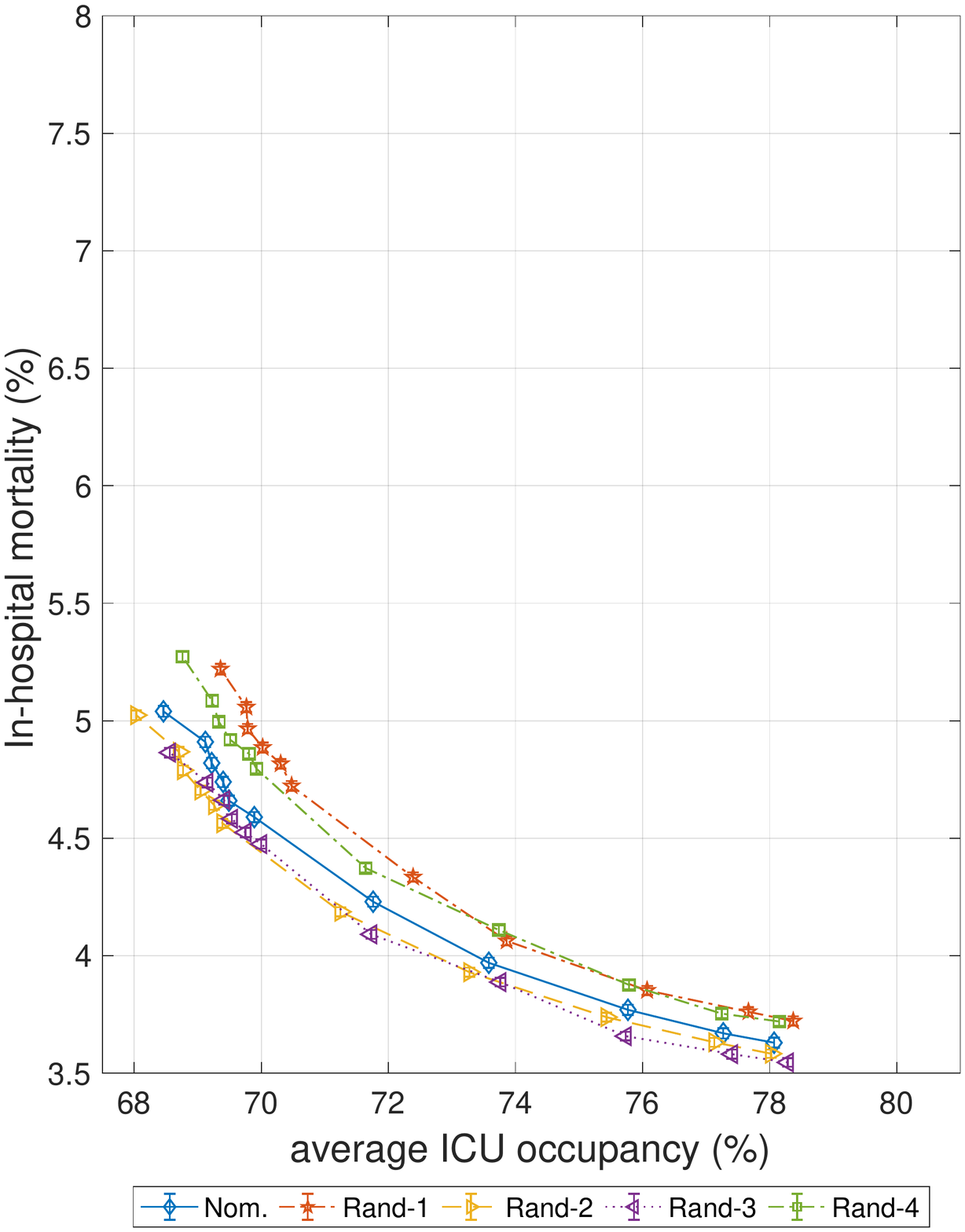}
\caption{Random samples analysis.}
\label{fig:mort_rand_bump_1}
\end{subfigure}
\begin{subfigure}{0.5\textwidth}
  \includegraphics[width=1.1\linewidth,height=11cm]{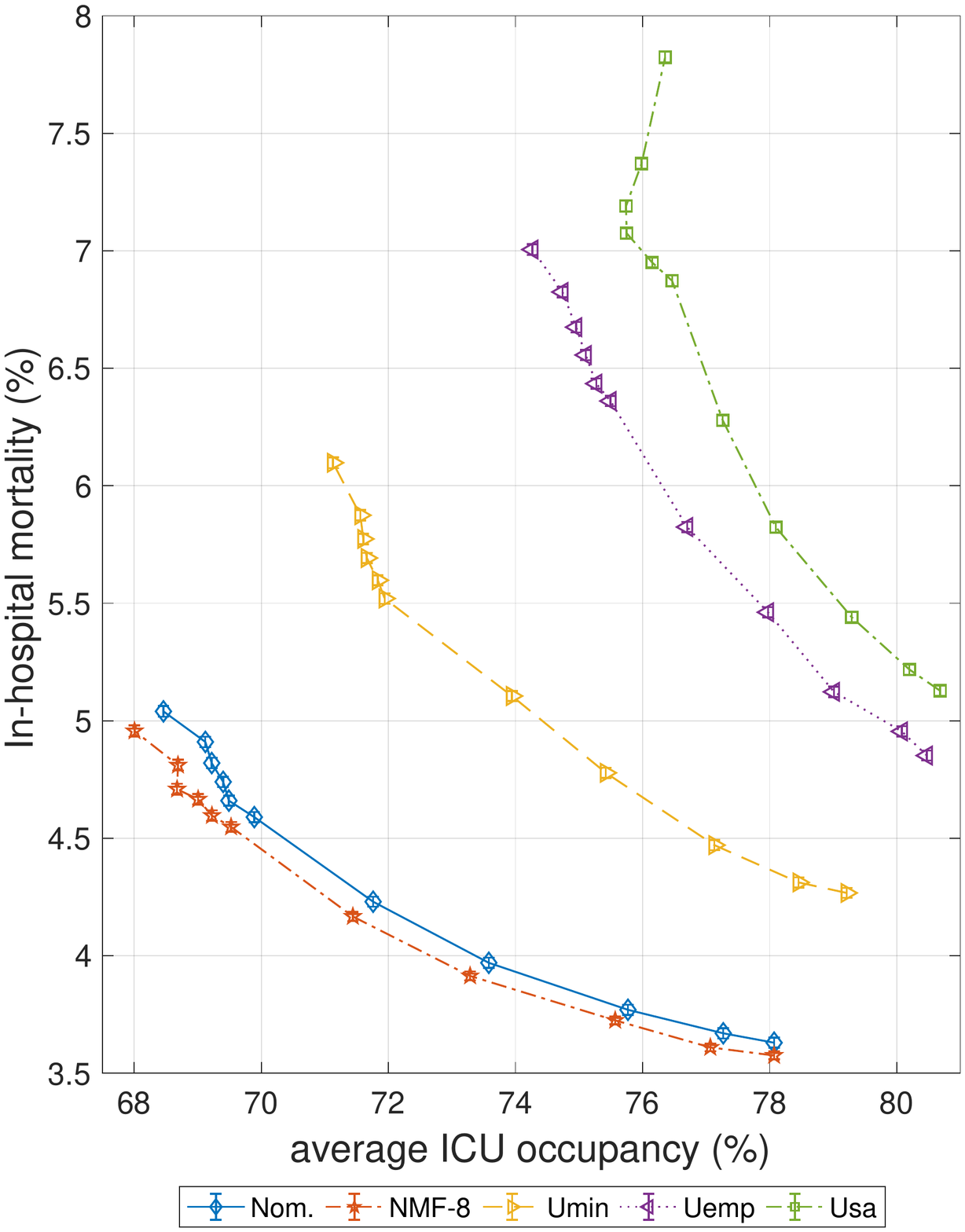}
\caption{Worst-case analysis.}
\label{fig:mort_worst_bump_1}
\end{subfigure}
\caption{In-hospital mortality of the $11$ threshold policies for the nominal estimated matrix,   randomly sampled matrices in the 95\% confidence intervals (left-hand side), and  the worst-case matrices found by our single MDP model (right-hand side). Each point corresponds to a threshold policy: the policy with highest mortality rate corresponds to threshold $\tau=11$ (top-left of each curve) and the threshold decreases until the bottom-right point of each curve, corresponding to threshold $\tau=1$. We consider the uncertainty sets $\U_{\min},\U_{\sf emp}$ and $\U_{\sf sa}$ when the rank $r=8$. On the right-hand side,  we also report the hospital mortality  rate when the transition matrix is our NMF approximation with rank $8$.}
\label{fig:mort}
\end{figure}

In general, as the threshold decreases, and proactive transfers are used more aggressively, the ICU occupancy increases while the mortality rate decreases. This behavior does not generalize to the uncertainty set $\U_{\sf sa}$. In particular, we notice in Figure \ref{fig:mort_worst_bump_1} that for $\U_{\sf sa}$, the worst-case mortality rate \textit{and} the ICU occupancy decrease from $\pi^{[11]}$ to $\pi^{[8]}$. Therefore, the (worst-case) average ICU occupancy decreases when the decision-maker decides to transfer more patients to the ICU. In principle, this could be explained by the fact that the patients with severity scores in $\{8,9,10\}$ are the sickest patients. Therefore, proactively transferring them may actually be Pareto improving. That said, we are somewhat cautious about the interpretations of $\U_{\sf sa}$. The worst-case transition matrices in $\U_{\sf sa}$ are  extreme perturbations from $\bm{T}^{0}$. For instance, the coefficients $T_{1,RL}^{0},...,T_{n,RL}^{0}$ become $T_{1,RL}^{0}- \alpha_{1},...,T_{n,RL}^{0}-\alpha_{n}, $ and the coefficients $T_{1,D}^{0},...,T_{n,D}^{0}$ become $T_{1,D}^{0} + 2 \cdot \alpha_{1},...,T_{n,D}^{0} + 2 \cdot \alpha_{n}.$ In that sense, such  coordinated, structured parameter misspecification appears unlikely. This is due to the ability to arbitrarily perturb the  coefficients of $\bm{T}^{0}$, provided that the resulting rows still form a transition kernel, rather than accounting for potential correlations across states that our factor matrix approach incorporates. Such extreme perturbations are unlikely to arise in practice, which is why we focus our attention on  the model of factor matrix uncertainty set.


{\color{black} Finally, we note that the worst-case matrices for our factor matrix uncertainty sets ($\U_{\min}$ and $\U_{\sf emp}$) are as far from the nominal estimation $\bm{T}^{0}$ as the worst-case matrices for $\U_{\sf sa}$, in terms of the $1$-norm. This is because $\U_{\min}$ and $\U_{\sf emp}$ are centered around our nonnegative matrix factorization and not around $\bm{T}^{0}$. Therefore, the differences in numerical results and insights, both for our single-patient MDP and in our hospital model, are caused by the rank-constrained nature of the factor matrix uncertainty sets, and not to a difference in the radii of the uncertainty sets. }

Underlying physiologic characteristics dictate the evolution of a patient's health. As such, it is reasonable to assume these similar medical factors manifest themselves in our model through correlated dynamics  across the different severity scores. Thus, we expect the uncertainty to be reasonably captured by a low-rank deviation from the nominal estimation $\bm{T}^{0}$. Therefore, we expect the true worst-case performance of the threshold policies to be somewhere in between the performance in $\U_{\min}$ and the performance in $\U_{\sf emp}$.

From these experiments we see that 1) ignoring parameter uncertainty may result in overly optimistic expectations of system performance, and 2) the type of parameter uncertainty (e.g. correlated or arbitrary) can have a substantial impact on the insights derived from the robust analysis.

\paragraph{Length-Of-Stay.}
\begin{figure}[htb]
\begin{subfigure}{0.5\textwidth}
 \includegraphics[width=1.1\linewidth,height=11cm]{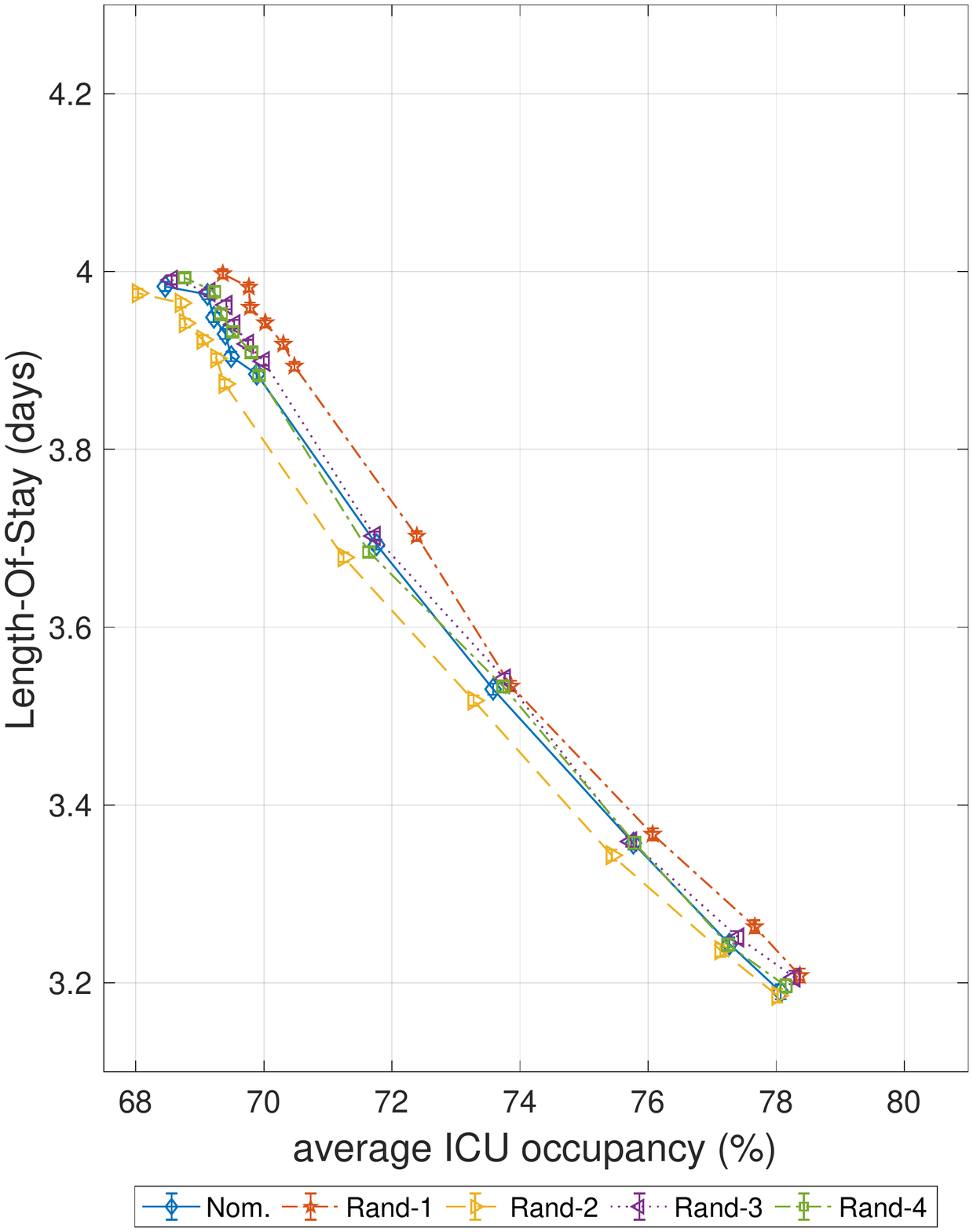}
\caption{Random samples analysis.}
\label{fig:LOS_rand_bump_1}
\end{subfigure}
\begin{subfigure}{0.5\textwidth}
  \includegraphics[width=1.1\linewidth,height=11cm]{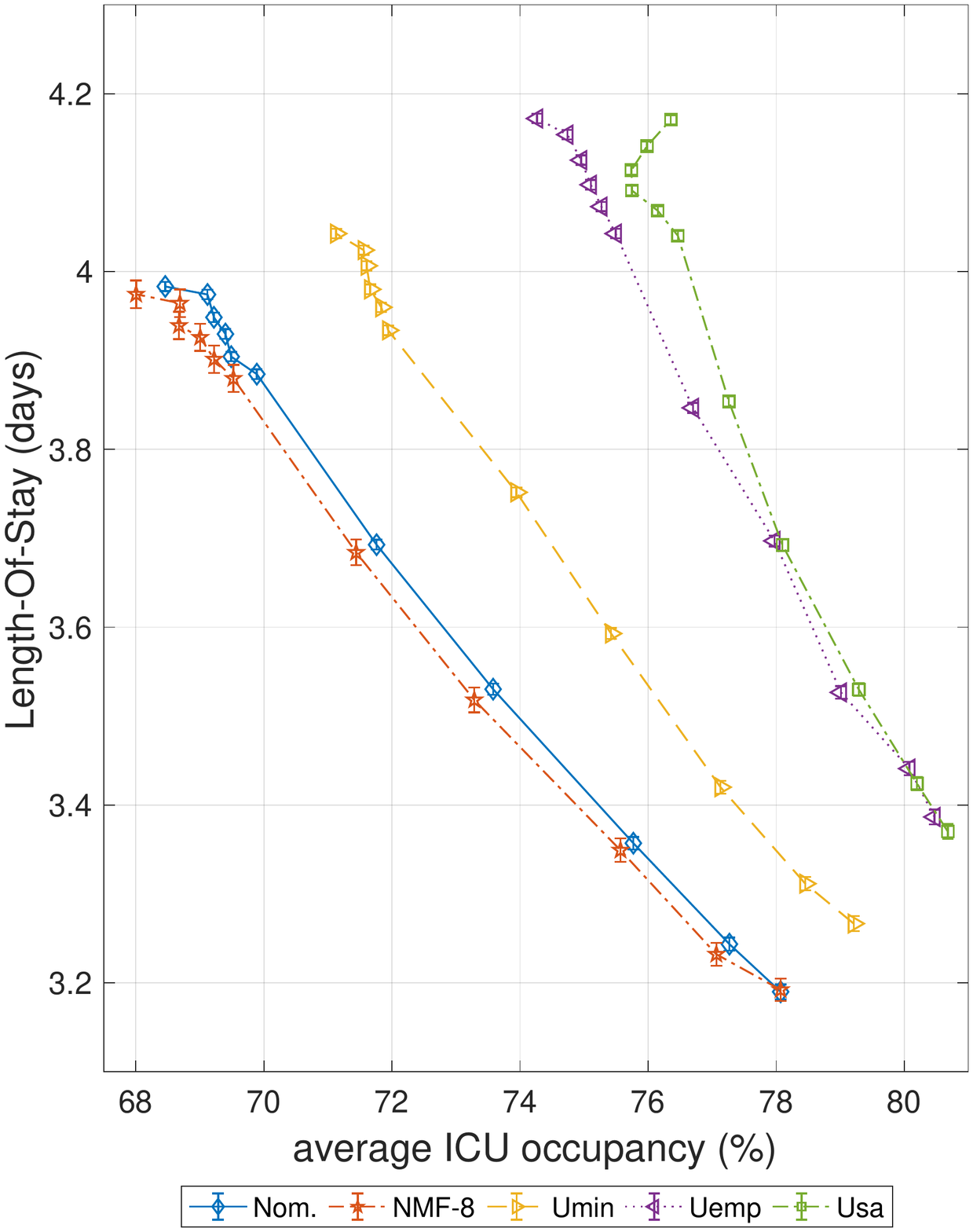}
\caption{Worst-case analysis.}
\label{fig:LOS_worst_bump_1}
\end{subfigure}
\caption{Length-Of-Stay of the $11$ threshold policies for the nominal estimated matrix,   randomly sampled matrices in the 95\% confidence intervals (left-hand side), and  the worst-case matrices found by our single MDP model (right-hand side). Each point corresponds to a threshold policy: the policy with highest LOS corresponds to threshold $\tau=11$ (top-left of each curve) and the threshold decreases until the bottom-right point of each curve, corresponding to threshold $\tau=1$. We consider the uncertainty sets $\U_{\min},\U_{\sf emp}$ and $\U_{\sf sa}$ when the rank $r=8$. On the right-hand side, we also report the hospital mortality  rate when the transition matrix is our NMF approximation with rank $8$.}
\label{fig:LOS}
\end{figure}

In the case of Length-Of-Stay (LOS), we notice similar trends as compared to the in-hospital mortality rate. Figure \ref{fig:LOS_rand_bump_1} shows the deviations in performance for some randomly sampled matrices. The average deviation ranges from $0.34 \%$ for $\pi^{[3]}$ to $1.04 \%$ of deviation for threshold $\pi^{[11]}$.  Therefore, the hospital flow seems stable as regards to parameters deviations from the nominal matrix $\bm{T}^{0}$

We compare the worst-case LOS with the nominal performance. We see that the LOS can increase by up to $2.5 \%$ in $\U_{\min}$, and up to $5.0 \%$ in $\U_{\sf emp}$ and $\U_{\sf sa}.$ The impact of worst-case parameter deviations  is less severe for the Length-Of-Stay than for the mortality rate. However, worst-case deviations are still more substantial than random deviations from the nominal transition (Figure \ref{fig:LOS_rand_bump_1}).
As for in-hospital mortality rate, in Figure \ref{fig:LOS_worst_bump_1} we notice that under uncertainty set  $\U_{\sf sa}$, it appears to be Pareto improving to be more aggressive in proactively transferring patients with threshold policy $\pi^{[8]}$ rather than $\pi^{[11]}$. However, as discussed before, we believe that  $\U_{\sf sa}$ is not able to fully capture reasonable types of uncertainty one would expect to see in practice.
{\color{black}
\subsubsection{Impact of proactive transfers on demand-driven discharges.}
As the proactive transfer policies admit more patients to the ICU than reactive policies, they may increase ICU congestion and, consequently, the number of \textit{demand-driven discharged} patients from the ICU. Such discharges are associated with worse outcomes \citep{chrush2009occupancy}. In Figures \ref{fig:DDD_nom}-\ref{fig:DDD_min}, we explore the impact of proactive transfers onto the number of patients that are demand-driven discharged.

We can see that it is possible to proactively transfer up to the top 5 severity conditions without significantly impacting the proportions of demand-driven discharged patients. We also note that in this metric, the trends we see with mortality/length-of-stay are preserved. In other words, proactively transferring a small fraction of the riskiest patients (i.e. the highest AAM scores) can improve  the average mortality rate/length-of-stay without significantly increasing the ICU occupancy or the number of demand-driven discharges.  The findings are similar with the worst-case transition matrices (Figures \ref{fig:DDD_sa}-\ref{fig:DDD_min}), with the  uncertainty set $\U_{\sf sa}$ leading to worse deterioration than our factor matrix uncertainty sets $\U_{\sf emp}$ and $\U_{\sf sa}$.
\begin{figure}[h]
\begin{subfigure}{0.24\textwidth}
 \includegraphics[width=0.9\linewidth,height=3cm]{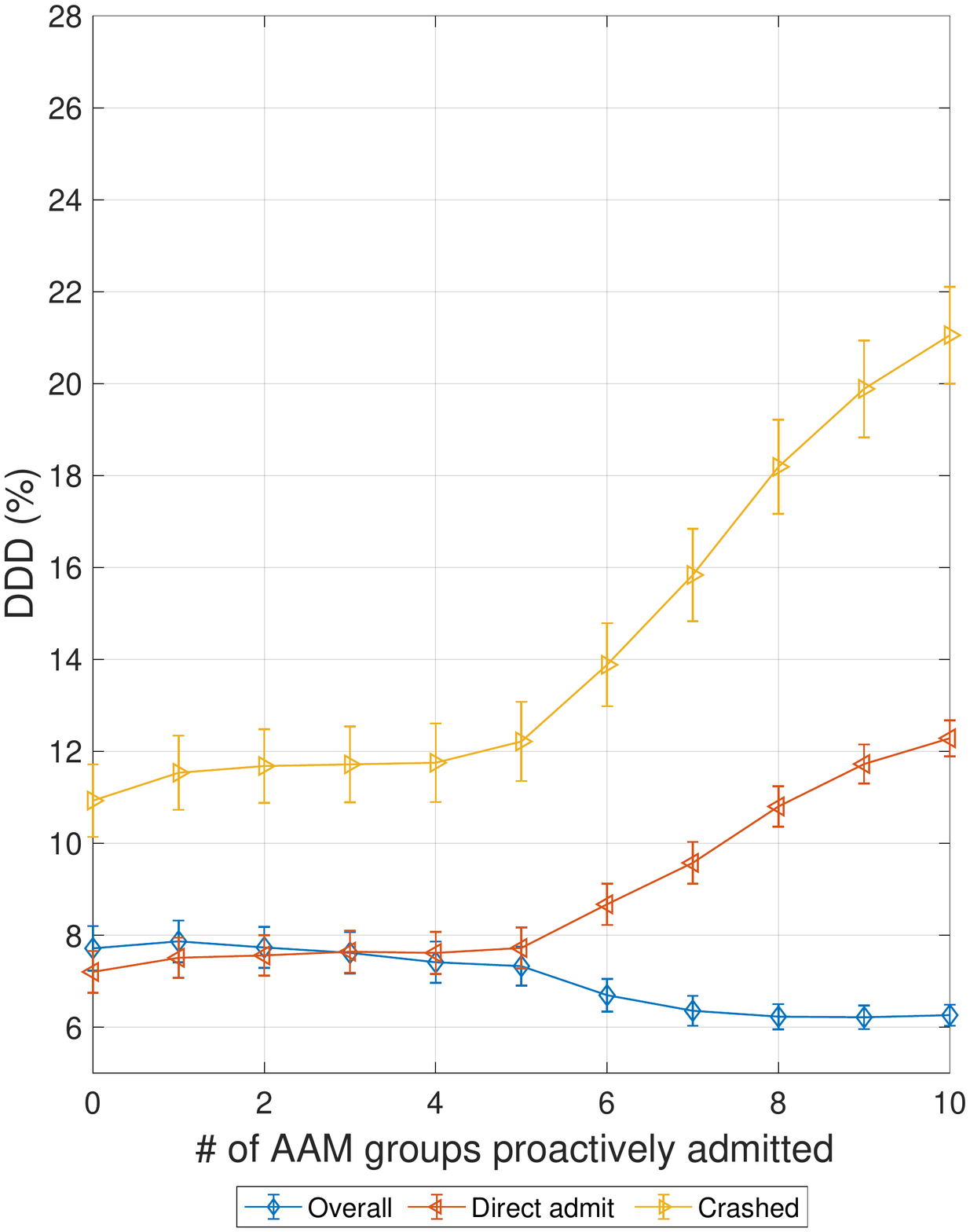}
\caption{Nominal.}
\label{fig:DDD_nom}
\end{subfigure}
\begin{subfigure}{0.24\textwidth}
 \includegraphics[width=0.9\linewidth,height=3cm]{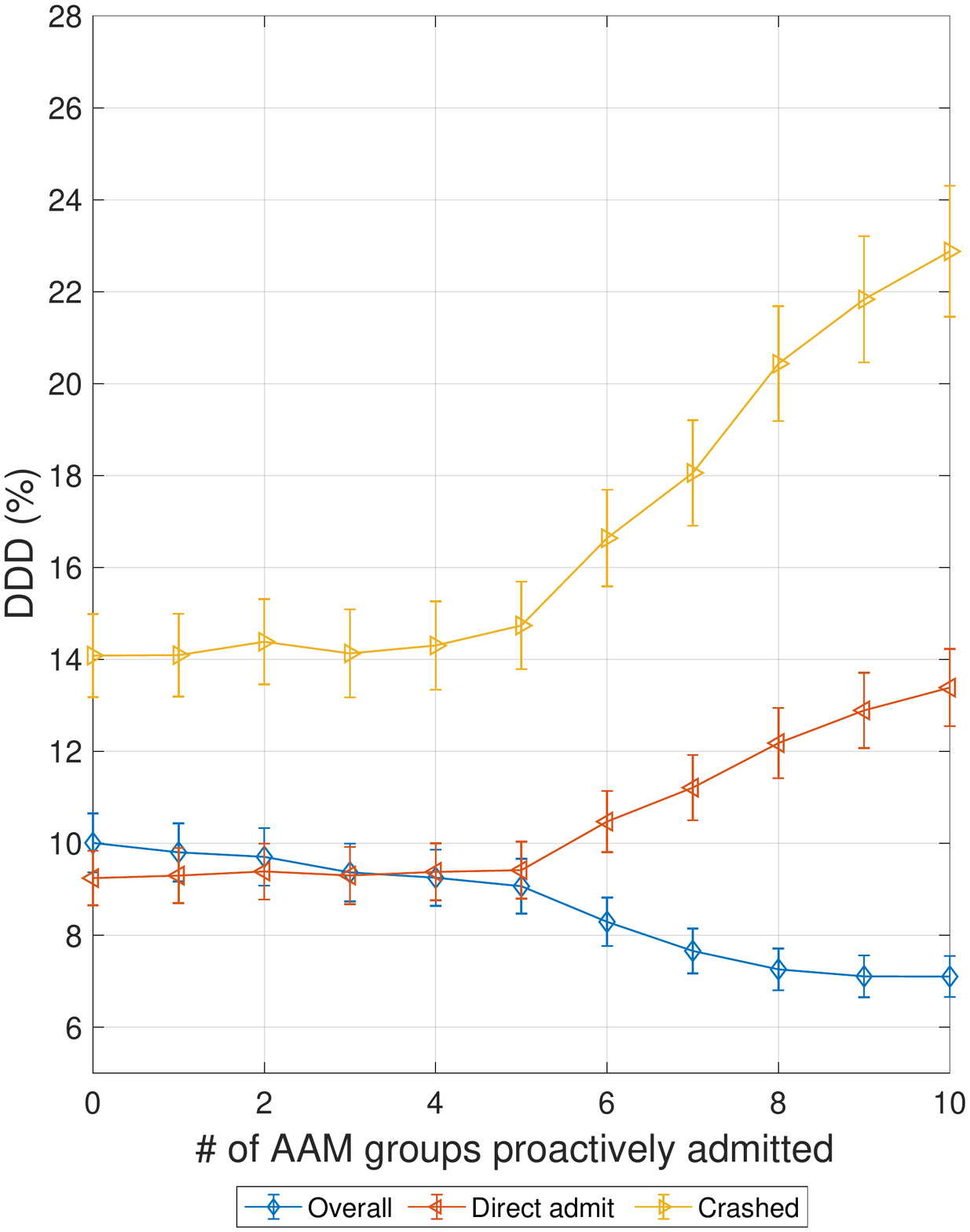}
\caption{$\U=\U_{\min}$.}
\label{fig:DDD_min}
\end{subfigure}
\begin{subfigure}{0.24\textwidth}
 \includegraphics[width=0.9\linewidth,height=3cm]{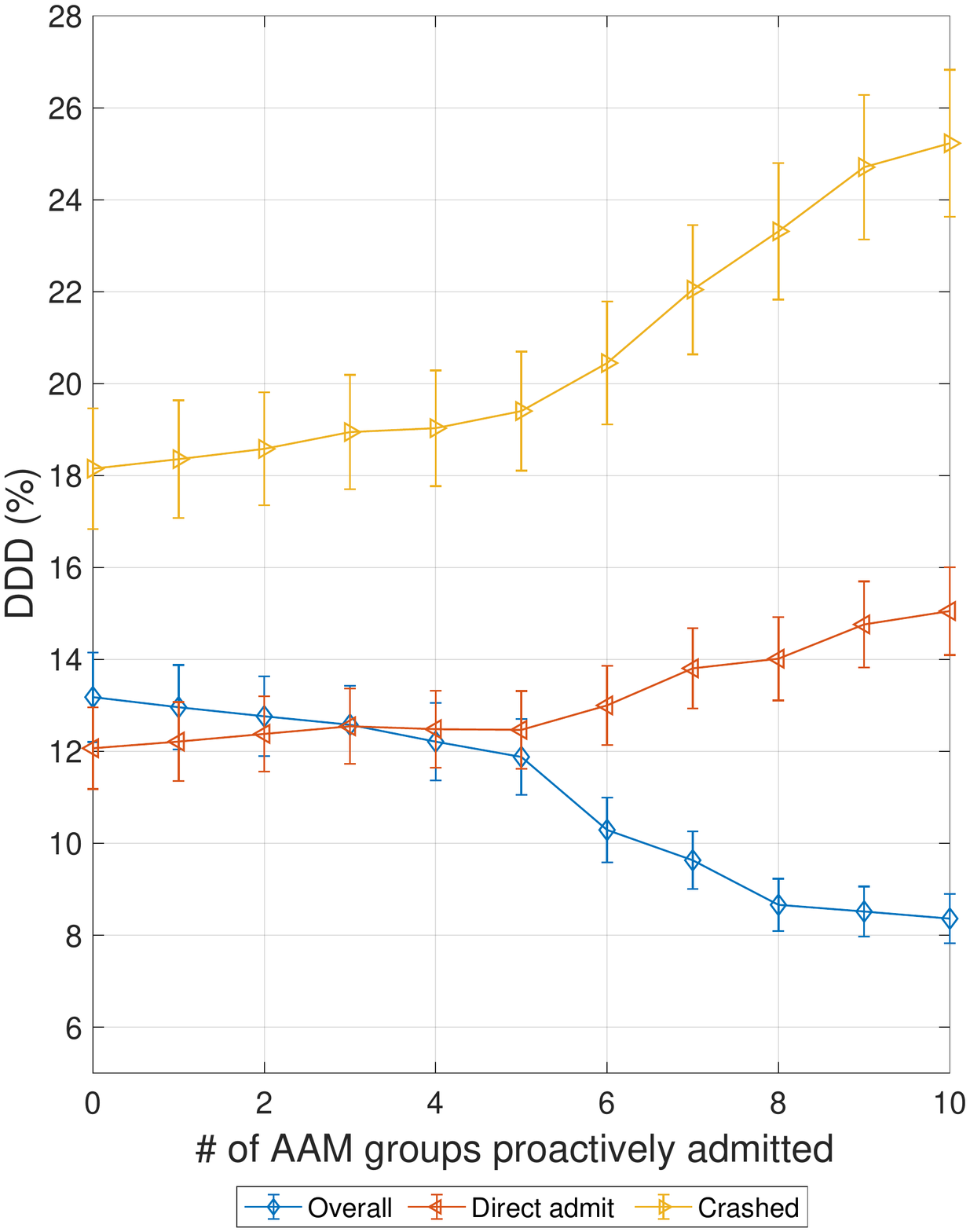}
\caption{$\U=\U_{\sf emp}$.}
\label{fig:DDD_emp}
\end{subfigure}
\begin{subfigure}{0.24\textwidth}
 \includegraphics[width=0.9\linewidth,height=3cm]{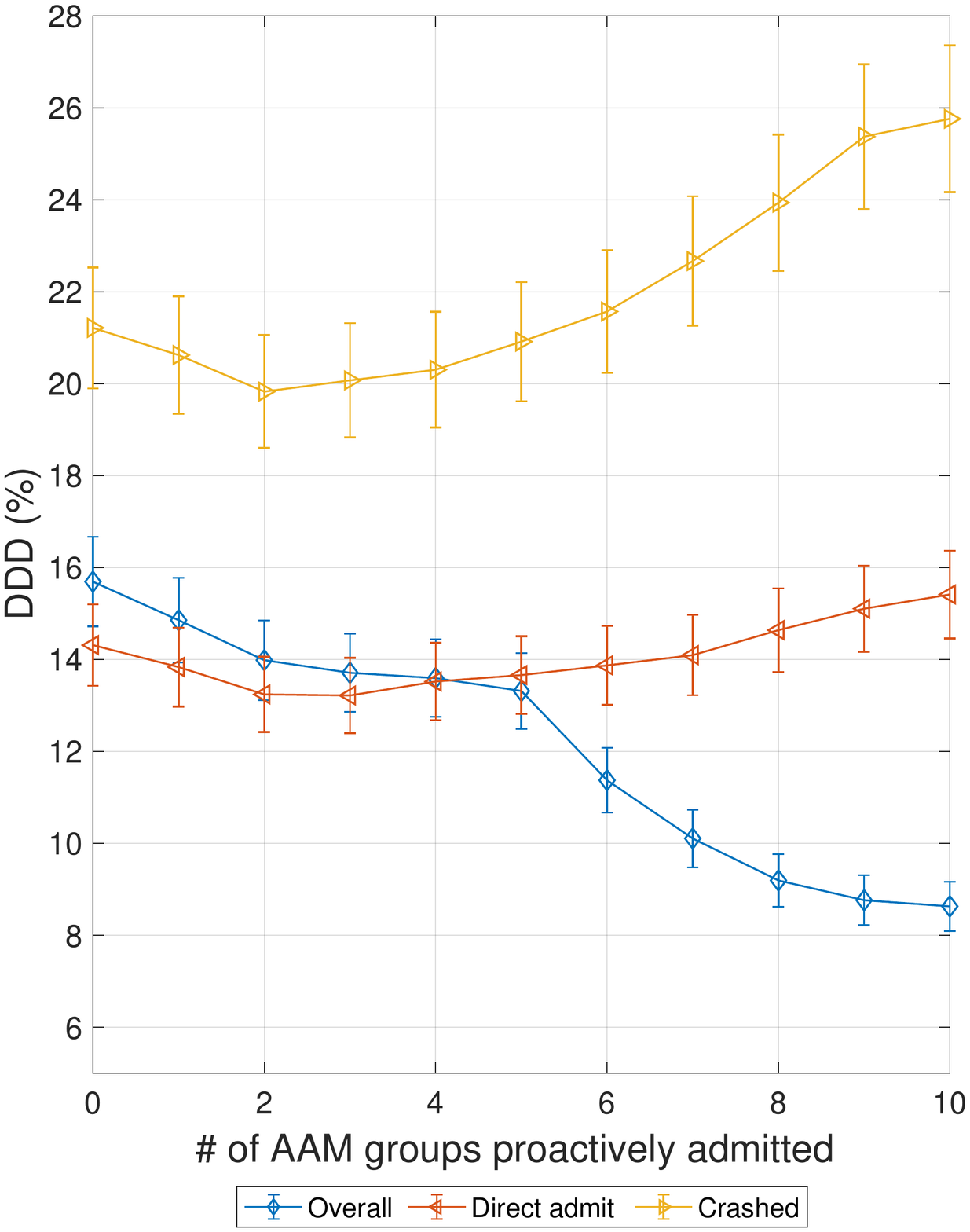}
\caption{$\U=\U_{\sf sa}$.}
\label{fig:DDD_sa}
\end{subfigure}
\caption{Percentages of Demand-Driven Discharges (DDD) among ICU transfers in terms of the number of Advanced Alert Monitor (AAM) scores proactively transferred. We present the results for our nominal transition matrix and for the worst-case transition matrices in $\U_{\sf sa}, \U_{\sf emp}$ and $\U_{\min}$.}
\label{fig:DDD}
\end{figure}

Because we do not allow proactive transfers when the ICU is full,  demand-driven discharges only occur when there is an external arrival, when a patient crashes on the ward,  or when a patient requires readmission to the ICU. Still, it would be concerning if the patients who are demand-driven discharged are among the most severe.
We find that demand-driven discharged patients were discharged from the ICU after $75$ to $85 \%$ of their ICU LOS, a duration which can often mean the patient has recovered enough to be safely transferred to a lower level of care \citep{lowery1992simulation}.

{\color{black}
\subsubsection{Impact of potential waiting in the ward}\label{sec:simu-queue} 

Our primary hospital model considered in Section \ref{sec:hosp} and the simulations thus far assumes that whenever a crashed, readmitted, or external patient arrives to a full ICU, it precipitates a demand-driven discharge patients from the ICU.
In practice, it is possible that a patient requiring ICU admission when the ICU is full would need to \textit{wait} in the ward until an ICU bed to become available. 

 We now numerically explore the impact of the possibility of waiting in the ward. In particular, we consider an alternative hospital model  where demand-driven discharges  never occur.
The dynamics of the hospital is the same as in Section \ref{sec:hosp}, except that if an ICU admission is required but the ICU is full, the new patient enters a \textit{waiting queue}. Patients are served from the queue in a First-Come-First-Serve service discipline. Note that patients can enter the queue either as (i) direct admits, (ii) crashed patients, or (iii) readmits to the ICU. Our proactive transfer policies cannot send a patient to the queue.

Every six hours, patients in the waiting queue either:
\begin{itemize}
\item Die in the queue: this happens with the same probability as that for the sickest patients (i.e. a severity score of 10), which amounts to a 6.84 \% chance of dying every six hours. This is because patients requiring ICU admission but wait in the queue are very severe.
\item Continue to wait in the queue. This would happen if there are no available ICU beds.
\item Are transferred to the ICU. This will only occur when there is an available ICU bed (e.g. after a natural discharge). Priority is given to patients who have stayed longest in the queue. To focus on the impact of the queue, we assume patients from the queue who are admitted to the ICU have similar LOS/mortality risk as if they had not been in the queue.
\end{itemize}}
 We present our results below for this alternative model of hospital dynamics.
\paragraph{Impact of proactive transfers on hospital metrics.}
We present our results for the mortality rate and the LOS in the hospital in Figure \ref{fig:mort_and_LOS_in_hospital_with_queue}, for this alternative model with a queue. The results are very similar to the primary hospital model without queue, even though the nominal metrics (mortality rate, ICU occupancy and LOS) are slightly worse. Despite these small quantitative differences, the qualitative insights are consistent. Threshold policies have the potential to improve in-hospital mortality rates and length-of-stay, at the price of an increase in ICU occupancy. Proactively transferring a small proportion of the patients (in our simulations, the top 10 \% sickest patients) does not significantly increase the ICU occupancy. The results in worst-case metrics confirm these trends, with the $s,a$-rectangular uncertainty showing similar anomalous results  as with the primary hospital model without queue.
 \begin{figure}[h]
\begin{subfigure}{0.5\textwidth}
 \includegraphics[width=1.1\linewidth,height=11cm]{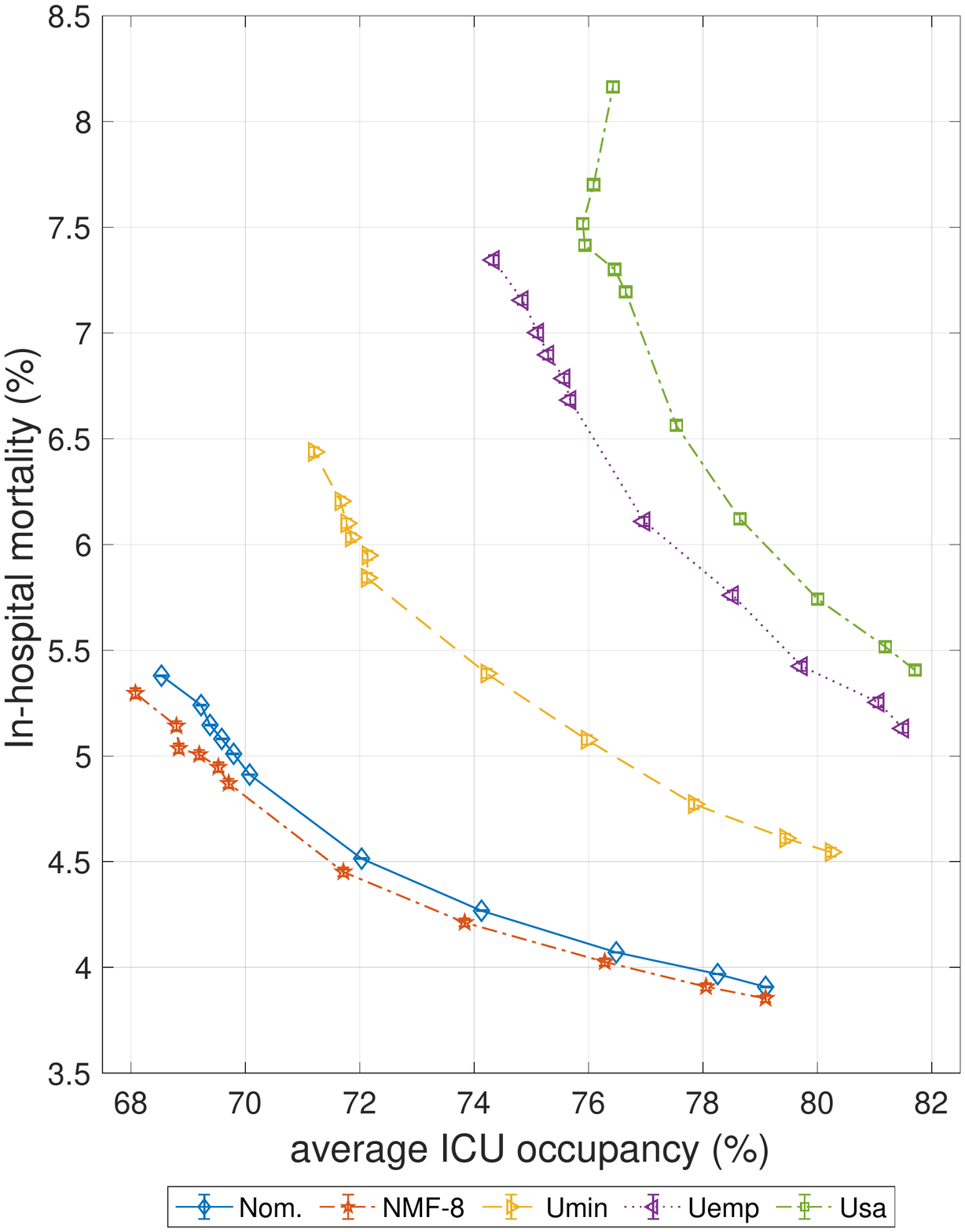}
\caption{In-hospital Mortality}
\label{fig:mort_nom_with_queue}
\end{subfigure}
\begin{subfigure}{0.5\textwidth}
 \includegraphics[width=1.1\linewidth,height=11cm]{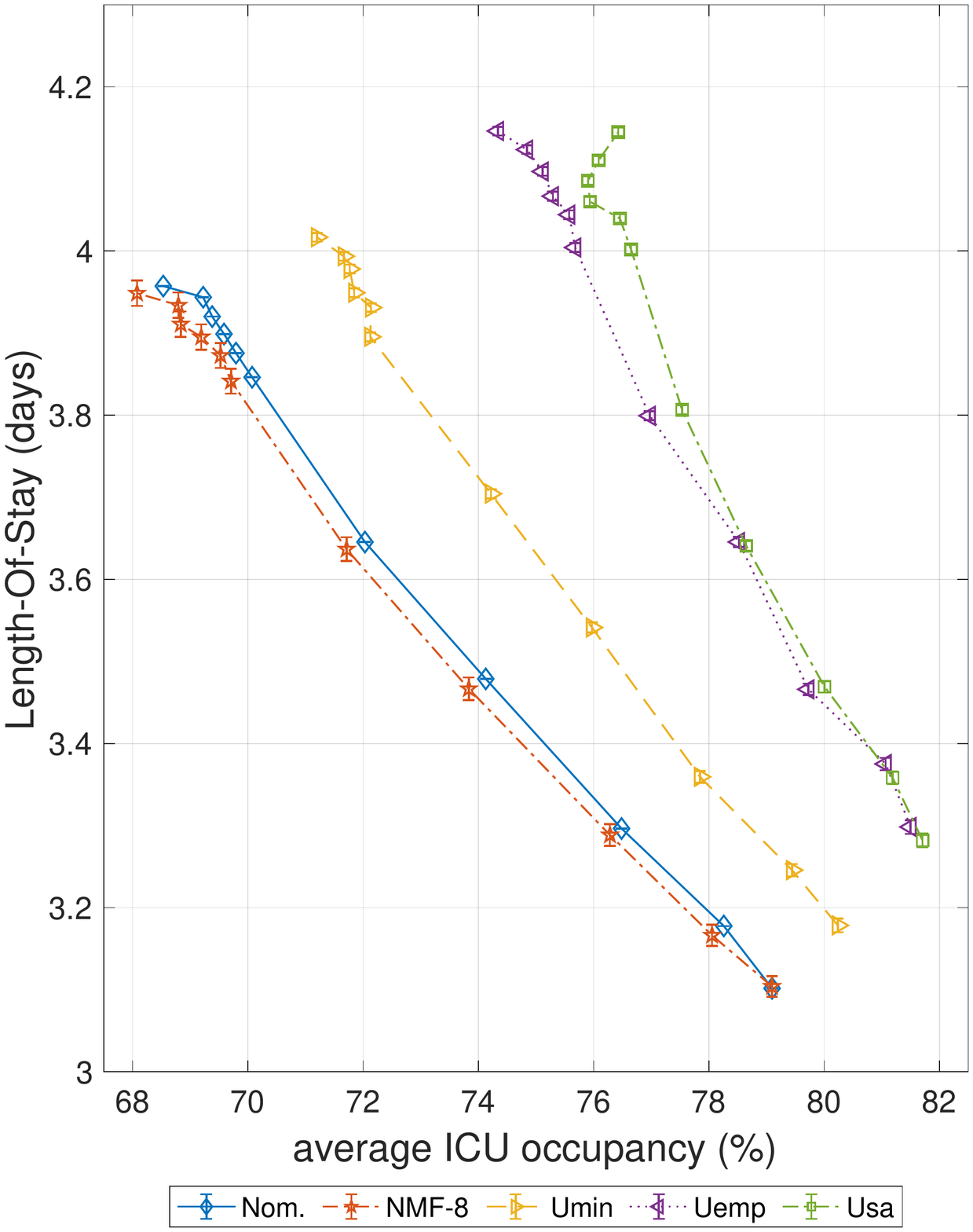}
\caption{Length-Of-Stay}
\label{fig:LOS_nom_with_queue}
\end{subfigure}
\caption{Average mortality and LOS in the hospital for the 11 threshold policies, in our hospital model with a queue.}
\label{fig:mort_and_LOS_in_hospital_with_queue}
\end{figure}
{\color{black}
\paragraph{Impact of proactive transfers on  waiting queue metrics.} We also consider the probability of entering the queue across different patient types (direct admits, crashed or readmitted),  the average LOS in the ICU,  the average length of the queue, and  the average LOS in the queue.

In Figure \ref{fig:proba_types_in_queue}, we notice that the (nominal) probability  a patient (direct admit, crashed or readmitted) enters the queue is increasing as more patients are proactively transferred to the ICU. This is because the ICU occupancy increases with the number of proactive transfers. Crucially, we see that the increase in probabilities does not increase substantially until proactive transfers become quite aggressive (i.e. threshold higher than or equal to 5). Therefore, proactively transferring up to the first 5 severity conditions does not result in significant changes in the number of patients in the queue. For the sake of brevity, we present our worst-case numerical experiments for this metric in Appendix \ref{app:proportion-and-patient-types}.

Additionally, in Figure \ref{fig:avg_LOS_in_ICU} we note that proactively transferring patients also improves the average length-of-stay in the ICU. In particular, even though the average occupancy of the ICU is increasing when more patients are proactively transferred (see Figure \ref{fig:mort_and_LOS_in_hospital_with_queue}), the patients who are admitted in the ICU tend to have shorter ICU LOS. This is because (i) proactively admitted patients have shorter ICU LOS  than crashed patients  from the same severity score, and (ii) as we proactively transfer more patients,  patients enter the ICU from lower severity conditions and stay less time in the ICU than patients with more critical severity conditions. Therefore, even though the ICU occupancy increases when more patients are proactively transferred, it is also the case that more patients can be admitted to the ICU (as patients in the ICU stay shorter periods of time).

As a consequence of shorter lengths of ICU visits, we observe that proactively transferring more patients has the beneficial effect of decreasing the average length of the queue and the average waiting time in the queue. (see Figures \ref{fig:avg_length}  and \ref{fig:avg_LOS_in_queue}). This may seem counter-intuitive given Figure \ref{fig:proba_types_in_queue} (where we see that more patients enter the queue as we proactively transfer more patients). However, following Figure \ref{fig:avg_LOS_in_ICU}, we see that the ICU LOS is decreasing, and therefore, patients can be admitted to the ICU more often (than when fewer patients are proactively transferred). This is an important beneficial aspect of proactive transfer policies, as delays in patients admissions to the ICU are associated with worse mortality \citep{chalfin2007impact}.
The simulations with worst-case transition matrices show similar decrease in average queue length and waiting time as we proactively transfer more patients.
 \begin{figure}[h]
\begin{subfigure}{0.24\textwidth}
 \includegraphics[width=1.1\linewidth,height=5cm]{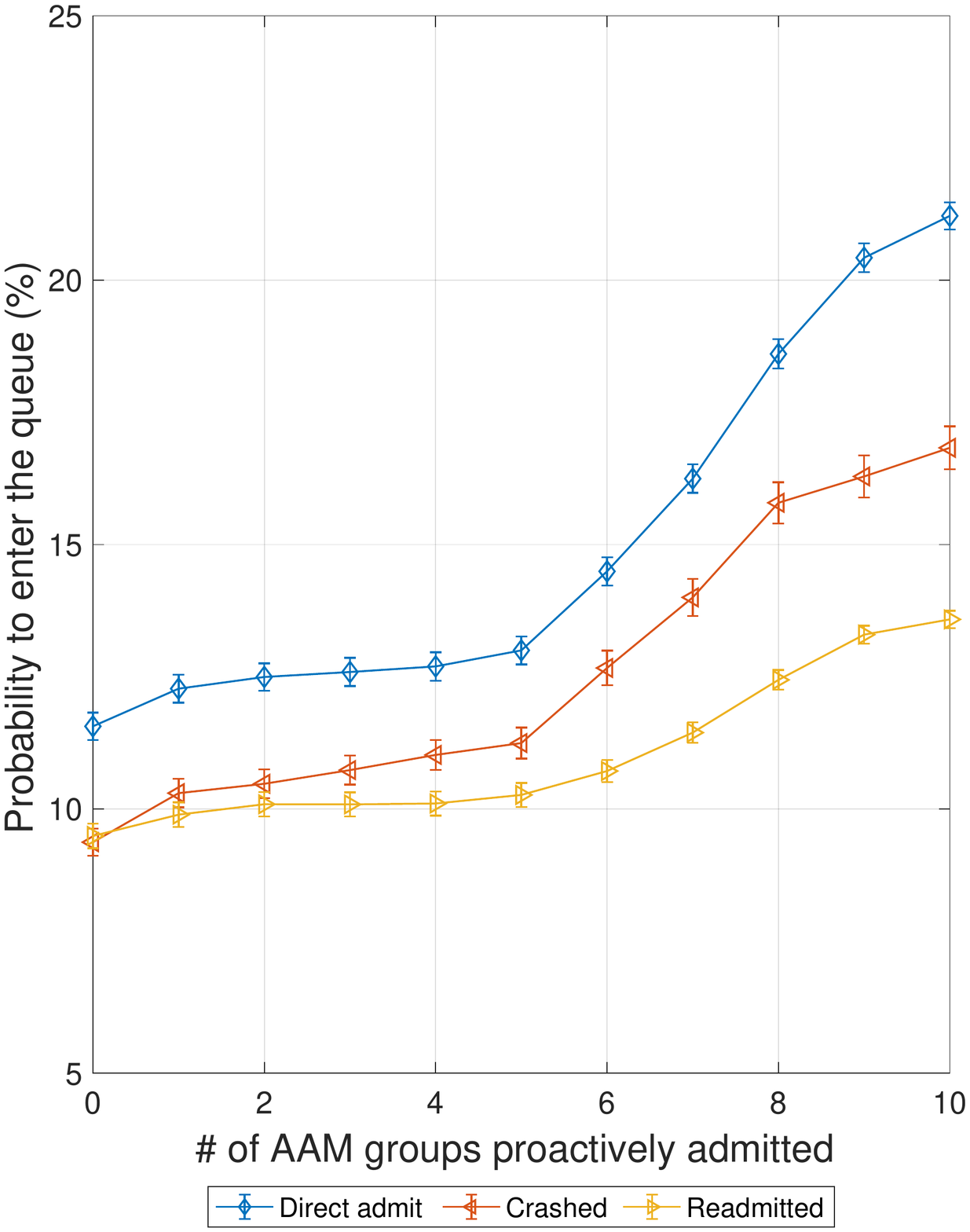}
\caption{Probability to enter the queue.}
\label{fig:proba_types_in_queue}
\end{subfigure}
\begin{subfigure}{0.24\textwidth}
 \includegraphics[width=1.1\linewidth,height=5cm]{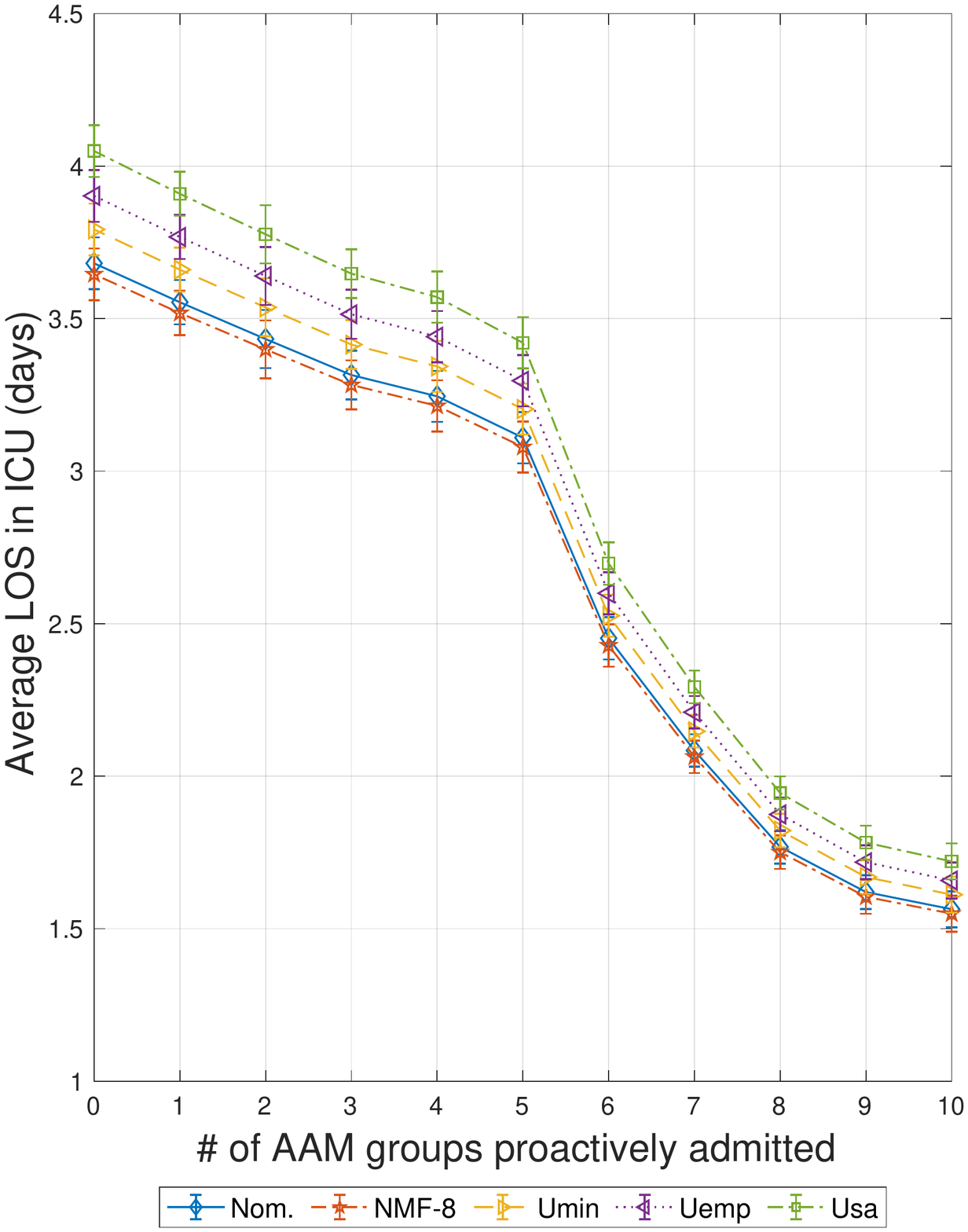}
\caption{Average LOS in ICU.}
\label{fig:avg_LOS_in_ICU}
\end{subfigure}
\begin{subfigure}{0.24\textwidth}
 \includegraphics[width=1.1\linewidth,height=5cm]{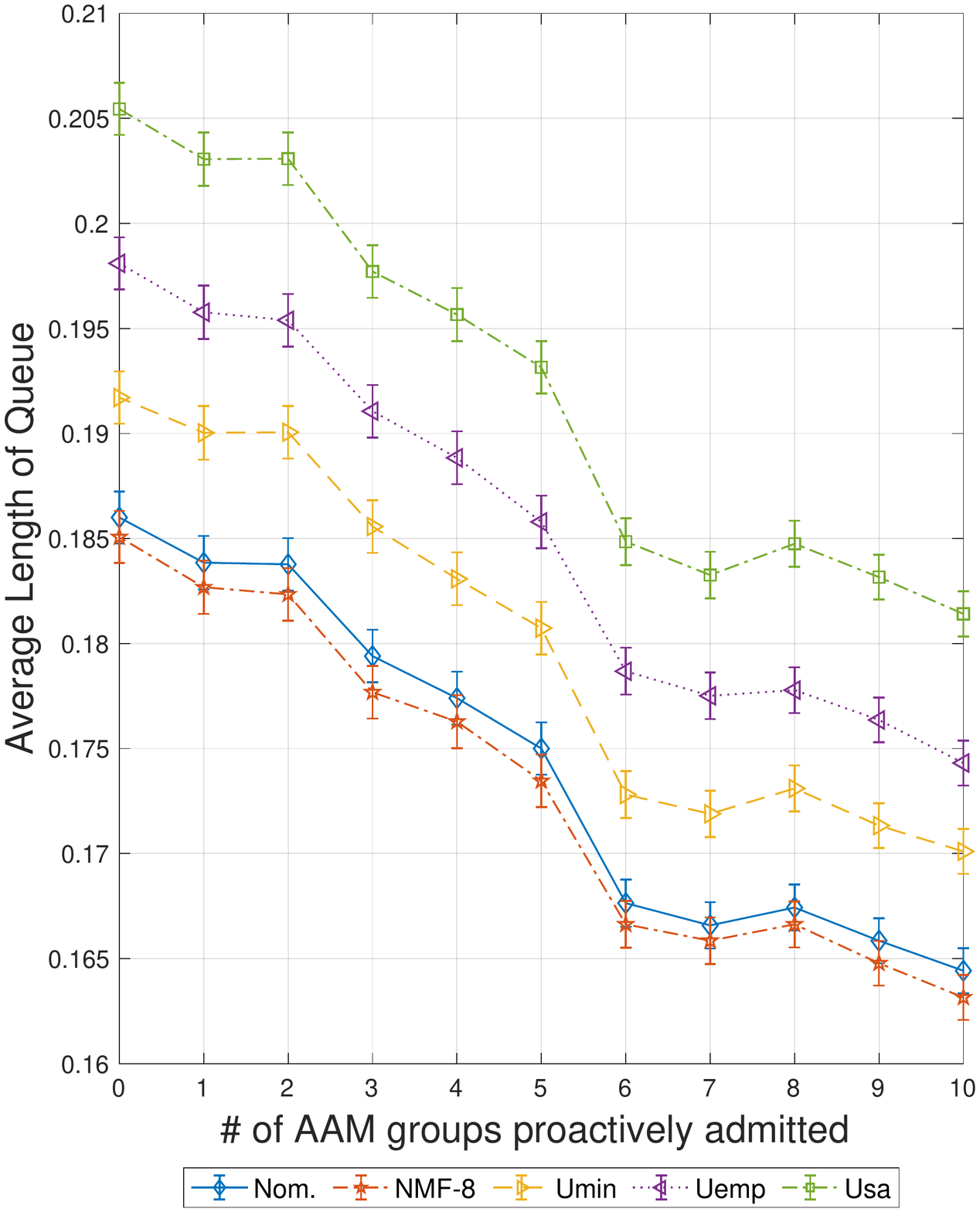}
\caption{Average length of the queue.}
\label{fig:avg_length}
\end{subfigure}
\begin{subfigure}{0.24\textwidth}
 \includegraphics[width=1.1\linewidth,height=5cm]{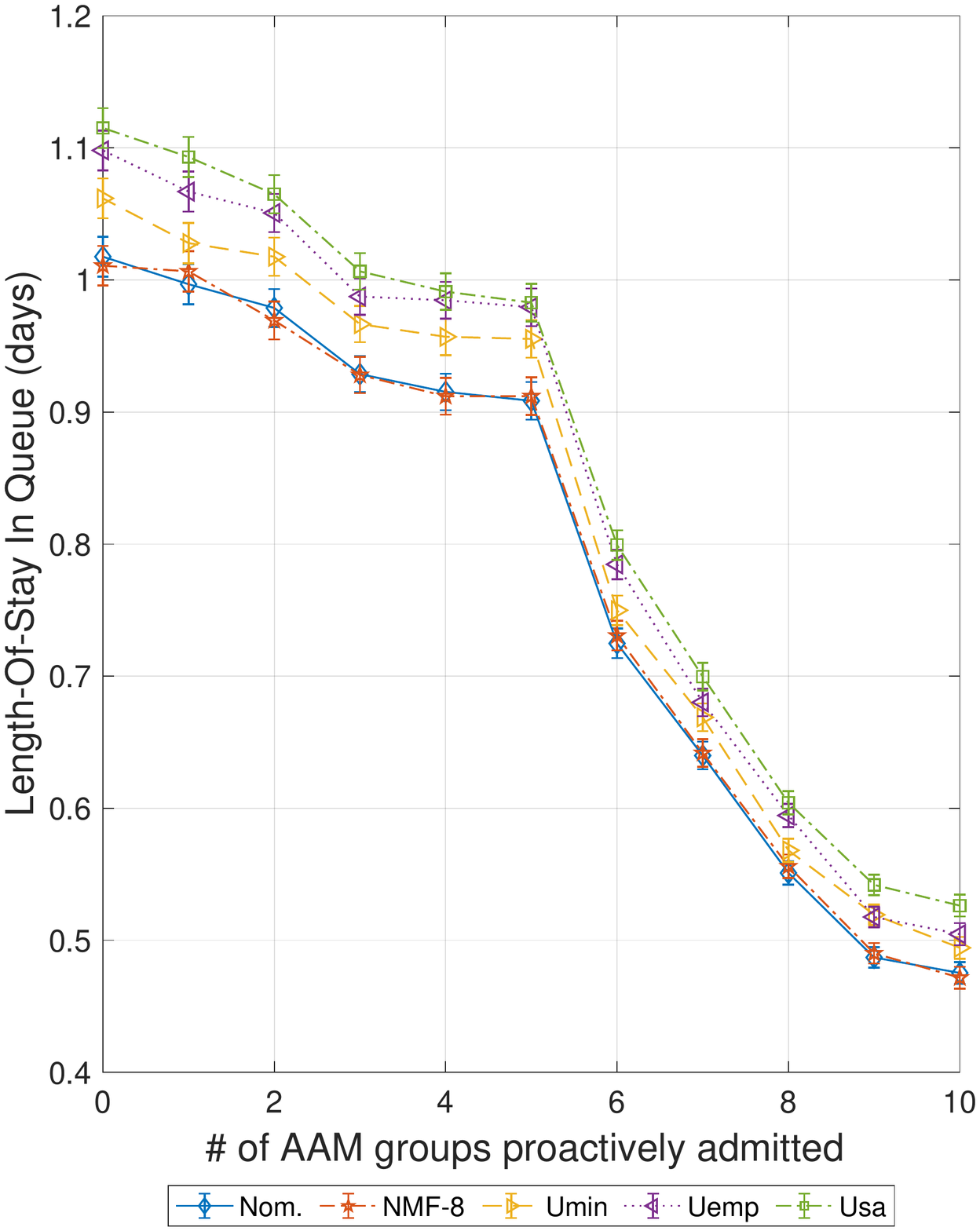}
\caption{Average LOS in the queue}
\label{fig:avg_LOS_in_queue}
\end{subfigure}
\caption{For our hospital with waiting queue, four different performance metrics in terms of the number of severity scores proactively transferred.}
\label{fig:queue_metrics}
\end{figure}
Overall,  this alternate hospital model with a waiting queue demonstrates similar beneficial impact of proactive transfer policies as in the primary hospital model without queue.
Note that in practice, we expect the discharge policy of the hospital to be a hybrid of these two models (some waiting and some demand-driven discharges). We have investigated these two extreme scenarios and found that in both cases, proactively transferring the sickest patients may decrease the mortality rate and LOS of the hospital, without significantly increasing the ICU occupancy.
}
{\color{black}
\subsection{Practical guidance.}
We now summarize here some practical implications and key take-aways of our numerical and theoretical analysis.
\begin{itemize}
\item \textbf{Impact of proactive transfer.}
Proactively transferring patients to the ICU may improve the average mortality rate and LOS, at the price of increasing the ICU occupancy (Figure \ref{fig:mort} and \ref{fig:LOS}) and Demand-Driven Discharge (DDD) rates (Figures \ref{fig:DDD_nom}-\ref{fig:DDD_sa}). Transferring only the sickest patients (here, the  top 10 \% of severity scores) does not lead to significant increases in ICU occupancy and DDD rates. The same conclusion holds when the hospital never demand-driven discharges patients from the ICU (Figures \ref{fig:mort_nom_with_queue}-\ref{fig:LOS_nom_with_queue}). In this case, proactively transferring patients may even improve the average time spent waiting in the queue (Figures \ref{fig:proba_types_in_queue}-\ref{fig:avg_LOS_in_queue}).

\item \textbf{Worst-case vs. random deviations.} The hospital metrics may significantly deteriorate for adversarial deviations in the transition parameters (Figures \ref{fig:mort_worst_bump_1} and \ref{fig:LOS_worst_bump_1}). In stark contrast, the hospital metrics look fairly stable under sensitivity analysis  (Figures \ref{fig:mort_rand_bump_1} and \ref{fig:LOS_rand_bump_1}).  This suggests that the naive approach of randomly perturbing the transition parameters in order to verify the effectiveness of our decisions may be misleading, as it does not take into account all potential events, some of which may have particularly adverse consequences even when the magnitude of the perturbation is small.
    
\item \textbf{Unrelated vs. Related Uncertainty sets.}  The deterioration of the hospital performance varies depending on the model of uncertainty; see Figures \ref{fig:mort_worst_bump_1} and  \ref{fig:LOS_worst_bump_1}. In particular, an unrelated model of uncertainty ($\U_{sa}$) leads to both more extreme deterioration in performance as well as anomalous observations -- e.g. that proactively transferring some patients may decrease the average ICU occupancy --  than related models of uncertainty ($\U_{\min}$ and $\U_{\sf emp}$). As the worst-case transition matrices for both types of uncertainty sets are similarly distanced from the nominal estimation $\bm{T}^{0}$ (as measured by the $1$-norm), the difference in optimal robust policies can be attributed to the rank-constrained nature of $\U_{\min}$ and $\U_{\sf emp}$. This lends additional support for the factor matrix uncertainty set being a more appropriate model of uncertainty for healthcare applications.
\end{itemize}
\paragraph{Practical insights.} Based on our theoretical and numerical analysis, our practical insights can be summarized as follows:

\vspace{1mm}
\noindent
\textit{Threshold policies can be very effective to guide decisions for proactive transfers as they are 1) easy to implement and 2) have good theoretical and numerical performance. These properties hold when there is no parameter uncertainty as well as when data challenges introduce parameter uncertainty. Our results suggest that when there is parameter uncertainty (e.g. due to limited data and/or unobserved covariates) providers should be slightly more aggressive in their transfer policies as compared to when there is no uncertainty.}
}
}
\section{Conclusions and Discussion.}
Interest in preventative and proactive care has been growing. With the advancements in machine-learning, the ability to conduct proactive care based on predictive analytics is quickly becoming a reality. In this work, we consider the decision to proactively admit patients to the ICU based on a severity score before they suffer a sudden health deterioration in the ward and require even more resources. In practice, an early warning system alert could trigger many potential interventions such as placing the patient in an evaluation state where admission decisions could be made from. While a threshold policy for proactive admission is a simplification of what could happen in practice when an alert is triggered, our analysis facilitates the derivation of valuable insights on the performance of this simple class of transfer policies and the impact of parameter uncertainty. We explicitly account for parameter uncertainty that arises naturally in practice due to the need to estimate model parameters from finite, real data which may also suffer from biases introduced by unobservable confounders. Since the severity scores are likely influenced by common underlying medical factors, we introduce a robust model that accounts for potential correlations in the uncertainty related to different severity scores. Under mild and interpretable assumptions, our model shows that the optimal nominal and optimal robust transfer policies are of threshold type, and that the optimal robust policy transfers more patients than the optimal nominal one.
Our extensive simulations show that not accounting for parameter misspecification may lead to overly optimistic estimations of the hospital performance, even for very small deviations. Moreover, we find that unrelated uncertainty may lead to extreme perturbations from the nominal parameters and unreliable insights for the impact of threshold policies on the patient flow in the hospital. Our work suggests that it is crucial for the decision-makers to account for parameters uncertainty when basing their decisions on  predictive models where some parameters are estimated from  real data.

One limitation of our work is the choice of worst-case as a relevant metric for the decision-maker. While worst-cases performance may be unlikely in practice, it is worth noting that the resulting parameter values for some worst-case performance are in the confidence intervals and therefore are as likely as the nominal parameters. Moreover, in the field of healthcare operations where the goal is to save the lives of the patients, it is still relevant to obtain an estimation of the potential deterioration of the performance of the hospital,
especially if the deviation from the nominal parameters is small (i.e., in the confidence intervals).
We would like to highlight that our work gives insight on potential misestimations of the metrics (average in-hospital mortality rate, length-of-stay and ICU occupancy) on which the physicians may base their decisions of a transfer policy.
More specifically, we provide a tool to estimate worst-case deterioration, within the confidence intervals given by some statistical estimators. It is to the merit of the physicians to decide what levels of risks are acceptable.

Another limitation is tied to the ability of the severity score  to fully capture the patient potential health deterioration.
We must recognize that the impact of proactive transfer policies will be highly dependent on the quality of the evaluation of the severity scores.
As such, proactive transfer policies could also be beneficial if based on other metrics, such as LAPS2, MEWS or others, as long as these scores accurately describe a predicted potential patient health deterioration.

There are various interesting directions for future research that arise from our work. For instance, one could consider various levels of actions for the proactive policies, ranging from a simple alert to the physicians for more continuous monitoring, to an immediate ICU transfer (the  action considered in this paper). Moreover, the proactive transfer policies considered in this work do not account for the number of empty beds in the ICU.
 One could consider \textit{adaptive thresholds}, varying with the number of free beds in ICU. While \cite{ICU-wenqi} shows with simulations that the performance of such adaptive threshold policies are comparable to the non-adaptive ones, it could be of interest to investigate the theoretical guarantees of such adaptive policies in the framework of our single-patient MDP. Finally, given the vast amount of patients trajectories available in our dataset, one could utilize recent methods from the \textit{off-policy evaluation} literature (see \cite{Kallus-1} for a review) to obtain a model-free performance estimator, in contrast to our model-based analysis using our single-patient robust MDP.
\small
\bibliographystyle{plainnat} 
\bibliography{ICU} 

%
%
%
\normalsize
\appendix
\section{Details about the hospital model of Figure \ref{fig:hosp_model}.}\label{app:detail_hosp}
We give details on the hospital model from Figure \ref{fig:hosp_model}, as introduced by the authors of \cite{ICU-wenqi}. As in \cite{ICU-wenqi}, the parameters of the model were calibrated using sample means across all $21$ hospitals and/or estimates from regression models.
\paragraph{Summary statistics.}
 {\color{black} The patient cohort is 53.80 \% female, the average age is 67.34 years (standard deviation (std): 17.71).  The average mortality is 3.2 \% (9.5 \% for patients who entered the ICU, 2.5 \% for patients who never entered the ICU). The mean LOS is 90.5 hours (std: 135.2, mean of 149.1 hours for patients who entered the ICU and mean 81.0 for patients who never entered the ICU).  14.2 \% of all hospitalizations were eventually admitted to the ICU.}  In order to have a more accurate estimate for the dynamics of the most severe patients, the number of patients in each severity group is nonuniform. 
\paragraph*{Direct admits.} The arrivals of the patients who are directly admitted to the ICU follows a non-homogeneous Poisson process. The empirical arrival rates are estimated using $12$ months of data across all $21$ hospitals. The LOS of a direct admit patient is log-normally distributed with mean $1/\mu_{E}$ and standard deviation $\sigma_{E},$ and a proportion $p_{E}$ (following a distribution with density $f_{p_{E}}$) of this LOS is spent in the ICU, while the remaining time is spent in the ward. The rate of readmission to the ICU is $\rho_{E}$. At the end of this LOS, the patients are discharged with probability $1-d_{E}$. The value of the parameters are the following: $d_{E}=9.41 \%, \rho_{E}= 15.76 \%, 1/\mu_{E}=5.49$ (days), $\sigma_{E}=5.71$ (days), $\E[p_{E}]=50.79 \%.$ The density $f_{p_{E}}$ is the empirical distribution derived from the dataset.
\paragraph*{Transfer from the ward.}
These patients can be divided into the patients who have versus have not already been to the ICU.
Consider the patients who have never been to the ICU. A patient arrives in the ward with a severity score of $i \in \{1,...,n\}$ following a non-homogeneous Poisson process. Every $6$ hours, s/he then transitions to another risk score $j$ with probability $T^{0}_{ij}$, or s/he may 	`crash' and require ICU admission, recover and leave the hospital, or die. After a patient has crashed, a LOS is chosen which is log-normally distributed with mean $1/\mu_{C}$ and standard deviation $\sigma_{C},$ and a proportion $p_{W}$ (following a distribution with density $f_{p_{W}}$) of this LOS is spent in the ICU, while the remaining proportion $1-p_{W}$ of time is spent in the ward. At the end of this LOS, the patient is discharged with probability $1-d_{C}$. If there are no available beds in the ICU when a crashed patient requires an admission, the ICU patient with the shortest remaining service time in the ICU will be discharged, and this is called a ``demand-driven discharge''. Such a patient will have a readmission rate of $\rho_{D}$, higher than the readmission rate $\rho_{C}$ of ward patients who were naturally discharged from the ICU after finishing their service time in the ICU. In particular, we have the following values, estimated through empirical averages in our dataset: $\rho_{C}= 16.88 \%, \rho_{D}=18.13 \%, 1/ \mu_{C}=12.54$ (days), $\sigma_{C}=10.13 $ (days), $\E[p_{W}] = 46.92 \%, $ and $d_{C} = 57.28 \%$. Similarly, the density $f_{p_{W}}$ is derived as the empirical distribution from the dataset.
\paragraph*{Proactive transfer.}
Every $6$ hours, the doctors might perform a proactive transfer and transfer a patient in the ward to the ICU, if there is an available bed. When a patient is proactively transferred, the hospital LOS is log-normally distributed with mean $1/\mu_{A,i}$ and standard deviation $\sigma_{A,i}$, while a proportion $p_{W} \sim f_{p_{W}}$ of this LOS is spent in the ICU. The patient will then survive the hospital discharge with a probability $1-d_{A,i}$. If this patient is naturally discharged from the ICU, the readmission rate is $\rho_{a,i}=\rho_{C}$, otherwise it is $\rho_{D}$.
We indicate below the proportion, the mortality rate and the LOS related to each $n=10$ severity scores.
\begin{center}\label{tab:repartition-severity-score}
\begin{table}[h]
\begin{tabular}{ |c|cccccccccc |}
 \hline
Severity score $i$ & 1 & 2 & 3 & 4 & 5 & 6 & 7 & 8 & 9 & 10 \\
\hline
Proportion $(\%)$ & 17.6 & 20.3 & 20.0 & 14.9 & 16.9 & 2.2 & 2.0 & 2.0 & 2.0 & 2.0 \\
 \hline
Mortality $d_{A,i}$ $(\%)$ & 0.01 & 0.02 & 0.04 & 0.05 & 0.11 & 0.18 & 0.28 & 0.39 & 0.70 & 6.84 \\
 \hline
 LOS average $1/\mu_{A,i}$ (days) & 0.85 & 0.91, & 0.97 & 1.04, & 1.17 & 1.36 & 1.45 & 1.57 & 1.85 & 3.77 \\
 \hline
  LOS std $\sigma_{A,i}$ (days) & 0.68 & 0.74 & 0.78 & 0.84 & 0.95 & 1.10 & 1.17 & 1.27 & 1.50 & 3.04 \\
 \hline
\end{tabular}
\caption{Statistics of patients across the $10$ severity scores.}
\end{table}
\end{center}
\section{Proof of Lemma \ref{app-lem-Vinfty}}\label{app-lem-Vinfty}
\begin{proof}

We will prove that for any policy $\pi$ and for any transition matrix $\bm{T}$, for all severity score $i \in [n]$, the value vector $\boldsymbol{V}$ of policy $\pi$ satisfies
$V_{i} \leq  r_{W}+\lambda \cdot r_{RL}.$
Let $ i \in [n]$.  By definition, $V_{i}$ is the expected infinite-horizon reward starting from state $i$:
$V_{i} = E^{\pi, \boldsymbol{T}} \left[ \sum_{t=0}^{\infty} \lambda^{t}r_{i_{t}a_{t}} \; \bigg| \; i_{0} = i \right].$
Let us consider a trajectory $O$ of the Markov chain on $\X$ associated with $(\pi,\boldsymbol{T}).$ Then either the patient stays infinitely in the ward, in which case the reward is $r_{W} \cdot (1-\lambda)^{-1}$, which is smaller than $r_{W}+ \lambda \cdot r_{RL}$ by Assumption \ref{assumption-0}. Otherwise, during the trajectory $O$, there is a time $t$ at which the patient leaves the ward and reaches the state $n+1=CR$,  $n+2 = RL$,  $n+3=D$, or $n+4= PT$. In that case the reward is smaller than $\dfrac{r_{W} \cdot (1-\lambda^{t})}{1-\lambda} + \lambda^{t+1} \cdot \max \{r_{CR},r_{RL},r_{D},r_{PT}\}.$
Since the maximum instantaneous reward is $r_{RL}$, the reward associated with the trajectory $O$ is smaller than $\dfrac{r_{W} \cdot (1-\lambda^{t})}{1-\lambda} + \lambda^{t+1} \cdot r_{RL}.$
Now
\begin{align}
\dfrac{r_{W} \cdot (1-\lambda^{t})}{1-\lambda} + \lambda^{t+1} \cdot r_{RL} & \leq (r_{W}+\lambda \cdot r_{RL}) \cdot (1-\lambda^{t}) + \lambda^{t+1}\cdot r_{RL} \label{ineq:lem-Vinfty}\\
& \leq r_{W}+ \lambda \cdot r_{RL},
\end{align}
where Inequality \eqref{ineq:lem-Vinfty} follows from Assumption \ref{assumption-0}.
Therefore, the reward associated with any trajectory $O$ is smaller than $r_{W}+ \lambda \cdot r_{RL}.$ We can thus conclude that the value vector $\boldsymbol{V}$ satisfies $V_{i} \leq r_{W} + \lambda \cdot r_{RL}, \forall \; i \in [n]. $  
\end{proof}
{\color{black}
\section{$N$-patient MDP}\label{app:N-patient-MDP}
\subsection{Model of patients dynamics.}
In order to give a theoretical justification to penalizing the reward $r_{PT}$ in our single-patient MDP, we present here an $N$-\textit{patient MDP}. This is a model whose complexity is somewhat intermediary between the (simpler) single-patient MDP and the (more complex) full hospital model.  It consists of $N \in \N$ patients, whose individual dynamics of the severity condition follows the single-patient MDP described in Figure \ref{fig:mdp_model}. The action is to choose to transfer, or not, patients to the ICU. Crucially, there is a constraint on the number of patients that can be transferred at the same time to the ICU.

\textit{Set of states.} A state is an $N$-tuple $\left( i_{1}, ..., i_{N} \right)$, where each $i_{\ell} \in \{1,..., N\}$ describes the severity condition of patient $\ell$, or one of the four absorbing states $r_{RL},r_{CR},r_{D},r_{PT}$. Note that the number of states, $(n+4)^{N}$, is exponential in the number of patients.

\textit{Set of actions and policy.} An action is an $N$-tuple $\left(\pi^{1}_{i_{1}}, ..., \pi^{N}_{i_{N}} \right)$ where $\pi^{\ell}_{i_{\ell}}$ represents the probability to proactively transfer a patient with severity score $i_{\ell}$. The number of (deterministic) actions is also exponential in $N$.
A policy is an $N$-tuple $\left(\pi^{1}, ..., \pi^{N} \right)$ where each $\pi^{\ell}$ is a map from the set of severity conditions $[n]$ to $[0,1]$, i.e. $\pi^{\ell}$ is a policy \textit{in the single-patient MDP.}

In order to incorporate the effect of proactive transfer on other patients in the hospital, we choose to incorporate  a constraint in our decision-making problem. In particular, the chosen action $\left(\pi^{1}_{i_{1}}, ..., \pi^{N}_{i_{N}} \right)$ for state $\left(i_{1}, ..., i_{N}\right)$ has to satisfy
\begin{equation}\label{eq:constraint-pi}
\sum_{\ell=1}^{N} \pi^{\ell}_{i_{\ell}} \leq m,
\end{equation}
where $m \in \N$ is the maximum number of patients that can proactively transferred at the same time.
Note that in practice, the person  (or group of persons) who makes the decision to transfer  patients from the ward to the ICU may not have access to an accurate estimate of the ICU capacity. The current analysis provides insights into the case that an accurate estimate for the capacity constraint is available.

\textit{Transitions.} Once action $\left(\pi^{1}_{i_{1}}, ..., \pi^{N}_{i_{N}} \right)$ is chosen in state $\left(i_{1}, ..., i_{N}\right)$, the patients who are proactively transferred transition to state $PT$, and other patients all transition to a next severity condition (or $RL,CR,D$) following the same transition matrix $\bm{T}^{0}$ as for the single-patient MDP.

\textit{Rewards.} The rewards only depend on the state $\left(i_{1}, ..., i_{N}\right)$ and is simply chosen as $\sum_{\ell=1}^{N} r_{i_{\ell}}$, where $r_{i_{\ell}}$ are the same rewards as in the single-patient MDP.

\paragraph{The Single-patient MDP as a Lagrangian relaxation of the $N$-patient MDP.}
We now show that our single-patient MDP is a Lagrangian relaxation of this $N$-patient MDP, in the sense of \cite{adelman2008relaxations}. In particular, the number of states and actions of the $N$-patient MDP are exponential in $N$, rendering it intractable.
While the Bellman equation for value vector $\bm{J}$ of the $N$-patient MDP is
\begin{equation}\label{eq:Bellman-N-patients}
J_{\left(i_{1}, ..., i_{N} \right)} = \max_{\left(\pi^{1}_{i_{1}}, ..., \pi^{N}_{i_{N}} \right); \sum_{\ell=1}^{N} \pi^{\ell}_{i_{\ell}} \leq m } \; \sum_{\ell=1}^{N} r_{i_{\ell}} + \lambda \sum_{\left(i_{1}', ..., i_{N}' \right) \in \X} \left( \Pi_{\ell=1}^{N} T^{0}_{i_{\ell},i_{ell}'} \right) J_{\left(i_{1}', ..., i_{N}'\right)}, \forall \; \left(i_{1}, ..., i_{N} \right) \in \X,
\end{equation}
the authors in \cite{adelman2008relaxations} suggest to approximate this equation by another equation on a vector $\bm{J}^{ \mu}$, for  a Lagrange multiplier $\mu \geq 0$:
\begin{equation}\label{eq:Bellman-N-patients-mu}
\begin{aligned}
J^{\mu}_{\left(i_{1}, ..., i_{N} \right)} = &  \max_{\left(\pi^{1}_{i_{1}}, ..., \pi^{N}_{i_{N}} \right)} \; \sum_{\ell=1}^{N} r_{i_{\ell}} + \lambda \sum_{\left(i_{1}', ..., i_{N}' \right) \in \X} \left( \Pi_{\ell=1}^{N} T^{0}_{i_{\ell},i_{ell}'} \right) J^{\mu}_{\left(i_{1}', ..., i_{N}'\right)} \\
& + \mu \left( m - \sum_{\ell=1}^{N} \pi^{\ell}_{i_{\ell}} \right) , \forall \; \left(i_{1}, ..., i_{N} \right)  \in \X.
\end{aligned}
\end{equation}
In this case,  we obtain that
\begin{equation}\label{eq:Bellman-N-sep}
J^{\mu}_{\left(i_{1}, ..., i_{N} \right)}  = \dfrac{m \cdot \mu}{1-\lambda} + \sum_{\ell=1}^{N} \bm{V}^{\mu}_{\ell}(i_{\ell}), \forall \; \left(i_{1}, ..., i_{N} \right) \in \X,
\end{equation}
where $\bm{V}^{\mu}_{\ell} \in \R^{n}$ are value vectors for the (penalized) single-patient MDPs:
\begin{equation}\label{eq:Bellman-v-mu}
V_{\ell}^{\mu}(i_{\ell}) = \max_{\pi^{\ell}_{i_{\ell}} \in [0,1]} r_{i_{\ell}} - \mu \pi^{\ell}_{i_{\ell}} + \lambda \left( (1-\pi^{\ell}_{i_{\ell}}) \bm{T}_{i}^{0 \; \top} \bm{V}^{\mu}_{\ell} + \pi^{\ell}_{i_{\ell}} r_{PT} \right), \forall \; i_{\ell} \in [n].
\end{equation}
Note that we can reformulate \eqref{eq:Bellman-v-mu} as
\begin{equation}\label{eq:Bellman-v-mu-penalty}
V_{\ell}^{\mu}(i_{\ell}) = \max \{  r_{i_{\ell}} + \lambda \bm{T}_{i}^{0 \; \top} \bm{V}_{\ell}^{\mu},   r_{i_{\ell}} + \lambda \left(r_{PT}  - \dfrac{\mu}{\lambda} \right) \}, \forall \; i_{\ell} \in [n].
\end{equation}
Equation \eqref{eq:Bellman-v-mu-penalty} is the Bellman equation for our single-patient MDP, except for the term $- \dfrac{\mu}{\lambda}$.

Therefore,  dualizing the binding constraint in our weakly constrained $N$-patient MDP amounts to penalizing the (terminal) reward for proactively transferring the patient.

\subsection{Whittle Indexability of the single-patient MDP.}\label{app:whittle}
This leads us to investigate the sensitivity of the optimal nominal policy to the reward $r_{PT}$ associated with proactive transfer.
 This is related to the notion of \textit{Whittle index} and \textit{indexability} \citep{whittle1988restless} of the optimal policy in our single-patient MDP.
\begin{definition}
For a given state $i$, the Whittle index of state $i$ is the choice of reward $r_{PT}$ such that it is equally desirable to proactively transfer the patient in state $i$ and to not proactively transfer her/him.
\end{definition}
We note that the Whittle index is a priori dependent on the state $i$ and may vary across states.
We now prove that the single-patient MDP is \textit{indexable.}
\begin{definition}
Let $\I(r_{PT})$ be the set of severity conditions for which the optimal action is to proactively transfer the patient to the ICU.

The single-patient MDP is \textit{indexable} if $\I(r_{PT})$ is monotonically increasing from $\emptyset$ to $[n]$ (the set of all severity conditions) when the reward $r_{PT}$ is increasing from $r_{PT}=0$ to $r_{PT}=r_{RL}$.
\end{definition}
\begin{proposition}\label{prop:indexable}
Assume that Assumption \ref{assumption-1} holds for $r_{PT} =0$.
Then the single-patient MDP is indexable for all $r_{PT} \in [0,r_{RL}]$.
\end{proposition}
We can reformulate this in more practical terms with the following immediate corollary.
\begin{corollary}\label{cor:whittle-interpretation}
The threshold of an optimal policy in the single-patient MDP is decreasing with $r_{PT}$.
\end{corollary}
The proof of Proposition \ref{prop:indexable} relies on two lemmas. We start by the following analysis of the sensitivity of Assumption \ref{assumption-1} as regards to the value of $r_{PT}$.
\begin{lemma}\label{lem:assumption-3-3}
Assume that Assumption \ref{assumption-1} holds for $r_{PT}=0$.
Then Assumption \ref{assumption-1} holds for all $r_{PT} \in [0,r_{RL}]$.
\end{lemma}
\begin{proof}

Note that condition \eqref{eq:assumption:sub-1} does not depend of $r_{PT}$, while if condition \eqref{eq:assumption:sub-2} is satisfied for $r_{PT}=0$, it is satisfied for any $r_{PT} \in [0,r_{RL}]$, since the left-hand side is increasing with $r_{PT}$ while the right-hand side does not depend of $r_{PT}$.
Therefore, if we assume that Assumption \ref{assumption-1} holds for $r_{PT}=0$, we know that it holds for any $r_{PT} \in [0,r_{RL}]$.

\end{proof}
We also need the following lemma, which relates variations in $r_{PT}$ to variation in the value vector.
\begin{lemma}\label{lem:v-star-epsilon}
Let $V^{*}$ the optimal value vector for a choice of $r_{PT}$. Now let us consider a reward $r'_{PT} = r_{PT} + \epsilon$. Let $\bm{V}^{*'}$ the new value vector.
Then
\begin{equation}\label{eq:v-star-v-epsilon}
V_{j}^{*'} \leq V_{j}^{*} + \epsilon, \; \forall \; j \in [n].
\end{equation}
\end{lemma}
Note that \eqref{eq:v-star-v-epsilon} simply means that increasing the reward for $PT$ by $\epsilon$ does not increase the value vector by more than $\epsilon$.
\begin{proof}

We can prove \eqref{eq:v-star-v-epsilon} in the same way that we prove Lemma \ref{lem-Vinfty}. In particular, by definition
\[ V_{j}^{*'} = E^{\pi, \bm{T}} \left[ \sum_{t=0}^{+ \infty} r_{i_{t},a_{t}} \; | \: i_{0} = j. \right] \]

We can consider a trajectory $O$ of the Markov chain associated with $\pi^{*'}, \bm{T}^{0}$. Along this trajectory, if the state $PT$ is visited, at time $t \geq 1$ a reward $\lambda^{t} r'_{PT} = \lambda^{t} \left( r_{PT}+\epsilon \right)$ is obtained and the trajectory ends. Otherwise, the trajectory goes on until it reaches another absorbing state. Therefore, along this trajectory, the total accumulated reward when the reward for $PT$, $r'_{PT}$, is within $ \lambda \epsilon$ of the  total accumulated reward for the same trajectory when the reward for $PT$ is $r_{PT}$. The reason for the factor $\lambda$ is because the state $PT$ has to be visited after at least one period, starting from a severity condition $i$. Taking the expectation over all possible trajectories, we obtain \eqref{eq:v-star-v-epsilon}.

\end{proof}

We are now ready to prove Proposition \ref{prop:indexable}.
\begin{proof}[Proof of Proposition \ref{prop:indexable}.]

Our proof proceeds in three steps.
\begin{itemize}
\item \textbf{Optimal policies are threshold} Following Lemma \ref{lem:assumption-3-3}, we know that Assumption \ref{assumption-1} is satisfied for $r_{PT} \in [0,r_{RL}]$. Therefore, following Theorem \ref{th:nom-thr}, an optimal policy can be chosen threshold for every $r_{PT} \in [0,r_{RL}]$.

\item \textbf{$\I(r_{PT})$ for $r_{PT}=0$ and $r_{PT}=r_{RL}$.}

Recall the Bellman equation: for a given choice of reward $r_{PT}$, let $\bm{V}^{*}$ the value vector of an optimal policy. Then the optimal action for severity condition $i$ is chosen as the argmax in
\begin{equation}\label{eq:Bellman-loc}
\max \{ r_{W} + \lambda \bm{T}_{i}^{0 \; \top}\bm{V}^{*}, r_{W} + \lambda r_{PT} \}.
\end{equation}
When $r_{PT}=0$, it is straightforward that the max of the above problem \eqref{eq:Bellman-loc} is always  $r_{W} + \lambda \bm{T}_{i}^{0 \; \top}\bm{V}^{*}$. Therefore, when $r_{PT}=0$, the optimal action for any severity condition is to not proactively transfer, i.e. $\I(0) = \emptyset$.

Now, we also know from Lemma \ref{lem-Vinfty} that $V^{*}_{s} \leq r_{W} + \lambda r_{RL}$ (as long as $r_{RL}$ is the maximum of the instantaneous rewards). Therefore, when $r_{PT} = r_{RL}$, an optimal action at each severity condition $i$ is to proactively transfer, i.e. $\I(r_{RL}) = [n]$.

Note that both these conclusions (for the cases $r_{PT}=0$ and $r_{PT} = r_{RL}$) are easy to interpret in our healthcare setting: when there is no reward for proactive transfer, we do not have any incentive to proactively transfer any patients, whereas when there is as much reward for proactive transfer as for recovering and leaving the ward (state $RL$), we proactively transfer every severity condition.

\item \textbf{Monotonicity of $\I(r_{PT})$.}
We now prove that $\I(r_{PT})$ is monotonically increasing. Using Lemma \ref{lem:assumption-3-3}, we can always choose an optimal policy that is threshold. Now we want to show that the threshold of the optimal policy is monotonically \textit{non-increasing} (so that $\I(r_{PT})$ is increasing).

Let us consider a choice of the reward $r_{PT}$ such that
\begin{equation}\label{eq:v-star}
 r_{W} + \lambda \bm{T}_{i}^{0 \; \top}\bm{V}^{*} \leq r_{W} + \lambda r_{PT} .
 \end{equation}
Note that $V^{*}$ depends on $r_{PT}$. Now let us consider a reward $r'_{PT} = r_{PT} + \epsilon$. Let $\bm{V}^{*'}$ the new value vector.  Our goal is to prove
\begin{equation}\label{eq:v-star-prime}
r_{W} + \lambda \bm{T}_{i}^{0 \; \top}\bm{V}^{*'}< r_{W} + \lambda r_{PT} + \epsilon
\end{equation}
so that the optimal policy for $r'_{PT}$ still proactively transfers a patient in a severity condition $i$, and $r_{PT} \mapsto \I(r_{PT})$ is monotonically increasing.

Now we obtain
\begin{align}
r_{W} + \lambda \bm{T}_{i}^{0 \; \top}\bm{V}^{*'} &  \leq r_{W} + \lambda \bm{T}_{i}^{0 \; \top}\bm{V}^{*} + \epsilon  \label{eq:prf-wittle-0}   \\
& \leq  r_{W} + \lambda r_{PT} + \lambda \epsilon \label{eq:prf-wittle-1}\\
& \leq r_{W} + \lambda r'_{PT},\label{eq:prf-wittle-2}
\end{align}
where \eqref{eq:prf-wittle-0} follows from \eqref{eq:v-star-v-epsilon} and the fact that $\bm{T}^{0}$ is a transition matrix, \eqref{eq:prf-wittle-1} follows from \eqref{eq:v-star}, and \eqref{eq:prf-wittle-2} follows from the definition of $r'_{PT}$.
\end{itemize}
\end{proof}

Finally, we note that computing a closed-form solution for the Whittle index of state $i$ appears more challenging. In particular, the vector $\bm{V}^{*}$ depends on the value of $r_{PT}$ and a patient in severity condition $i$ can transition to any other severity condition (and terminal states $RL,CR$ and $D$), while closed-forms computation of the Whittle index appear in problems with simpler dynamics \citep{whittle1988restless,hsu2018age,tripathi2019whittle}.
\subsection{Sensitivity as regards to the constraint on $m$.}

Our $N$-patient MDP relies on an \textit{exogenous} static parameter $m$, the maximum number of patients that can be proactively transferred at any time period. While this is a limitation of the generality of the $N$-patient MDP, it remains an interesting extension of the single-patient MDP toward more realistic models of hospital patient flow dynamics. We use the results of the previous section on the \textit{indexability} of the single-patient MDP in order to better understand the impact of the parameter $m$ on the policies of the single-patient MDP.

In particular, we show the following result:
\begin{proposition}\label{prop:m-threshold}
The threshold of an optimal nominal policy in the single-patient MDP is decreasing with  $m$.
\end{proposition}
In other words, as the maximum number of patients that can be proactively transferred at the same time in the $N$-patient MDP increases, an optimal policy in the single-patient MDP transfers more patients.
\begin{proof}[Proof of Proposition \ref{prop:m-threshold}.]

Remember that $\bm{p}_{0}$ is the initial distribution over the set of states $\{1,...,n\}$.
From Equation \eqref{eq:Bellman-N-sep}, the goal of the decision-maker in the relaxation of the $N$-patient MDP is to solve
\begin{align*}
\min_{\mu \geq 0} \; \; \dfrac{m \cdot  \mu}{1-\lambda} + \sum_{\ell=1}^{N} \bm{p}_{0}^{\top}\bm{V}^{\mu}_{\ell}.
\end{align*}
By a standard duality argument, this is a convex minimization problem. Now, let us consider $m$ and $m'$ such that $m \leq m'$. Let $f: \mu \mapsto  \sum_{\ell=1}^{N} \bm{p}_{0}^{\top}\bm{V}^{\mu}_{\ell}$. We write $g_{m}: \mu \mapsto \dfrac{m \cdot  \mu}{1-\lambda} + f(\mu)$ and $g_{m'}: \mu \mapsto \dfrac{m' \cdot  \mu}{1-\lambda} + f(\mu)$. Clearly,  for any $\mu \geq 0$, we have $g_{m'}'(\mu) = m' + f'(\mu) \geq m +f'(\mu) = g'_{m}(\mu).$

Let $\mu^{*}$ and $\mu^{' \; *}$ denote the optimal Lagrange multipliers for $m$ and $m'$, characterized by
$g'_{m}(\mu^{*})=0$ and $g'_{m'}(\mu^{' \; *}) =0$.
 Therefore $0 = g'_{m'}(\mu^{' \; *}) \geq g'_{m}(\mu^{' \; *}).$ Since $g$ is convex, this means that $\mu^{*} \geq \mu^{' \; *}$.

Therefore, we have proved that the optimal Lagrange multiplier $\mu^{*}$ is a \textit{non-increasing} function of $m$.

In order to conclude our proof of Proposition \ref{prop:m-threshold}, we note that we have proved the following.

\begin{itemize}
\item  When $m$ increases, the optimal Lagrange multiplier $\lambda^{*}$ is non-increasing.
\item This results in higher reward for proactive transfer in the single-patient MDP, as a Lagrange multiplier of $\mu$ in the $N$-patient MDP results in a penalty of $-\mu/\lambda$ for $r_{PT}$ in the single-patient MDP.
\item In Corollary \ref{cor:whittle-interpretation}, we have proved that the threshold of an optimal policy is monotonically decreasing as $r_{PT}$ is increasing.
\item Therefore, we can conclude that the threshold of an optimal policy in the single-patient MDP is non-increasing as the maximum number of proactive transfers augments in the $N$-patient MDP.
\end{itemize}
\end{proof}

\subsection{Summary of the $N$-patient MDP.}
We have proved the following.
\begin{itemize}
\item \textbf{Lagrangian relaxation.}
Our single-patient MDP can be viewed as the Lagrangian relaxation of a more complex $N$-patient MDP. The Lagrangian multiplier $\mu \geq 0$ becomes a penalty for the reward $r_{PT}$ in the single-patient MDP.
\item \textbf{Whittle index.} We proved that the threshold of an optimal policy in the single-patient MDP is \textit{monotonically} decreasing as $r_{PT}$ is increasing. This is related to the \textit{Whittle index} of the single-patient MDP and we show that the single-patient MDP is indexable.
\item \textbf{Exogenous parameter $m$.} Finally, we study the impact of the parameter $m$ (the maximum number of patients that can be proactively transferred at the same time, reflecting the ICU capacity constraint) on the threshold of an optimal policy in the single-patient MDP. Leveraging our results on the Whittle index, we show that as $m$ increases in the $N$-patient MDP, an optimal policy transfers more and more patients in the single-patient MDP.  Thus, one can expect that it is reasonable to be more aggressive with proactive transfers as the number of available ICU beds increases. 
\end{itemize}
}

\section{Homogeneity and Translation for conditions \eqref{eq:assumption:sub-1} and \eqref{eq:assumption:sub-2}.}\label{app:Lemmas-tr-sc}
We prove the following two lemmas.
\begin{lemma}\label{lem:assumption-simple-tr-sc}
Let $(r_{W},r_{CR},r_{RL},r_{D},r_{PT}) \in \R^{5}_{+}$ denote some rewards and $\bm{T}$ a transition matrix such that condition \eqref{eq:assumption:sub-1} and condition \eqref{eq:assumption:sub-2} hold.
\begin{enumerate}
\item Let $\alpha \geq 0$. For $\alpha \cdot (r_{W},r_{CR},r_{RL},r_{D},r_{PT})$ and $\bm{T}$, condition \eqref{eq:assumption:sub-1} and condition \eqref{eq:assumption:sub-2} still hold.
\item Let $\alpha \geq 0.$ For $ (r_{W} + \alpha,r_{CR}+ \alpha,r_{RL}+ \alpha,r_{D}+ \alpha,r_{PT}+ \alpha)$ and $\bm{T}$, condition \eqref{eq:assumption:sub-2} still holds.
\end{enumerate}
\end{lemma}
\begin{proof}

\begin{enumerate}
\item This follows from $\alpha \geq 0$ and $\alpha \cdot r_{CR} \cdot T_{i,n+1}^{0} + \alpha \cdot  r_{RL} \cdot T_{i,n+2}^{0} +\alpha \cdot r_{D} \cdot T_{i,n+3}^{0} = \alpha \cdot ( r_{CR} \cdot T_{i,n+1}^{0} + r_{RL} \cdot T_{i,n+2}^{0} + r_{D} \cdot T_{i,n+3}^{0}), $ and
$ \dfrac{\alpha \cdot r_{W} + \lambda \cdot \alpha \cdot  r_{PT}}{\alpha \cdot r_{W} + \lambda \cdot \alpha \cdot  r_{RL}} = \dfrac{ r_{W} + \lambda \cdot  r_{PT}}{ r_{W} + \lambda  \cdot  r_{RL}}.$
\item Let us assume that $
     \dfrac{r_{W} + \lambda \cdot r_{PT}}{r_{W} + \lambda \cdot r_{RL}} \geq  \dfrac{\left( \sum_{j=1}^{n} T_{i+1,j}^{0} \right) }{\left( \sum_{j=1}^{n} T_{ij}^{0} \right)}, \forall \; i \in [n-1].
$
We write $\phi$ the function of $\R$ such that for any scalar $\alpha$,
$\phi(\alpha)=\dfrac{r_{W} + \alpha + \lambda \cdot ( r_{PT}+ \alpha )}{r_{W}+ \alpha  + \lambda \cdot (r_{RL}+ \alpha )}.$ We will prove that $\phi$ is non-decreasing and therefore that $$\phi(\alpha) \geq \phi(0) = \dfrac{r_{W} + \lambda \cdot r_{PT}}{r_{W} + \lambda \cdot r_{RL}} \geq  \dfrac{\left( \sum_{j=1}^{n} T_{i+1,j}^{0} \right) }{\left( \sum_{j=1}^{n} T_{ij}^{0} \right)}, \forall \; i \in [n-1].$$
Indeed, $\phi$ has a derivative in $\R_{+}$ and
$\phi'(\alpha) =  \dfrac{\lambda (1+\lambda) \cdot (r_{RL}-r_{PT})}{(r_{W}+ \alpha  + \lambda \cdot (r_{RL}+ \alpha ))^{2}} \geq 0,$
since $r_{RL} \geq r_{PT}.$ Therefore $\phi$ is a non-decreasing function, and for all $\alpha \geq 0$, the condition \eqref{eq:assumption:sub-2} holds. 
\end{enumerate}
\end{proof}

\begin{lemma}\label{lem:assumption-1-tr-sc}
Let $(r_{W},r_{CR},r_{RL},r_{D},r_{PT}) \in \R^{5}_{+}$ denote some rewards and $\bm{T}$ a transition matrix such that condition \eqref{assumption-2} holds.
\begin{enumerate}
\item Let $\alpha \geq 0.$ For $\alpha \cdot (r_{W},r_{CR},r_{RL},r_{D},r_{PT})$ and $\bm{T}$, condition \eqref{assumption-2} still holds.
\item Let $\alpha \geq 0$ and let us assume that
$ \sum_{j=1}^{n} T_{ij} \geq \sum_{j=1}^{n} T_{i+1,j}, \forall \; i \in [n-1].$
Then for $ (r_{W}+\alpha,r_{CR}+\alpha,r_{RL}+\alpha,r_{D}+\alpha,r_{PT}+\alpha)$ and $\bm{T}$, condition \eqref{assumption-2} still holds.
\end{enumerate}
\end{lemma}
\begin{proof}

\begin{enumerate}
\item Let $\alpha$ be a non-negative scalar. For the same reason as in Lemma \ref{lem:assumption-simple-tr-sc}, condition \eqref{assumption-2} still holds for $\alpha \cdot (r_{W},r_{PT},r_{RL},r_{CR})$ and $\bm{T}$.
\item Let $\alpha \geq 0.$ For any $i \in [n-1]$, we have
\begin{align}\label{eq:app:lem:3}
     \left( \sum_{j=1}^{n} T_{ij}^{0} \right) \cdot ( r_{W}+ \lambda \cdot r_{PT}) +  out(i) \geq \left( \sum_{j=1}^{n} T_{i+1,j}^{0} \right) \cdot (r_{W} + \lambda \cdot r_{RL}) +out(i+1),
\end{align}
where $ out(i)= r_{CR} \cdot T_{i,n+1}^{0} + r_{RL} \cdot T_{i,n+2}^{0} + r_{D} \cdot T_{i,n+3}^{0}.$
Since $\sum_{j=1}^{n+3} T_{\ell,j} =1$ for any severity score $\ell \in [n-1]$, we notice that adding $\alpha$ to all rewards is equivalent to adding $\alpha + \lambda \cdot \alpha \cdot \left( \sum_{j=1}^{10} T_{ij} \right)$ to the left-hand side of \eqref{eq:app:lem:3} and $\alpha + \lambda \cdot \alpha \cdot \left( \sum_{j=1}^{10} T_{i+1,j} \right)$ to the right-hand side of \eqref{eq:app:lem:3}. Therefore, condition \eqref{assumption-2} holds for all $\alpha \geq 0$, as long as
$ \sum_{j=1}^{n} T_{ij} \geq \sum_{j=1}^{n} T_{i+1,j}, \forall \; i \in [n-1].
 $

\end{enumerate}

\end{proof}

\section{Proof of Theorem \ref{th:nom-thr}.}\label{app:pf-main-th}
\begin{proof}
Let $\boldsymbol{V}^{0} \in \R^{n+4}$ such that $ V_{i}^{0}  =0, \forall \; i \in [n],   V_{n+1}^{0} = V_{n+2}^{0} = V_{n+3}^{0}=0,    V_{n+4}^{0}  =r_{PT}.$
Let $F:\R^{n+4} \rightarrow \R^{n+4}$ denote the function that maps $\boldsymbol{V} \in \R^{n+4}$ to $F(\boldsymbol{V}),$ where
\begin{align*}
    F(V)_{i} & = \max \{r_{W}+\lambda \cdot \sum_{j=1}^{n+3}T_{ij}^{0}V_{j},r_{W}+\lambda \cdot r_{PT} \}, \forall \; i \in [n], \\
    F(V)_{n+1}& = r_{CR}, F(V)_{n+2}=r_{RL}, F(V)_{n+3}=r_{D},\\
        F(V)_{n+4} & =r_{PT}.
\end{align*}
The function $F$ is the Bellman operator associated with our single-patient MDP with transition kernel $\boldsymbol{T}^{0}.$ Therefore, we know that the value iteration algorithm finds an optimal policy \citep{Puterman}:
if $\pi^{*}$ is an optimal policy and $\bm{V}^{\pi^{*}}$ is its value vector, then
$ \lim_{t \rightarrow \infty} F^{t}(\boldsymbol{V}^{0}) =\lim_{t \rightarrow \infty} \boldsymbol{V}^{t} =  \boldsymbol{V}^{\pi^{*}},$
and $ \lim_{t \rightarrow \infty} \pi^{t} \rightarrow \pi^{*}$ where $\pi^{t}$ is the sequence of deterministic policies such that $$ \pi^{t}(i) = 1 \iff r_{W}+\lambda \cdot \sum_{j=1}^{n+3}T_{ij}^{0}V_{j}^{t} <r_{W}+\lambda \cdot r_{PT}, \forall \; i \in [n].$$
%

We will prove by induction that the policy $\pi^{t}$ is a threshold policy at every iteration $t \geq 1.$

At $t=1$, for $i \in [n]$,
\begin{align*}
F(\bm{V})_{i} & = \max \{r_{W}+\lambda \cdot \sum_{j=1}^{n+3}T_{ij}V^{0}_{j},r_{W}+\lambda \cdot r_{PT} \}
=\max \{r_{W}+\lambda \cdot 0,r_{W}+\lambda \cdot r_{PT} \},
\end{align*}
and therefore we have
$
\pi^{1}(i) =1, \forall \; i \in [n].
$
Therefore, $\pi^{1}$ is a threshold policy, and its threshold is $1: \pi^{1} = \pi^{[1]}.$

Let $t \geq 1$ and let us assume that $\pi^{t}$ is a threshold policy, and let its threshold be $\tau \in [n+1]$. We prove that the policy $\pi^{t+1}$ is a threshold policy. In order to do so, we will prove that for any $i \in [n-1],$
$\pi^{t+1}(i)=1 \Rightarrow \pi^{t+1}(i+1)=1.$
Let $i \in [n-1]$ such that $\pi^{t+1}(i)=1$.
From the definition of $\pi^{t+1}$,
$$ r_{W}+\lambda \cdot \left( \sum_{j=1}^{n}T_{ij}^{0}V^{t}_{j} + out(i) \right) < r_{W}+\lambda \cdot r_{PT}. $$
Moreover, since $\pi^{t}=\pi^{[\tau]}$, we know that the vector $\boldsymbol{V}^{t}$ is such that
\begin{align}
 \label{ineq:vec:th}   V_{\ell}^{t} & > r_{W} + \lambda \cdot r_{PT}, \forall \; \ell\in \{1,...,\tau-1\}, \\
    V_{\ell}^{t} & = r_{W}+\lambda \cdot r_{PT}, \forall \; \ell \in \{\tau,...n\},\\
    V_{n+1}^{t}& = r_{CR}, V_{n+2}^{t}=r_{RL}, V_{n+3}^{t}=r_{D},\\
        V_{n+4}^{t} & =r_{PT}.
\end{align}
Now following the value iteration algorithm, we know that
\begin{align*}
F(\boldsymbol{V}^{t})_{\ell} & \in  \max \{r_{W}+\lambda \cdot \sum_{j=1}^{n+3}T_{\ell,j}^{0}V^{t}_{j},r_{W}+\lambda \cdot r_{PT} \}, \forall \; \ell \in [n].
\end{align*}
We have
\begin{align}
  \label{eq:1}  r_{W}+\lambda \cdot r_{PT} & \geq r_{W}+\lambda \cdot \sum_{j=1}^{n+3}T_{ij}^{0}V^{t}_{j} \\
 \label{eq:2}   & \geq r_{W}+\lambda \cdot \sum_{j=1}^{n}T_{ij}^{0}V^{t}_{j} + \lambda \cdot out(i) \\
 \label{eq:3}   & \geq r_{W}+\lambda \cdot \sum_{j=1}^{\tau-1}T_{ij}^{0}V^{t}_{j} + \lambda \cdot \sum_{j=\tau}^{n}T_{ij}^{0}V^{t}_{j} + \lambda \cdot out(i) \\
\label{eq:4}    & \geq r_{W}+\lambda \cdot \sum_{j=1}^{\tau-1}T_{ij}^{0}V^{t}_{j} + \lambda \cdot \sum_{j=\tau}^{n}T_{ij}^{0}(r_{W}+\lambda \cdot r_{PT}) + \lambda \cdot out(i) \\
 \label{eq:5}   & \geq r_{W}+\lambda \cdot \sum_{j=1}^{\tau-1}T_{ij}^{0}(r_{W}+\lambda \cdot r_{PT}) + \lambda \cdot \sum_{j=\tau}^{n}T_{ij}^{0}(r_{W}+\lambda \cdot r_{PT}) + \lambda \cdot out(i) \\
 \label{eq:6}   & \geq r_{W}+\lambda \cdot \sum_{j=1}^{\tau-1}T_{i+1,j}^{0}(r_{W}+\lambda \cdot r_{RL}) + \lambda \cdot \sum_{j=\tau}^{n}T_{i+1, j}^{0}(r_{W}+\lambda \cdot r_{PT}) + \lambda \cdot out(i+1) \\
  \label{eq:7}  & \geq r_{W}+\lambda \cdot \sum_{j=1}^{\tau-1}T_{i+1,j}^{0}V^{t}_{j} + \lambda \cdot \sum_{j=th}^{n}T_{i+1,j}^{0}(r_{W}+\lambda \cdot r_{PT}) + \lambda \cdot out(i+1) \\
 \label{eq:8}   & \geq r_{W}+\lambda \cdot \sum_{j=1}^{n+3}T_{i+1,j}^{0}V^{t}_{j},
\end{align}
where Inequality \eqref{eq:1} follows from $\pi^{t+1}(i)=1,$ Inequality \eqref{eq:4} follows from the fact that the policy $\pi^{t}$ is a threshold policy and therefore the vector $\boldsymbol{V}^{t}$ satisfies Inequalities \eqref{ineq:vec:th}. Inequality \eqref{eq:6} follows from Assumption \ref{assumption-1}, because Assumption \ref{assumption-1} and $r_{PT} \leq r_{RL}$ imply that for all threshold $\tau \in \{1,...,n+1\}$,
    $$ \left( \sum_{j=1}^{n} T_{ij}^{0} \right) \cdot ( r_{W}+ \lambda \cdot r_{PT}) + \lambda \cdot out(i) \geq \left( \sum_{j=1}^{\tau-1} T_{i+1,j}^{0} \right) \cdot (r_{W} + \lambda \cdot r_{RL}) + \left( \sum_{j=\tau}^{n} T_{i+1, j}^{0} \right) \cdot (r_{W} + \lambda \cdot r_{PT})   + \lambda \cdot out(i+1).$$ Inequality \eqref{eq:7} follows from Lemma \ref{lem-Vinfty}: $\forall \ell \in [n]
V_{\ell}^{t}  \leq V^{\pi^{*}}_{\ell}
\leq r_{W} + \lambda \cdot r_{RL},$
where the first inequality follows from the fact that the operator $F$ is a \textit{non-increasing} mapping and
$ \lim_{t \rightarrow \infty} \boldsymbol{V}^{t} = \lim_{t \rightarrow \infty} F^{t}(\boldsymbol{V}^{0}) = \boldsymbol{V}^{\pi^{*}}.$

Therefore, we conclude that
$r_{W} + \lambda \cdot r_{PT} \geq r_{W}+\lambda \cdot \sum_{j=1}^{n+3}T_{i+1,j}^{0}V^{t}_{j},$
which implies that $\pi^{t+1}(i+1)=1.$ We can thus conclude that the policy $\pi^{t+1}$ is a threshold policy.

Therefore, for all $t \geq 1,$ the policy $\pi^{t}$ is threshold. We can conclude that there exists an optimal policy that is a threshold policy.
\end{proof}

\section{Non-threshold optimal policies.}\label{app:counterex}
We provide an example of a single-patient MDP which does not satisfy Assumption \ref{assumption-1} and for which the optimal nominal policy is not threshold. In particular, the optimal policy in the MDP of Figure \ref{fig:counterexample} is to proactively transfer a patient in state $1$ and to not proactively transfer a patient in state 2.
 \begin{figure}[H]
\center \includegraphics[scale=0.2]{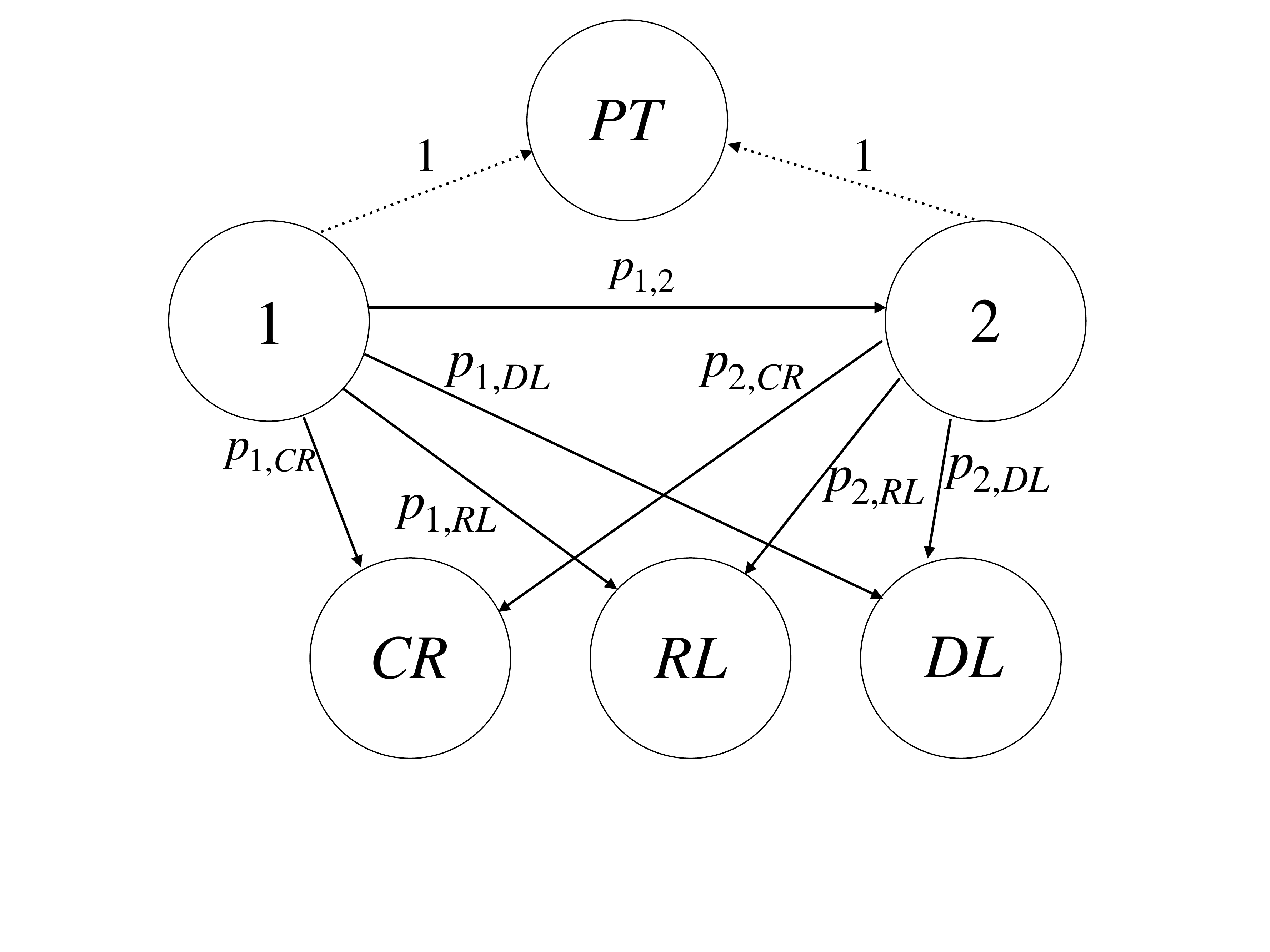}
 \caption{Example of an MPD where the optimal policy is not threshold. There is no self-transition in state $1$ or $2$. In state $1$, the patient transitions to state $2, CR, RL$ or $D$ (solid arcs), or is proactively transferred (dashed arc). In state $2$, the patient has to exit the ward, either by proactive transfer , in which case s/he transitions to $PT$ with probability $1$ (dashed arc), either by transitioning $CR, RL$ or $D$ (solid arcs).  The patient can not transition back to state $1$. We provide values of the rewards and transitions for which the optimal policy is not threshold.}
 \label{fig:counterexample}
 \end{figure}
 Condition \eqref{eq:assumption:sub-1} is not satisfied, i.e., $out(1) < out(2).$ We also set the rewards $r_{W}, r_{PT}, r_{RL},r_{CR},r_{D}$  such that $out(2) > r_{PT}$, which means that the optimal nominal policy will not proactively transfer the patient when in state $2$: $\pi^{*}(1)=0$. However,
\begin{align*}
\pi^{*}(1) = 1 & \iff r_{W} + \lambda \cdot r_{PT} > r_{W} + \lambda \cdot (out(1) + p_{1,2} (r_{W} + \lambda \cdot out(2)) \\
& \iff r_{PT} > out(1) + p_{1,2} (r_{W} + \lambda \cdot out(2)).
\end{align*}
Therefore, when $out(1) < r_{PT} < out(2)$ and the discount factor $\lambda$ is small enough, the decision-maker has an incentive to proactively transfer the patient in state $1$. In particular, this is the case for the following set of parameters: $$(p_{1,RL},p_{1,CR},p_{1,D})=(0.3,0,0.3), p_{1,2}=0.4, (p_{2,RL},p_{2,CR},p_{2,D})=(0.3,0.4,0.3), \lambda=0.01,$$
 $$r_{W}=1.6, r_{RL}=3,r_{CR}=2,r_{D}=1.5,r_{PT}=2.15.$$
We detail the computation of an optimal policy for the single-patient MDP of Figure \ref{fig:counterexample}.
We start with the Bellman Equation in state $2$:
$V^{*}_{2} = \max \{ r_{W} + \lambda \cdot out(2), r_{W} + \lambda \cdot r_{PT} \}.$
Since $out(2)= (0.3,0.4,0.3)^{\top}(3,2,1.5) = 2.15 > r_{PT}=2$, we know that $\pi^{*}(2)=0$, and the optimal policy does not transfer the patient with a severity score of $2$. Moreover,
$V^{*}_{2}=r_{W} + \lambda \cdot out(2) = 1.6 + 0.01 \cdot 2.15 = 1.6215.$
Let us compute $V^{*}_{1}.$  The Bellman Equation in state $1$ gives
$$V^{*}_{1} = \max \{ r_{W} + \lambda \cdot (p_{1,2} \cdot V^{*}_{2}+ out(1)), r_{W} + \lambda \cdot r_{PT} \}.$$
Moreover, $out(1)=(0.3,0,0.3)^{\top}(3,2,1.5) = 1.35 <out(2) =2.15 .$
Therefore,
$$V^{*}_{1} = \max \{ 1.6 + 0.01\cdot (0.4 \cdot 1.6215+ 1.35), 1.6 + 0.01 \cdot 2 \} = \max \{1.619986, 1.62\}=1.62, $$
from which we conclude that $\pi^{*}(1)=1$, and the optimal policy proactively transfers the patients with severity score of $1$. Therefore, the optimal nominal policy is not threshold. We would like to note that we could have chosen any set of parameters for which
$ out(2) > r_{PT}>out(1), r_{PT} > p_{1,2} \cdot (r_{W} + \lambda \cdot out(2) )+out(1).$
In practice, the discount factor is likely to be significantly higher than $0.01$, since the decision-maker in the hospital likely does care about the long-term impacts of the transfer policies.

\section{Details about the nominal matrix.}
\label{app:values-nom-parameters}
\paragraph{Confidence intervals.}We use the method in \cite{conf_interval} to compute $95 \%$ confidence intervals around the nominal matrix $\bm{T}^{0}$. This method yields
\begin{align}
[T_{ij}^{0} - \alpha_{i}, T_{ij}^{0}+2 \cdot \alpha_{i}], \forall \; (i,j) \in [10] \times [13], \boldsymbol{\alpha}= 10^{-4} \cdot (  4 ,8, 10, 14,  15,  43,  46, 47, 46, 45).
\end{align}
We notice that the confidence intervals are larger for small severity scores than for larger severity scores (up to one order of magnitude). This is because large severity scores correspond to more serious health conditions, which are less likely to be observed than smaller severity scores.
\paragraph{Details on nominal factor matrix.}
 We want to know if the errors between $\bm{T}^{0}$ and $\bm{\hat{T}}$ are more important for some severity scores than others. Therefore, we compute the maximum absolute and relative deviations between each row of $\bm{\hat{T}}$ and each row of $\bm{T}^{0}$. In general, we notice that the absolute errors are higher for high severity scores. For instance, for severity score $1$ the maximum absolute error is $0.0023$. On the other hand, for severity score of $9$, the maximum error is $0.0049$. However, the maximum \textit{relative} error is higher for low severity scores (from $1$ to $5$). Even though the absolute deviations are small, they amount to large relative deviations because they occur on coefficients that are already small. For instance, $T_{1,7}^{0} = 5.40 \cdot 10^{-5}$ and $\hat{T}_{1,7} = 6.23 \cdot 10^{-5}$, which gives a relative deviation of about $15 \%$, even though the absolute deviation is in the order of $10^{-5}.$

\section{Sensitivity Analysis for our single-parameter MDP and our hospital simulations.}\label{app:sens}
In this section we present a detailed sensitivity analysis for our single-patient MDP and our hospital simulations.  We have mentioned that the \textit{ordering} of the rewards of our single-patient MDP can be inferred from the outcomes of the patients and the use of ICU resources (see \eqref{eq:rec-def}). We consider the impact of a change in the reward parameters presented in Section \ref{sec:MDP-parameters}. We have seen in Section \ref{sec:MDP-emp-results} that for this setting of rewards, the optimal nominal policy is $\pi^{[6]}$, while the optimal robust policy is $\pi^{[5]}$ (for $\U_{\min}$ and $\U_{\sf emp}$). We choose to study the variations in $r_{RL}$ and $r_{PT-RL}$, since these rewards have reversed influences on the thresholds of the optimal policies. Indeed, $r_{RL}$ is associated with a patient recovering from the ward, i.e., a patient who has not been proactively transferred, while $r_{PT-RL}$ is the reward associated with a patient recovering \textit{after} having been proactively transferred.

We present in Table \ref{tab:sens-hosp-PT} the variations in the hospital performance of the optimal robust policies, for changes in the value of $r_{RL}$ (from $240$ to $260$) and in the value of $r_{PT-RL}$ (from $180$ to $200$). We would like to note that both the thresholds of the optimal robust policies and the associated worst-case transition matrix may vary when the rewards parameters change.

\begin{table}[H]

 \caption{Worst-case mortality and average ICU occupancy of the optimal robust policies for the uncertainty sets $\U_{\min}, \U_{\sf emp}$ and $\U_{\sf sa}$, for variations in the rewards $r_{PT-RL}$ and $r_{RL}$.}
    \centering
       \label{tab:sens-hosp-PT}
	\begin{tabular}{c||ccc|ccc||c|ccc|ccc|}
    	\hline
    	\multicolumn{1}{c||}{$r_{PT-RL}$} & \multicolumn{3}{c}{Mort. ($\%$)} &  \multicolumn{3}{c}{ICU occupancy ($\%$)}  & \multicolumn{1}{||c|}{$r_{RL}$} & \multicolumn{3}{c}{Mort. ($\%$)} & \multicolumn{3}{c}{ICU occupancy ($\%$)}  \\
    	    	\hline
 $ \; $ & \; $\U_{\min}$ \;  &  $\U_{\sf emp}$ & \; $\U_{\sf sa}$ \; & \; $\U_{\min}$ \;  &  $\U_{\sf emp}$ & \; $\U_{\sf sa}$ & $ \; $ & \; $\U_{\min}$ \;  &  $\U_{\sf emp}$ & \; $\U_{\sf sa}$ \;  & \; $\U_{\min}$ \;  &  $\U_{\sf emp}$ & \; $\U_{\sf sa}$ \\
	      \hline
$180$ & 5.73	& 6.47	& 7.02  & 71.88&	75.34 &	77.25 & $240$ & 4.78	& 5.46 & 	6.02   & 75.4 & 	77.97	& 78.96\\
$185$& 5.57 &	6.35 &	7.02  & 72.34&75.62&	77.34 & $245$& 5.11	& 5.82&	6.41  & 73.95&	76.70	& 77.81\\
$190$ & 5.11&	5.83	 & 6.40  & 73.95&	76.70 &	77.80 & $250$ & 5.11 & 	5.83	 & 6.40  & 73.95 &	76.70	& 77.80\\
$195$ & 5.11&	5.46	& 6.02   & 73.94	&77.98	&78.96 & $255$ &5.57 &	6.36 &	6.41  & 72.33 &	75.63	&77.81\\
$200$ & 4.78	& 5.12	 & 5.64 & 75.39	&79.03&	80.02 & $260$ & 5.57 &	6.36	& 7.02  & 72.33	&75.63	&77.24\\
   	   \hline
    \end{tabular}

\end{table}

In Table \ref{tab:sens-hosp-PT}, we notice that the hospital performance of the optimal robust policies can vary when the set of reward parameters of our single-patient MDP does change. However, these changes are mostly due to the fact that we are comparing the optimal robust policies for different values of the rewards, and therefore the thresholds of the policies that we are comparing do vary. For instance, for $\U_{\min}$, the optimal robust policy is $\pi^{[8]}$ for $r_{PT-RL}=240$ but it is $\pi^{[4]}$ for $r_{PT-RL}=260$. On the contrary, we notice that when the optimal robust threshold is the same, \textit{variations in rewards value do not impact the hospital worst-case performance.} For instance, the optimal robust policies are the same (equal to $\pi^{[5]}$ across all three uncertainty sets) when $r_{RL}=250$ and when $r_{RL}=245$. The hospital worst-case performance (mortality, LOS, ICU occupancy) are also the same, when $r_{RL}=250$ and when $r_{RL}=245$. However, these two rows of worst-case performance are computed for \textit{different} worst-case matrices (since the reward parameters for our single-patient MDP were different).

Therefore, we can conclude that even though the optimal robust and nominal policies can vary in our single-patient MDP, the hospital performance remain stable for each threshold policy individually. The ordering of the rewards does yield worst-case transition matrices for our single-patient MDP that are \textit{good} and \textit{robust} candidate worst-case transition matrices for the hospital worst-case performance, since variations in the rewards parameters on the single-patient MDP side still yield worst-case hospital performance that are very similar.

%

\section{Numerical results for rank $r=7$.}\label{app:r=7}
In this section we present our numerical results for the performance of the hospital, when the NMF approximation $\bm{\hat{T}}$ is of rank $r=7$.
\paragraph{Errors of the NMF approximations.}
 For $r=7$, we compute a new $\hat{\bm{T}}$ solution of the NMF optimization program. Our solution $\boldsymbol{\hat{T}}=\boldsymbol{U\hat{W}}^{\top}$ achieves the following errors:
    $ \| \boldsymbol{T}^{0} - \boldsymbol{\hat{T}} \|_{1} = 0.1932, \| \boldsymbol{T}^{0} - \boldsymbol{\hat{T}} \|_{\infty} = 0.0224, \| \boldsymbol{T}^{0} - \boldsymbol{\hat{T}} \|_{\sf relat,\boldsymbol{T}^{0}} = 0.4093.$
    In more details, $\bm{\hat{T}}$ achieves the following errors.
\begin{table}[H]
\center
\begin{tabular}{|c|cccc|}
\hline
                   & max. & mean    & median  & 95\% percentile \\
                   \hline
absolute deviation & 0.0224  & 0.0015 & 0.0004 & 0.0069          \\
relative deviation & 0.4093  & 0.0856  & 0.0432  & 0.3247 \\
\hline
\end{tabular}
\caption{Statistics of the absolute and relative deviations of $\bm{\hat{T}}$ from $\bm{T}^{0}$ for a rank $r=7$.}
\label{tab:abs-relat-r-7}
\end{table}
As we can see in Table \ref{tab:abs-relat-r-7}, the absolute deviations remains small. Additionally, the relative differences between the coefficients are moderate, with half being less than $8.56 \%$. That said, the maximum relative different is 40.93\%. This occurs with $\hat{T}_{4,6}=0.0035$, while $T_{4,6}^{0}=0.0060$; so while the relative deviation is quite large, the  absolute variation is only in the order of $10^{-3}.$

\paragraph{Mortality and Length-Of-Stay.} We present the worst-case performance of the $11$ threshold policies, for our uncertainty sets $\U_{\min}$ and $\U_{\sf emp}$, when the rank is $r=7$. For references we still show the performance for the nominal transition kernel $\bm{T}^{0}$ (nominal performance), for the uncertainty set $\U_{\sf sa}$ and for our NMF solution of rank $r=7$.

We first note that the hospital performance with our NMF approximation of rank $r=7$ are very close to the hospital performance for $\bm{T}^{0}$, which provides support that $\hat{\bm{T}}$ is a plausible transition matrix.
We notice that the performance of the threshold policies can still significantly deteriorate, even for small variations from the nominal matrix $\bm{T}^{0}$. In particular, there is a $20 \%$ increase in the average mortality, for some worst-cases matrices in $\U_{\min}$ and $\U_{\sf emp}.$ Interestingly, the uncertainty set $\U_{\min}$ yields worst-case mortality rates that are higher than for worst-cases matrices in $\U_{\sf emp}$, contrary to what we noticed in Section \ref{sec:exp} for rank $r=8$. However, these two uncertainty sets still yield the same insights, which are that the performance can significantly degrade even for small deviations, and that in worst-case, the initial decrease for proactively transferring the patients with the highest severity scores (policy $\pi^{[11]}$, top-left of each curve, to policy $\pi^{[6]}$, the sixth point of each curve, starting from the left) is steeper than the initial decrease for the nominal performance. Moreover, these insights are still different from the worst-cases performance in $\U_{\sf sa}$, since the results for $\U_{\sf sa}$ are independent of the rank chosen for our NMF approximation. In particular, for worst-cases in $\U_{\sf sa}$, the decision-maker appears to be able to proactively transfer the patients with severity scores in $\{8,9,10\}$, \textit{without} increasing the ICU occupancy.

Therefore, our numerical simulations for rank $r=7$ are corroborating our numerical simulations of Section \ref{sec:exp} for rank $r=8$. We do not present the hospital simulations for lower ranks, since the NMF approximations become very poor for rank $r$ lower than $7$. For instance, for $r=6$, there are $54$ coefficients (out of $130$ coefficients) outside of the confidence intervals, and for a rank $r=5$, our NMF solution has $70$ coefficients  that are outside the $95 \%$ confidence intervals.

\begin{figure}[H]
\begin{subfigure}{0.24\textwidth}
 \includegraphics[width=1.1\linewidth,height=5cm]{figures/Fig4a}
\caption{Random samples analysis (mortality).}
\label{fig:mort_rand_bump_1_rk_7}
\end{subfigure}
\begin{subfigure}{0.24\textwidth}
  \includegraphics[width=1.1\linewidth,height=5cm]{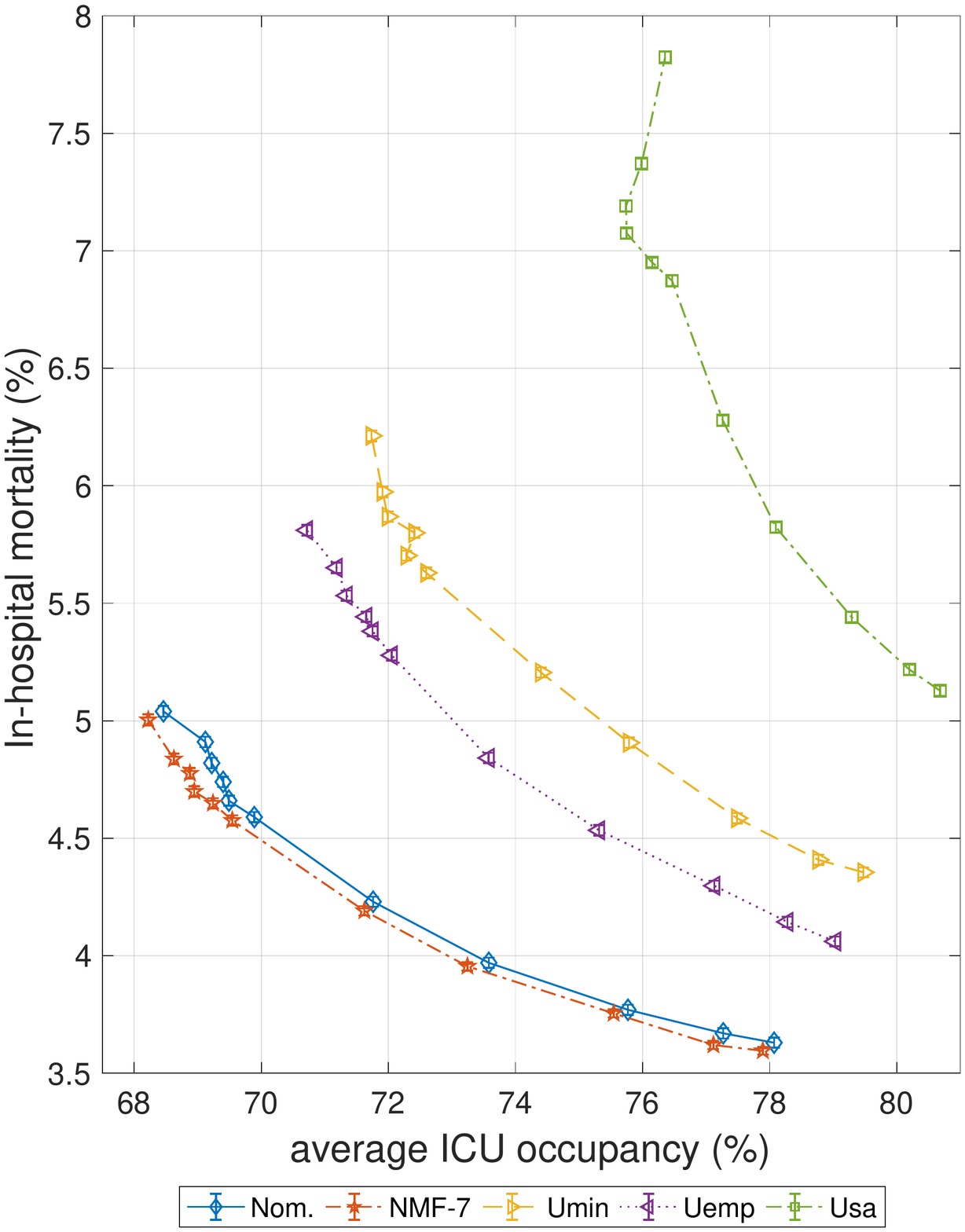}
\caption{Worst-case analysis (mortality).}
\label{fig:mort_worst_bump_1_rk_7}
\end{subfigure}
\begin{subfigure}{0.24\textwidth}
 \includegraphics[width=1.1\linewidth,height=5cm]{figures/Fig5a}
\caption{Random samples analysis (LOS).}
\label{fig:LOS_rand_bump_1_rk_7}
\end{subfigure}
\begin{subfigure}{0.24\textwidth}
  \includegraphics[width=1.1\linewidth,height=5cm]{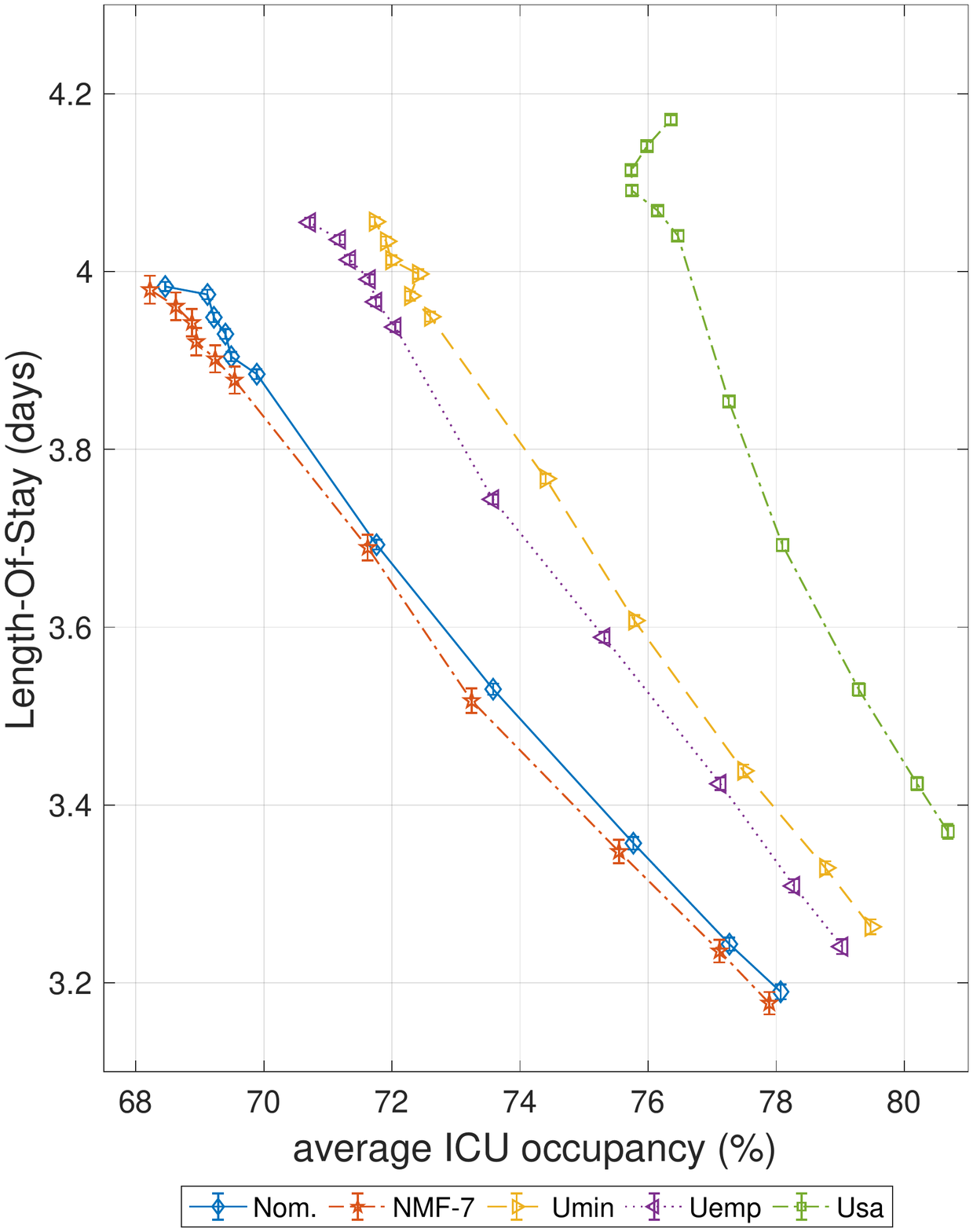}
\caption{Worst-case analysis (LOS).}
\label{fig:LOS_worst_bump_1_rk_7}
\end{subfigure}
\caption{In-hospital mortality and length-of-stay of the $11$ threshold policies for the nominal estimated matrix, randomly sampled matrices in the $95\%$ confidence intervals and  the worst-case matrices found by our single MDP model (right-hand side).}
\label{fig:mort-rk-7}
\end{figure}
%

\section{Hospital performance for random deviations around the nominal kernel.}\label{app:rand-sample}

We sample at random $20$ matrices in the confidence intervals \eqref{eq:conf_unrelated}. In order to do so, we first sample a matrix of deviations $\bm{D} \in \R^{10 \times 13}$, with $D_{ij} \in [- \alpha_{i},+[ 2 \cdot \alpha_{i}], \forall (i,j) \in [10] \times [13].$ Note that the matrix $\bm{T}^{0} + \bm{D}$ is not necessarily a transition matrix, because each of its row does not necessarily sum up to $1$. Therefore, we project each of the rows of the matrix $\bm{T}^{0} + \bm{D}$ onto the simplex and we obtain a matrix a new matrix $\tilde{\bm{T}}$. If the corresponding matrix $\tilde{\bm{T}}$ is inside the confidence-intervals, we compute the hospital performance of the $11$ threshold policies. Otherwise, we reject $\tilde{\bm{T}}$ and sample a new deviation matrix $\bm{D}.$

Using this method, we compute the performance of the threshold policies for $20$ matrices chosen randomly inside the confidence intervals \eqref{eq:conf_unrelated}. Out of these $20$ simulations, $8$ were pessimistic (higher mortality / Length-Of-Stay / ICU occupancy than in the nominal case) and $12$ were optimistic (lower mortality / Length-Of-Stay / ICU occupancy than in the nominal case)

\paragraph{Mortality.} For the in-hospital mortality, the average relative deviations from the nominal performance ranged from $3.00 \%$ to $3.84 \%$ (from threshold $0$ to threshold $11$). For each policy, the maximum relative deviation from the nominal performance ranged from $6.19 \%$ to $8.82 \%$ (again for threshold $0$ to threshold $11$).

\paragraph{LOS.}

For the Length-Of-Stay, the average relative deviations from the nominal performance ranged from $0.34 \%$ to $1.04\%$ (for threshold $3$ and threshold $11$). For each policy, the maximum relative deviation from the nominal performance ranged from $0.91 \%$ to $2.79 \%$ (for threshold $0$ and threshold $11$).

\paragraph{ICU occupancy.}

For the average ICU occupancy, the average relative deviations from the nominal performance ranged from $0.37 \%$ to $0.57\%$ (for threshold $0$ and threshold $11$). For each policy, the maximum relative deviation from the nominal performance ranged from $0.99 \%$ to $1.57 \%$ (for threshold $11$ and threshold $0$).

{\color{black}
\section{Additional figures for Section \ref{sec:simu-queue}.}\label{app:proportion-and-patient-types}
\subsection{Worst-case simulations for the probability to enter the queue}
We present in Figure \ref{fig:worst-case-prob} the worst-case probabilities that a patient of a certain type (direct admit, crashed and readmitted) will enter the waiting queue. The nominal probabilities are given in Figure \ref{fig:proba_types_in_queue}. We notice that in the worst-case, the probabilities can moderately deteriorate; still, the trends remain the same as in the nominal case. Namely, the probability that a certain type of patient enters the queue remains fairly stable until the threshold increases above 5 (proactively transferring the patients with the top 10 \% riskiest severity conditions).
\begin{figure}[h]
\begin{subfigure}{0.27\textwidth}
 \includegraphics[width=1.1\linewidth,height=6cm]{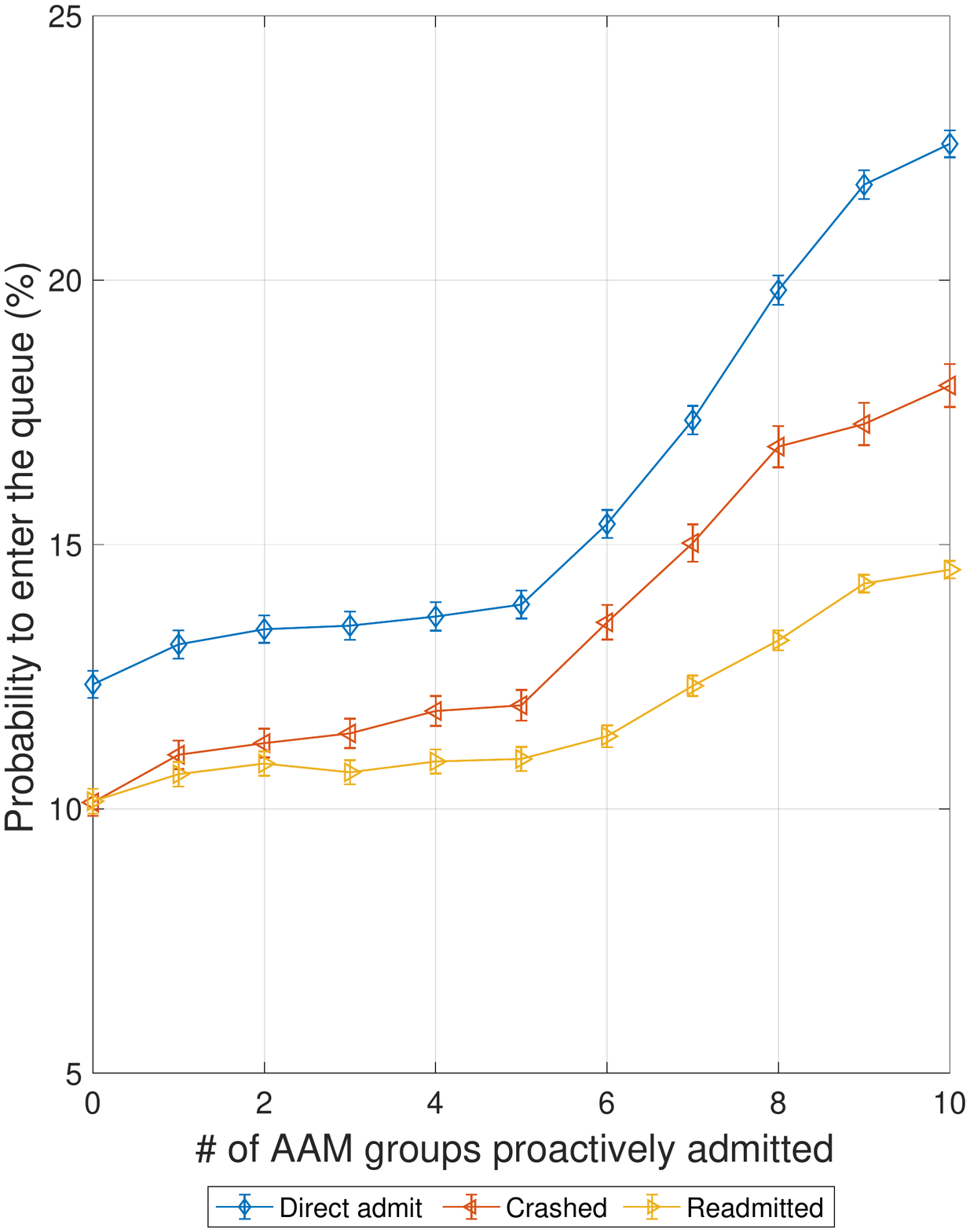}
\caption{$\U=\U_{\min}$.}
\label{fig:prob-U-min}
\end{subfigure}
\begin{subfigure}{0.27\textwidth}
 \includegraphics[width=1.1\linewidth,height=6cm]{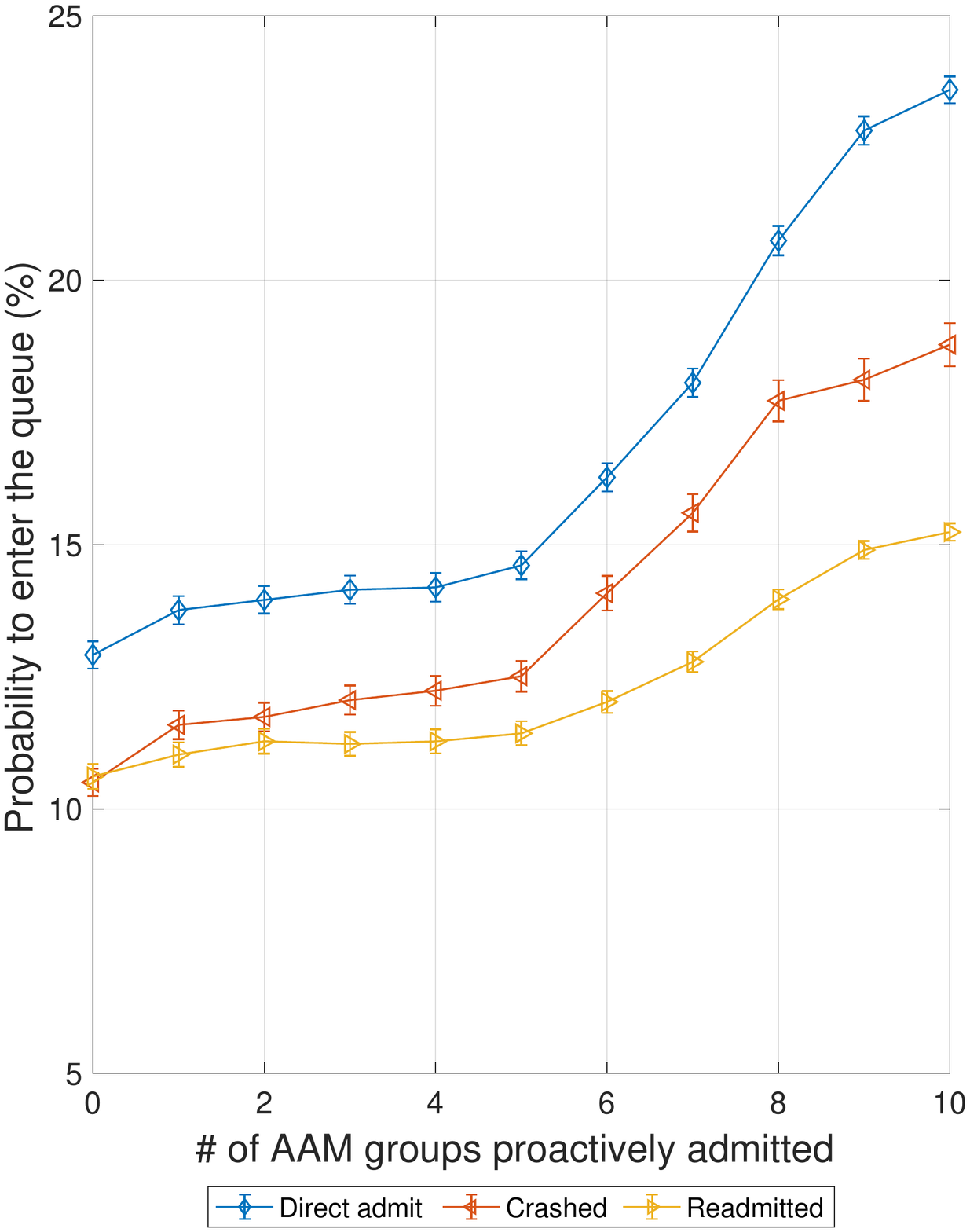}
\caption{$\U=\U_{\sf emp}$.}
\label{fig:prob-U-emp}
\end{subfigure}
\begin{subfigure}{0.27\textwidth}
 \includegraphics[width=1.1\linewidth,height=6cm]{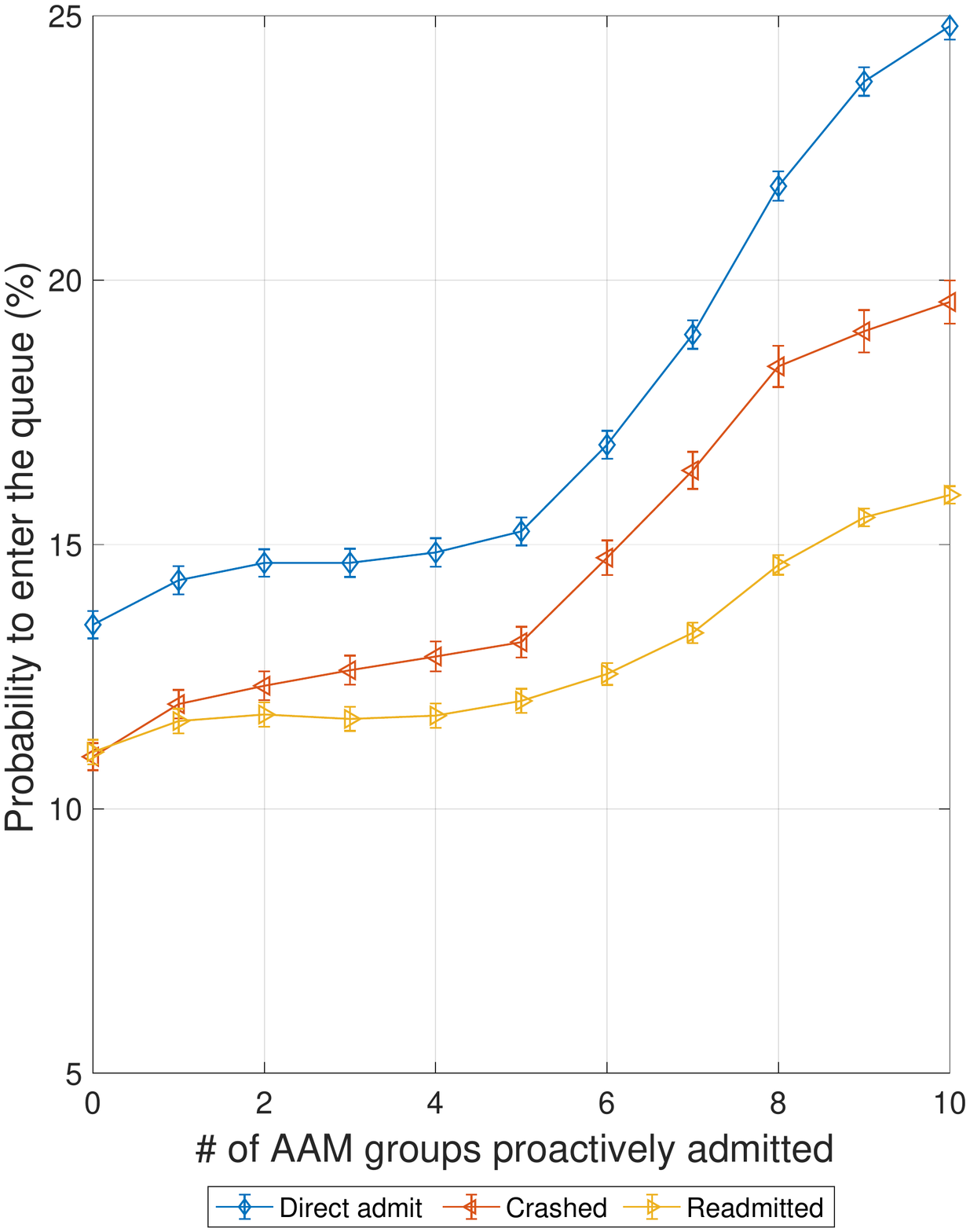}
\caption{$\U=\U_{\sf sa}$.}
\label{fig:prob-U-sa}
\end{subfigure}
\caption{For different uncertainty sets, worst-case probability to enter the queue for different patient types (direct admits, crashed and readmitted), for different threshold policies.}
\label{fig:worst-case-prob}
\end{figure}
\subsection{Worst-case simulations for the proportion of patient types in the queue}
We present here our worst-case simulations for the proportion of patient types in the waiting queue. We note that the worst-case results are very similar to the nominal case, which is as expected since we are computing \textit{proportions} (and not the absolute number of patient types in the queue). The only change compared to the nominal case is that the proportion of crashed patients slightly increases as the worst-case transition matrices chosen in the uncertainty sets are increasing the likelihood of crash. The first six threshold policies only moderately increase the worst-case proportions.
\begin{figure}[h]
\begin{subfigure}{0.24\textwidth}
 \includegraphics[width=1.1\linewidth,height=6cm]{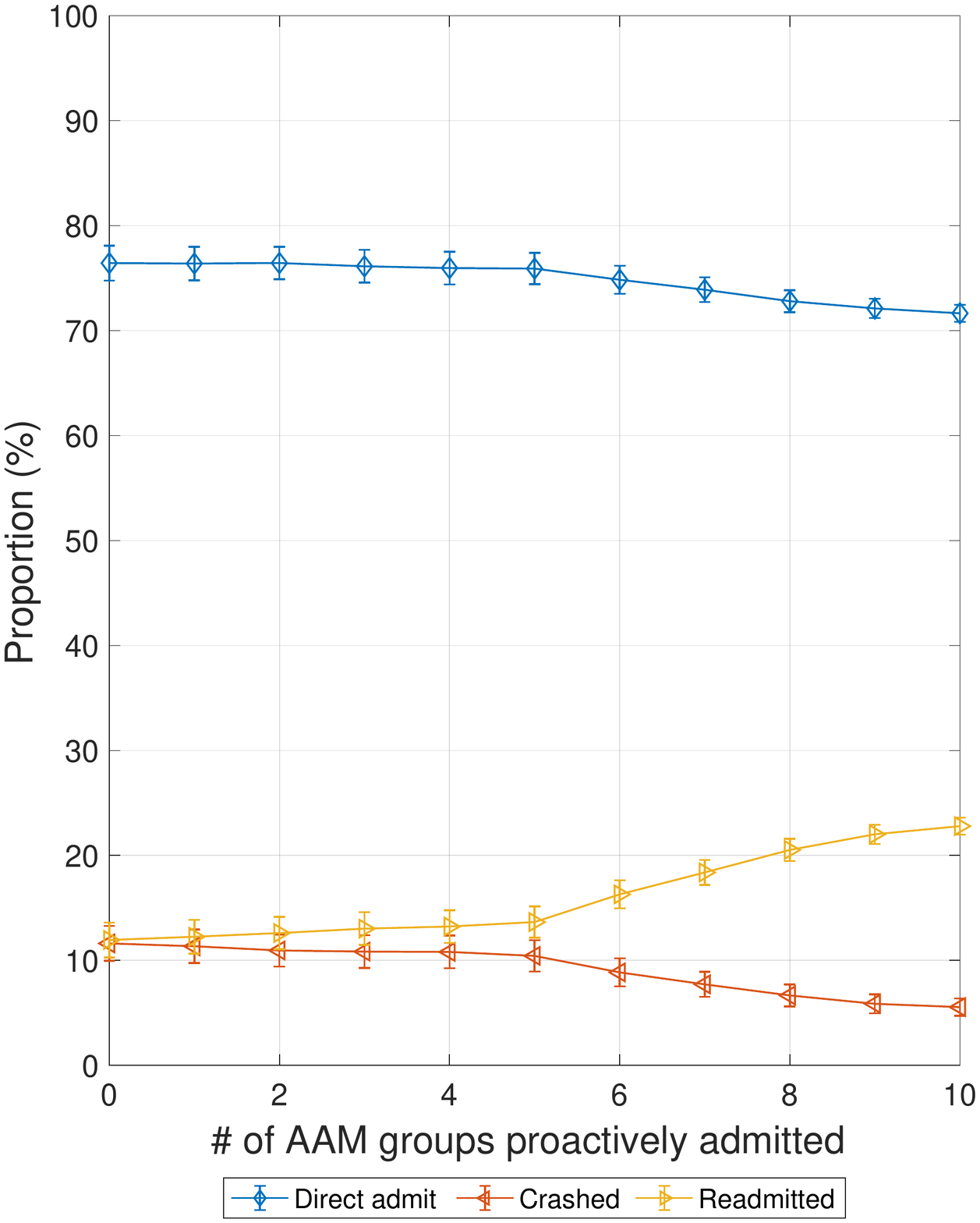}
\caption{Nominal.}
\end{subfigure}
\begin{subfigure}{0.24\textwidth}
 \includegraphics[width=1.1\linewidth,height=6cm]{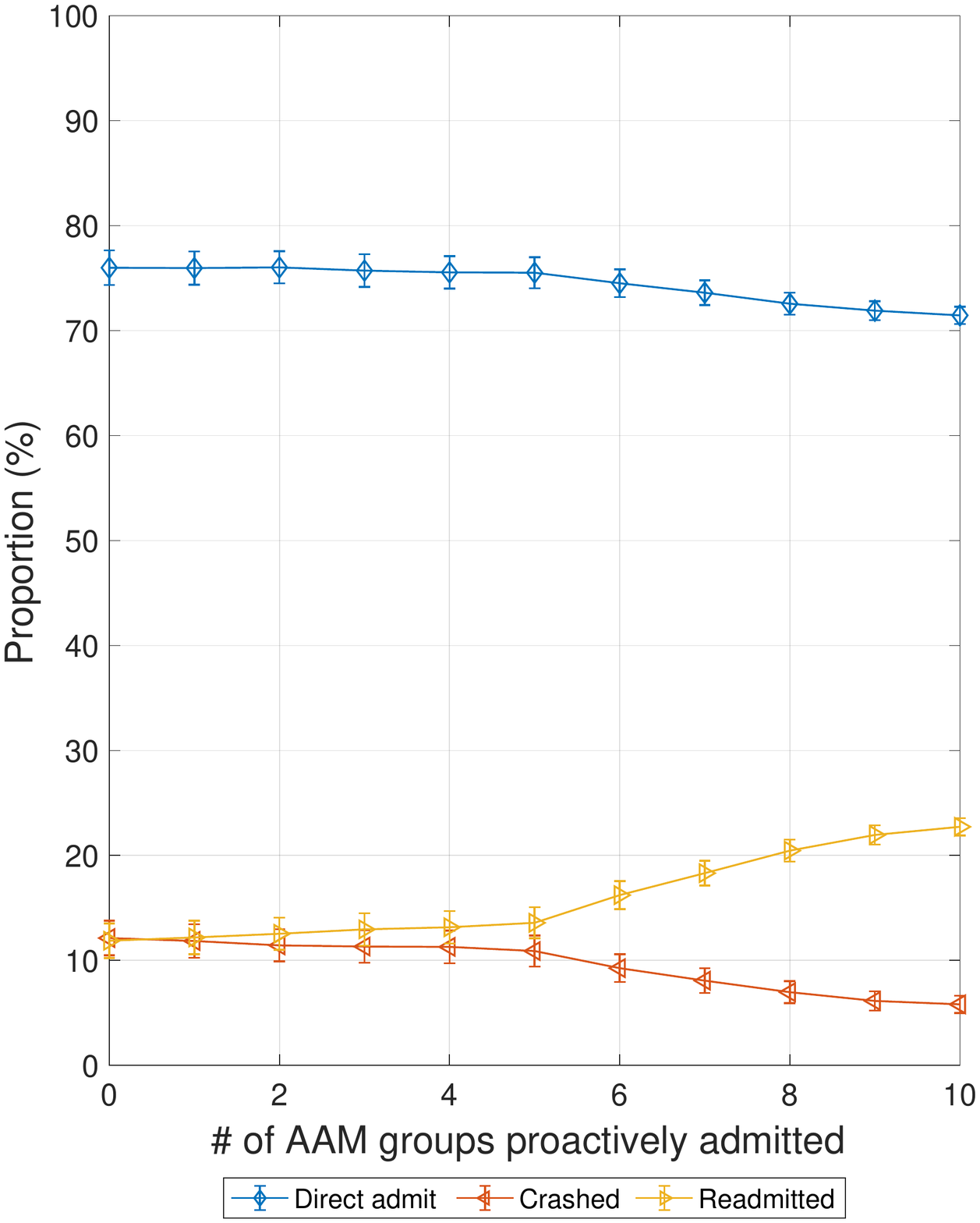}
\caption{$\U=\U_{\min}$.}
\label{fig:proportion-U-min}
\end{subfigure}
\begin{subfigure}{0.24\textwidth}
 \includegraphics[width=1.1\linewidth,height=6cm]{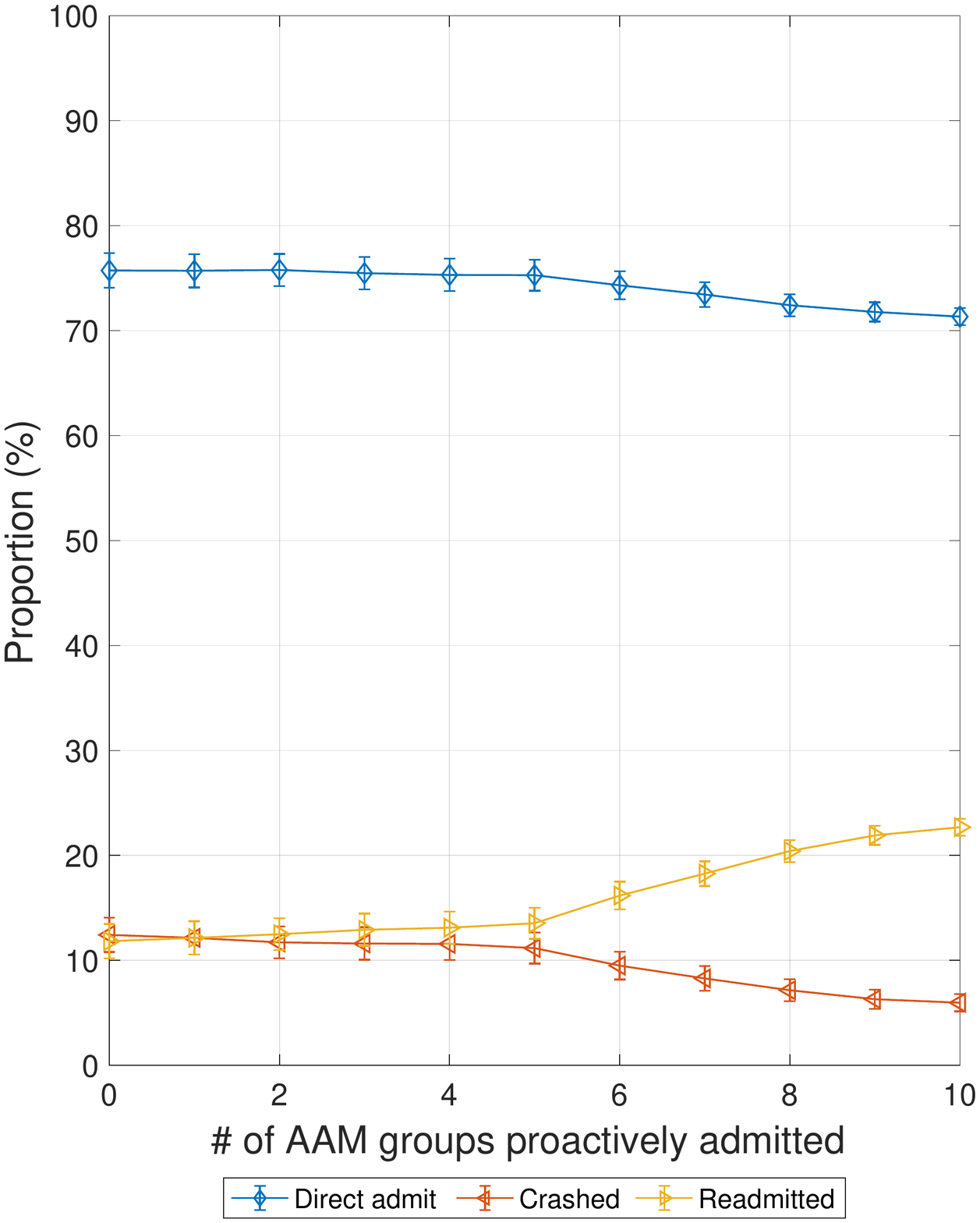}
\caption{$\U=\U_{\sf emp}$.}
\label{fig:proportion-U-emp}
\end{subfigure}
\begin{subfigure}{0.24\textwidth}
 \includegraphics[width=1.1\linewidth,height=6cm]{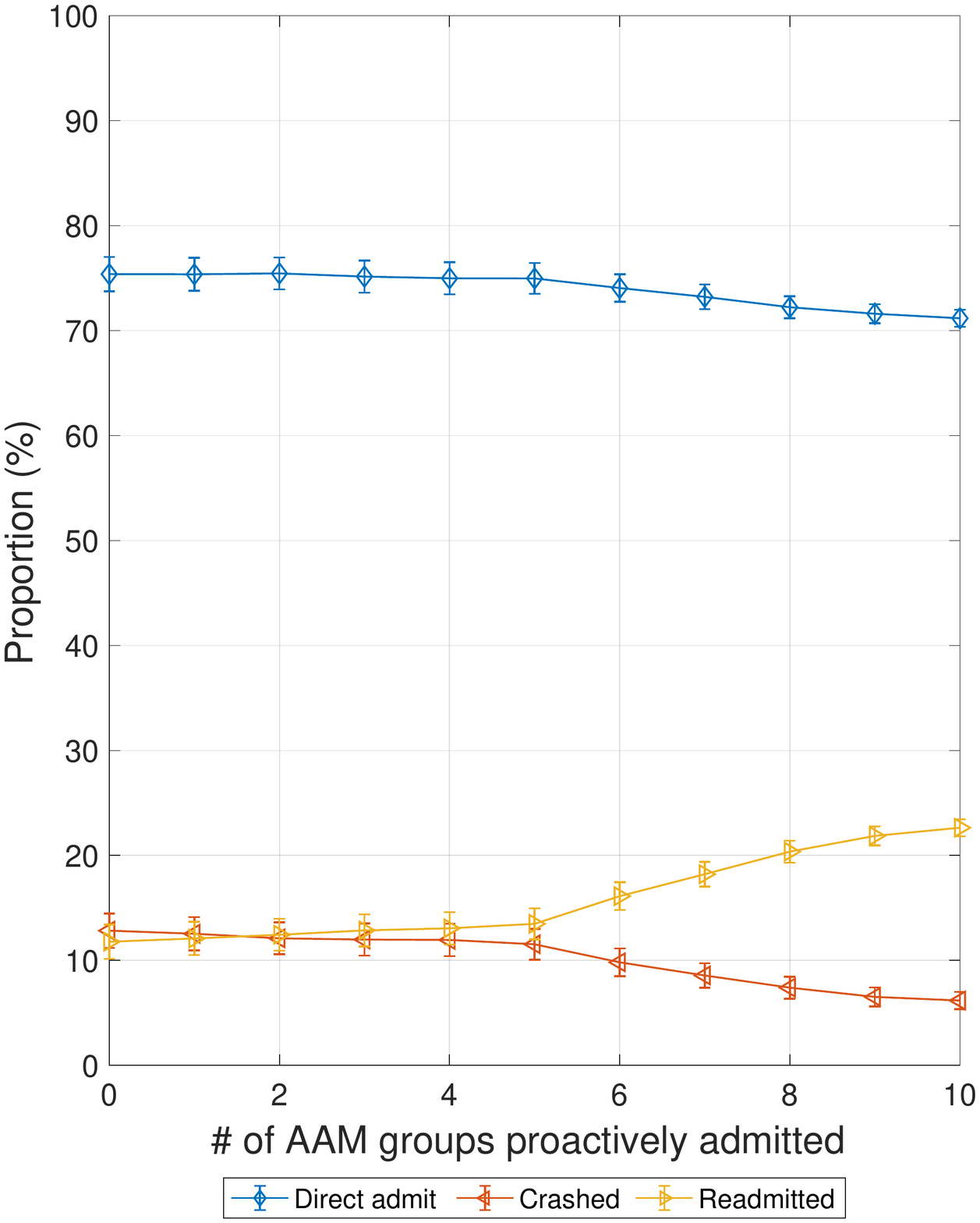}
\caption{$\U=\U_{\sf sa}$.}
\label{fig:proportion-U-sa}
\end{subfigure}
\label{fig:worst-case-proportion-type}
\caption{For the nominal matrix and for the worst-case for different uncertainty sets, proportions of patient types (direct admits, crashed and readmitted) in the waiting queue for different threshold policies.}
\end{figure}
}



\end{document}